\documentclass{article} 
\usepackage{iclr2025_conference,times}

\usepackage{amsmath,amsfonts,bm}

\def\eqref#1{equation~\ref{#1}}

\def\1{\bm{1}}

\DeclareMathAlphabet{\mathsfit}{\encodingdefault}{\sfdefault}{m}{sl}
\SetMathAlphabet{\mathsfit}{bold}{\encodingdefault}{\sfdefault}{bx}{n}

\newcommand{\R}{\mathbb{R}}

\usepackage{hyperref}
\usepackage{url}

\usepackage[utf8]{inputenc} 
\usepackage[T1]{fontenc}    
\usepackage{hyperref}       
\usepackage{url}            
\usepackage{booktabs}       
\usepackage{amsfonts}       
\usepackage{nicefrac}       
\usepackage{microtype}      
\usepackage{xcolor}         

\usepackage{comment}
\usepackage{amsmath}
\usepackage{graphicx}
\usepackage{subcaption}
\usepackage[ruled,vlined,linesnumbered,noend]{algorithm2e}
\usepackage{amsmath}
\usepackage{amssymb}
\usepackage{mathtools}
\usepackage{amsthm}
\usepackage{pgfplots}   
\usepackage{microtype}
\usepackage{graphicx}
\usepackage{booktabs} 
\usepackage{multirow}

\usepackage{enumitem}
\usepackage{accents} 
\usepackage{caption}
\usepackage{graphicx}

\usepackage[mathscr]{euscript}
\usepackage{cleveref}
\usepackage{wrapfig}

\usepackage{nicematrix}

\def\ALGNAME{GeoLoRA}
\def\ALGNAMELONG{Geometric Low-Rank Adaptation}

\newtheorem{proposition}{Proposition}

\newcommand{\ssnote}[1]{\textcolor{teal}{[{SS}] $\diamondsuit$ #1}}

\title{GeoLoRA: Geometric integration for parameter efficient fine-tuning\footnotemark[1]}
\usepackage{amssymb}
\usepackage{pifont}
\newcommand{\cmark}{\ding{51}}%
\newcommand{\xmark}{\ding{55}}%

\newcommand{\set}[1]{\left\lbrace#1\right\rbrace}

\newcommand{\augU}{{\widehat{U}}}
\newcommand{\augS}{{\widehat{S}}}
\newcommand{\newU}{{\widetilde{U}}}
\newcommand{\newV}{{\widetilde{V}}}
\newcommand{\augV}{{\widehat{V}}}

\newcommand{\augWr}{{\widehat{W}^r}}

\newcommand{\norm}[1]{\left\lVert#1\right\rVert}

\newcommand{\rom}[1]{\uppercase\expandafter{\romannumeral #1\relax}}
\newcommand{\landauO}{\mathcal{O}}
\newtheorem{theorem}{Theorem}

\newtheorem{lemma}{Lemma}
\newtheorem{assumption}{Assumption}

\newcommand{\inner}[1]{\left< #1 \right>}

\newcommand{\fx}{\mathbf{x}}

\newcommand{\fz}{\mathbf{z}}

\author{%
  Steffen Schotth\"ofer \\
     Computer Science and Mathematics Division,\\
  Oak Ridge National Laboratory, \\
  Oak Ridge, TN, USA\\
  \texttt{schotthoefers@ornl.gov} \\
  \And 
   Emanuele Zangrando \\
   School of Mathematics, \\
  Gran Sasso Science Institute, \\
  L'Aquila, Italy \\
   \texttt{emanuele.zangrando@gssi.it
} \\
  \And
  Gianluca Ceruti\\
  Department of Mathematics, \\
  University of Innsbruck, \\
  Innsbruck,  Austria \\
   \texttt{gianluca.ceruti@uibk.ac.at
} 
      \And
  Francesco Tudisco \\
School of Mathematics and Maxwell Institute, \\
University of Edinburgh, 
Edinburgh, UK; \\
School of Mathematics, \\
Gran Sasso Science Institute, 
  L'Aquila, Italy \\
  \texttt{f.tudisco@ed.ac.uk}     
  \And
  Jonas Kusch \\
  Department of Data Science, \\
Norwegian University of Life Sciences, \\
  Ås, Norway\\
  \texttt{jonas.kusch@nmbu.no
} \\}

\iclrfinalcopy 
\begin{document}

\maketitle

\begin{abstract}
Low-Rank Adaptation (LoRA) has become a widely used method for parameter-efficient fine-tuning of large-scale, pre-trained neural networks. However, LoRA and its extensions face several challenges, including the need for rank adaptivity, robustness, and computational efficiency during the fine-tuning process. We introduce \ALGNAME{}, a novel approach that addresses these limitations by leveraging dynamical low-rank approximation theory. \ALGNAME{} requires only a single backpropagation pass over the small-rank adapters, significantly reducing computational cost as compared to similar dynamical low-rank training methods and making it faster than popular baselines such as AdaLoRA. This allows \ALGNAME{} to efficiently adapt the allocated parameter budget across the model, achieving smaller low-rank adapters compared to heuristic methods like AdaLoRA and LoRA, while maintaining critical convergence, descent, and error-bound theoretical guarantees. The resulting method is not only more efficient but also more robust to varying hyperparameter settings. We demonstrate the effectiveness of \ALGNAME{} on several state-of-the-art benchmarks, showing that it outperforms existing methods in both accuracy and computational efficiency.
\footnotetext[1]{{
This manuscript has been authored by UT-Battelle, LLC under Contract No. DE-AC05-00OR22725 with the U.S. Department of Energy. The United States Government retains and the publisher, by accepting the article for publication, acknowledges that the United States Government retains a non-exclusive, paid-up, irrevocable, world-wide license to publish or reproduce the published form of this manuscript, or allow others to do so, for United States Government purposes. The Department of Energy will provide public access to these results of federally sponsored research in accordance with the DOE Public Access Plan(\url{http://energy.gov/downloads/doe-public-access-plan}).}
\renewcommand{\thefootnote}{\arabic{footnote}}}
\end{abstract}

\section{Introduction}
Large-scale pre-trained and fine-tuned models have significantly advanced the performance of deep learning models in assisting various natural language processing and computer vision tasks. However, their deployment often incurs substantial computational and memory costs due to the enormous number of trainable parameters. To address this, parameter-efficient fine-tuning (PEFT) methods have been developed, which modify a subset of model parameters while keeping the rest frozen. Among these, low-rank adaptation (LoRA) \citep{hu2021lora} has emerged as a prominent approach, allowing efficient fine-tuning by injecting low-rank updates into pre-trained model weights. Despite its efficiency, LoRA faces limitations in adaptively distributing the parameter budget across weight matrices, and its performance is sensitive to the choice of hyperparameters \citep{zhang2023adalora}.

Recent works, such as AdaLoRA \citep{zhang2023adalora},  DyLoRA \citep{valipour2023dylora}, and ReLoRA \citep{lialin2023relora}, have attempted to improve LoRA by dynamically adjusting the rank of the low-rank adapters during training. While these methods enhance parameter efficiency, they are constructed as simultaneous descent methods and therefore do not guarantee convergence to optimal low-rank adapters. Methods that guarantee convergence to optimal adapters exist \citep{Schothoefer_2022,schotthöfer2024federateddynamicallowranktraining,zangrando2023rank}. However, these require several gradient tapes per iteration and, therefore, have an intrinsically higher run time per training step.

In this paper, we introduce \ALGNAME{} (\ALGNAMELONG), a novel dynamical low-rank training method for parameter-efficient fine-tuning. \ALGNAME{} leverages the dynamical low-rank approximation theory from matrix differential equations \citep{koch2007dynamical, ceruti2021rank,ceruti2023parallelrankadaptiveintegratordynamical} and exploits the intrinsic low-rank geometry of the weight matrices to allocate the parameter budget across the model adaptively. 
This dynamic allocation is facilitated by a novel training strategy that updates the low-rank factors in parallel, contrasting with other recent methods based on dynamical low-rank approximation theory \citep{Schothoefer_2022,schotthöfer2024federateddynamicallowranktraining,zangrando2023rank}, which require individual gradient tapes computed sequentially per each low-rank factor. 
Instead, \ALGNAME{} requires a single backprop pass over the small-rank adapters, limiting its computational cost and making it faster than popular baselines such as AdaLoRA \citep{zhang2023adalora}. Moreover, \ALGNAME{} maintains the exact orthonormality of the low-rank factors, avoiding the ill-conditioning issues associated with well-known high-curvature challenges arising in low-rank optimization \citep{Schothoefer_2022}.

Through extensive experiments on the GLUE benchmark, Vision Transformers, and Stable Diffusion, we show that \ALGNAME{} outperforms existing PEFT methods both in terms of accuracy and computational efficiency.

Along with the experimental evaluation, we provide a thorough convergence analysis, showing convergence to stationary points under standard assumptions, and a detailed error-bound analysis, demonstrating that \ALGNAME{}'s low-rank adaptation remains close to its full-rank counterpart throughout the training process. This robustness is critical in ensuring that the fine-tuning process does not diverge, even under challenging conditions. 

Overall, the main contributions of this work are as follows:
\begin{itemize}[leftmargin=*,noitemsep,topsep=0pt]
    \item We propose \ALGNAME{}, a dynamical low-rank training method for low-rank adapters that leverages low-rank geometry and matrix differential equations to achieve adaptive parameter allocation.
    \item \ALGNAME{} only requires a single gradient tape and one small-size SVD per training step, making it competitive with existing baselines such as AdaLoRA. 
    \item We provide a convergence analysis and error bound guarantees for \ALGNAME{}, ensuring robust training behaviour and convergence to a stationary point.
    \item Extensive experimental results demonstrate the superior performance of \ALGNAME{} over existing methods, with improved accuracy and training speed.
\end{itemize}

\section{Related Work}
The growing size of neural networks has led to significant computational and memory challenges during both training and deployment. Several strategies have been proposed to mitigate these issues, including sparsification \citep{guo2016dynamic, molchanov2017pruning, he2017channel} and quantization \citep{wu2016quantized, courbariaux2016binarized}. Among these, layer factorization has gained traction as an effective approach to reducing memory requirements. Layer factorization techniques have been applied successfully in both pre-training \citep{wang2021pufferfish, khodak2021initialization, Schothoefer_2022, schotthöfer2024federateddynamicallowranktraining, zangrando2023rank, zhao2024galore} and fine-tuning scenarios \citep{hu2021lora, valipour2023dylora, zhang2023adalora, hayou2024lora, zhao2024galore, lialin2023relora}, demonstrating their versatility across various tasks. 

Low-rank adapters such as LoRA \citep{hu2021lora} have become a standard approach for PEFT by applying low-rank corrections to pre-trained models. LoRA introduces a low-rank decomposition to the weight matrices of the model, significantly reducing the number of trainable parameters while preserving performance. Despite its efficiency, LoRA's effectiveness heavily relies on the selection of hyperparameters such as learning rates and parameter budgets \citep{zhang2023adalora, hayou2024lora}. These limitations have spurred the development of rank-adaptive methods.
AdaLoRA \citep{zhang2023adalora} is a popular extension of LoRA, which dynamically adjusts the rank of the low-rank adapters during training. By incorporating an orthogonality regularizer and SVD-like adaptation, AdaLoRA aims to address the challenges of rank selection and adaptation. It outperforms static low-rank methods by automatically allocating parameter budgets based on the importance of each matrix component. DyLoRA \citep{valipour2023dylora} provides an alternative approach that hierarchically adjusts the rank during training, demonstrating that higher-rank adapters can lead to better performance than very low-rank ones. DoRA \citep{Mao2024DoRAEP} proposes to sample a set of rank-1 updates for each LoRA layer and to combine them into a rank-$r$ update. Optimal rank-1 components are chosen during fine-tuning using an importance score based on the norm of the LoRA layer.

Beyond fine-tuning, low-rank methods have been successfully applied during the training and pre-training phases of neural networks. Techniques such as Pufferfish \citep{wang2021pufferfish}, intrinsic dimension reduction \citep{aghajanyan2020intrinsic}, and DLRT \citep{Schothoefer_2022} suggest that large deep learning models have an inherently low intrinsic dimensionality, making them amenable to low-rank approximations. These methods propose reducing the number of parameters during training, potentially improving both efficiency and generalization.
Recent works in dynamical low-rank training have explored the use of geometric properties of the low-rank parameter space to improve training stability and convergence. For example, the geometry-aware training approach for tensor layers in Tucker format \citep{zangrando2023rank} dynamically adapts the rank of the factorized layers, ensuring robust convergence even when the initial rank estimation is inaccurate. This method leverages the Riemannian geometry of the parameter space to avoid the ill-conditioning commonly encountered in low-rank training. 
ReLoRA \citep{lialin2023relora} introduces a parameter-efficient training method by using multiple low-rank updates to effectively train high-rank networks. This method allows training larger models with significant memory savings and training speed improvements compared to conventional methods. GaLore \citep{zhao2024galore} introduces a memory-efficient training strategy by projecting gradients onto a low-rank subspace. This approach achieves significant memory savings while maintaining performance.

\section{Low-rank optimization: what can go wrong}\label{sec:what_can_go_wrong}
This section aims to discuss the nature of the critical points and optimization trajectories obtained when using gradient-based strategies for low-rank parameters, and why a straightforward application of gradient-based steps to factorized adapters may lead to suboptimal results.

Consider a neural network layer of the form
\begin{align}\label{eq_usv_adapter}
    \fz = \sigma(W_{\mathrm{pt}}\fx + USV^\top\fx),
\end{align}
where $\sigma$ is an arbitrary activation function, $W_{\mathrm{pt}} \in \mathbb{R}^{n \times n}$ are the frozen pre-trained weights, and $U, V \in \mathbb{R}^{n \times r}$, $S \in \mathbb{R}^{r \times r}$ are the rank-$r$ adapter weights, with input $\fx$. For simplicity, we omit the bias term. Low-rank adapters of the form $W = USV^\top\in\mathbb{R}^{n\times n}$ have gained popularity in recent approaches such as \citep{zhang2023adalora}, although our discussion extends to other equivalent formulations like $W = AB$ \citep{hu2021lora}.
The objective of the training process is to minimize a loss function $\mathcal{L}(W)$ to find an optimal adapter weight $W_{\star}$. For full-rank matrices ($r = n$), optimality requires that $\nabla_W \mathcal{L}(W_{\star}) = 0$. However, when $r < n$, this condition is generally unattainable due to the reduced parameter space. In this scenario, we seek a matrix $W_{\star}$ that is locally optimal within the low-rank parameter space, meaning no further reduction in the loss function $\mathcal{L}$ is possible in the neighborhood of $W_{\star}$.
A necessary condition for local optimality can be expressed as $P(W_{\star}) \nabla \mathcal{L}(W_{\star}) = 0$, see e.g., \citep[Theorem~3.4]{sato2021riemannian}. For orthonormal $U$ and $V$, the projection operator $P(USV^{\top})Z := UU^{\top}Z(I - VV^{\top}) + ZVV^{\top}$ represents the \emph{orthonormal} projection of $Z$ onto the tangent space at $USV^{\top}$. If $W_\star$ is not a saddle point, then this condition ensures that no search direction within the tangent space of $W_{\star}$ can further decrease the loss. See also \Cref{app:local_optimality}. Note that this only guarantees local optimality, a limitation shared by all gradient-based optimizers.

Current training methods for low-rank adapters aim to optimize the low-rank factors with a single backpropagation pass to compute all the required gradients simultaneously. This boils down to integrating the following gradient flow equations for each individual factor
\begin{align}\label{eq_fine_tune_grad_flow}
\begin{aligned}
    \dot{U} =& -\nabla_U\mathcal{L} = -(\nabla_W\mathcal{L})VS\,, \\
    \dot{V} =& -\nabla_V\mathcal{L} = -(\nabla_W\mathcal{L})^{\top}US^{\top}\,, \\
    \dot{S} =& -\nabla_S\mathcal{L} = -U^{\top}\nabla_W\mathcal{L}V\,,
\end{aligned}
\end{align}
where we use the chain rule and the decomposition $W = USV^\top$ to derive the expressions for $\nabla_{U,S,V}\mathcal{L}$. Here, we have omitted the dependence on the time variable $t$, i.e., $U,S,V = U(t),S(t),V(t)$ for improved readability, and we use dots to denote time derivatives. An explicit time discretization with a time step size equal to the learning rate $\lambda$ leads to the simultaneous gradient descent updates commonly employed in conventional training methods for LoRA. At first glance, this procedure appears effective, as a single update step will decrease the loss \emph{if we freeze all but one of the low-rank factors}. However, in practice, LoRA training modifies all low-rank factors \emph{simultaneously}, raising the question of how this affects the overall optimization trajectory.

To address this, consider the evolution equation for $W = USV^\top$, derived directly using the chain rule and \cref{eq_fine_tune_grad_flow}
\begin{align}\label{eq:gradflow_simgd}
    \dot{W} =\,& \dot{U}SV^{\top}+U\dot SV^{\top}+US\dot V^{\top}\nonumber\\
    \stackrel{(\ref{eq_fine_tune_grad_flow})}{=}\,& - \nabla_W\mathcal{L}VS^2V^{\top} + UU^{\top}\nabla_W\mathcal{L}VV^{\top} -U(S^{\top})^2U^{\top}\nabla_W\mathcal{L} =: -\widehat{P}(W)\nabla_W\mathcal{L}\,.
\end{align}
The operator $\widehat P(USV^{\top})Z:=  ZVS^2V^{\top} -UU^{\top}ZVV^{\top} +U(S^{\top})^2U^{\top}Z$ again represents a linear mapping onto the tangent space at $W = USV^\top$. Note that this projection depends on the individual low-rank factors $U$, $S$, and $V$, but we use the notation $\widehat{P}(W)$ for brevity.
Simultaneous descent methods approximate the gradient flow of \cref{eq:gradflow_simgd}, which ideally converges to a solution $W_{\star}$ such that $\widehat{P}(W_{\star})\nabla\mathcal{L}(W_{\star}) = 0$. However, $\widehat{P}$ is orthogonal only when $U$ and $V$ are orthonormal and $S = I$, where $I$ denotes the identity matrix. If these conditions are not met, the resulting optimization process may not find an optimal weight within the low-rank parameter space. This is because $\widehat P(W_\star)\nabla \mathcal L(W_\star) = 0$ does not imply $P(W_\star)\nabla \mathcal L(W_\star) = 0$, thus there could still be some decrease direction along the tangent space as depicted in \Cref{fig:CompareSGDandRGD}.

To construct methods that converge to an optimal low-rank solution, an alternative approach is to evolve the adapter $W$ along the projected gradient flow $\dot{W}(t) = -P(W(t))\nabla \mathcal{L}(W(t))$. In this case, the corresponding evolution equations for the low-rank factors take the form
\begin{align}\label{eq_fine_tune_grad_flow_dlrt}
\begin{aligned}
    \dot{U} &= -(I - UU^{\top})\nabla_W\mathcal{L}VS^{-1}\,, \\
    \dot{V} &= -(I - VV^{\top})\nabla_W\mathcal{L}^{\top}US^{-\top}\,, \\
    \dot{S} &= -U^{\top}\nabla_W\mathcal{L}V\,,
\end{aligned}
\end{align}
assuming that $U$ and $V$ are orthonormal \citep{koch2007dynamical}. 

\begin{figure}
\centering
\begin{subfigure}{0.48\textwidth}
\centering
\includegraphics[width = \textwidth]{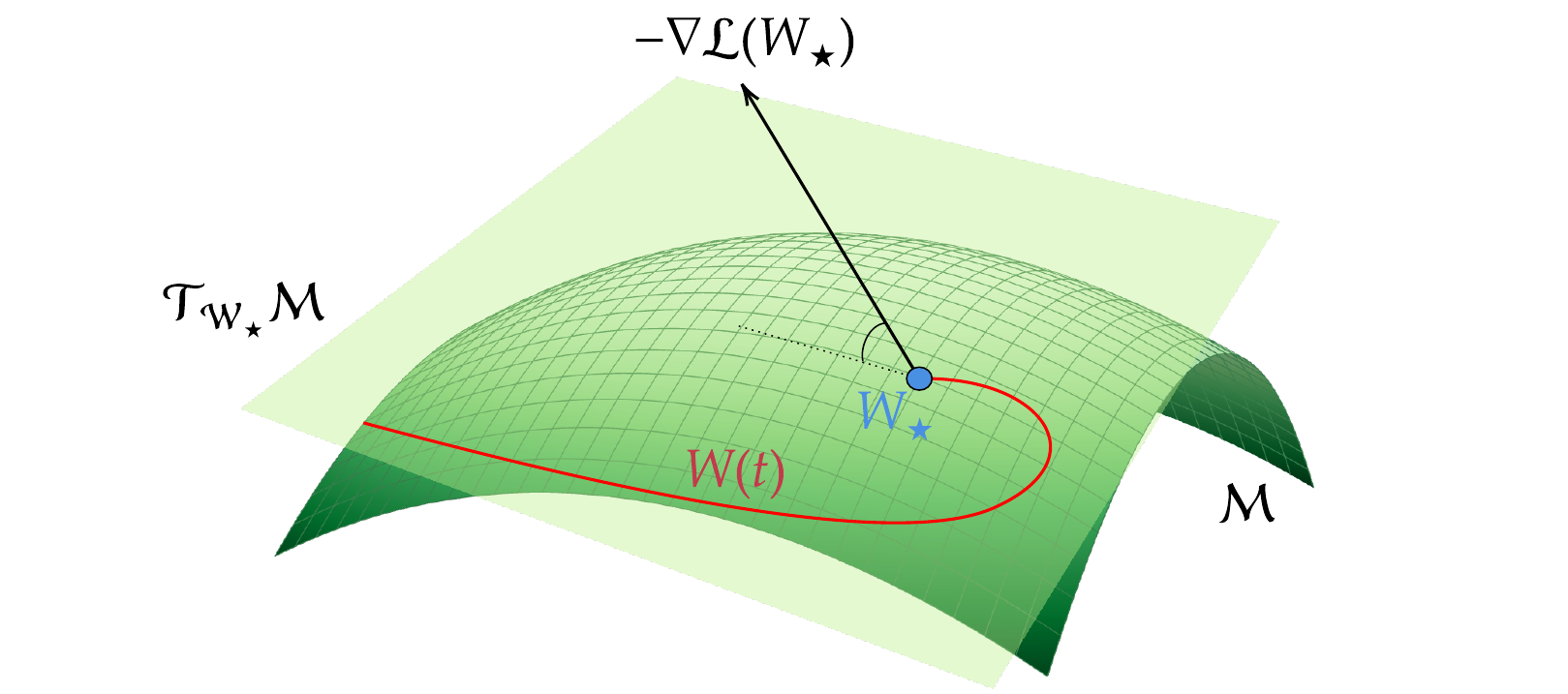}
\caption{Gradient flow of simultaneous gradient descent. }
\label{fig:left}
\end{subfigure}
\begin{subfigure}{0.48\textwidth}
\centering
\includegraphics[width = \textwidth]{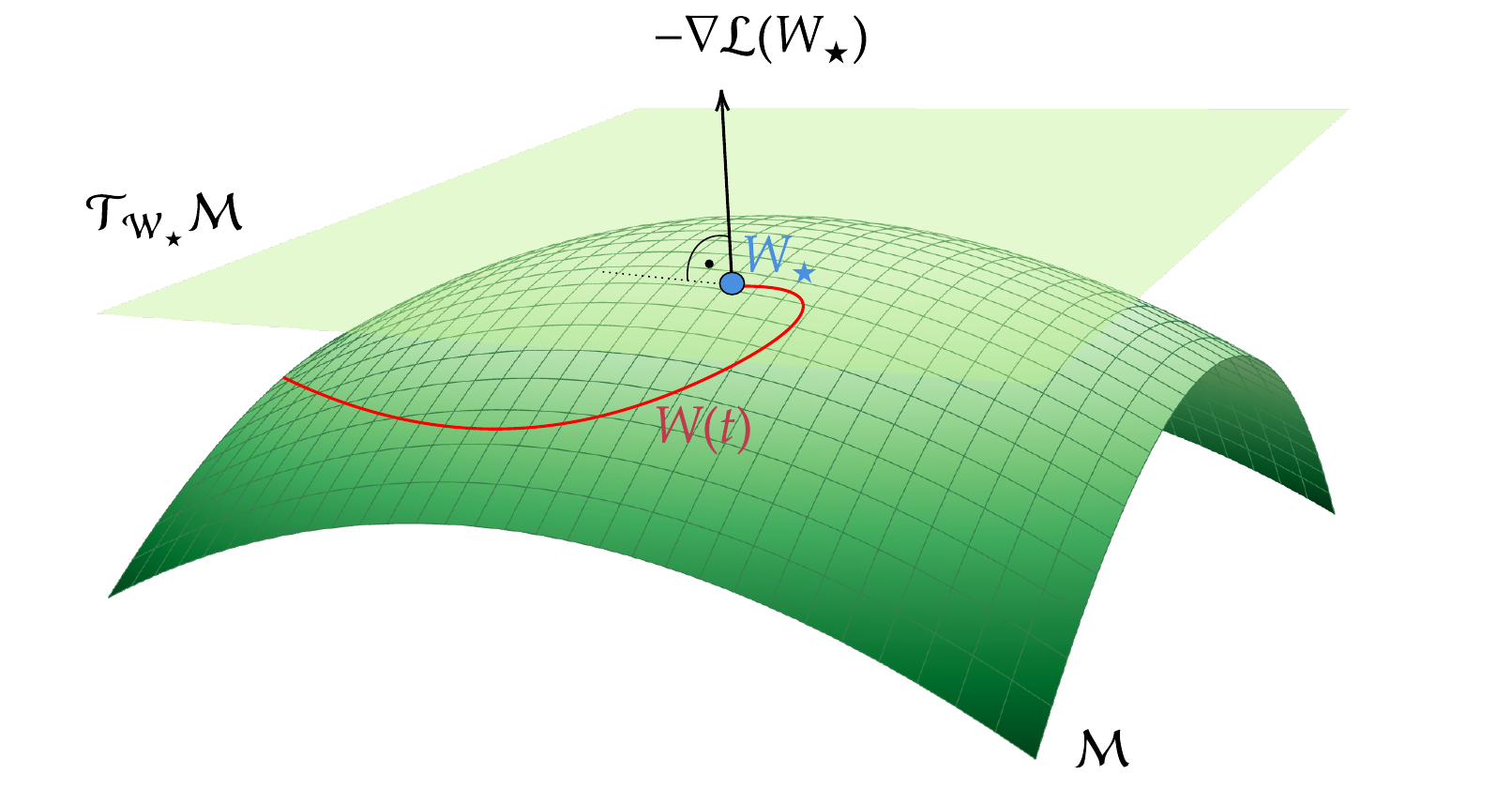}
\caption{Riemannian gradient flow. }
\label{fig:right}
\end{subfigure}
\caption{Illustration of simultaneous vs. Riemannian gradient flow. The projector of the simultaneous gradient flow converges to a point $W_{\star}$ such that $\widehat P(W_{\star})\nabla\mathcal{L} = 0$. Since $\widehat P$ is not an orthogonal projection, the gradient is not orthogonal to the tangent plane, i.e., $W_{\star}$ is suboptimal. For Riemannian gradient flows, the adapter converges to a point $W_{\star}$ such that $P(W_{\star})\nabla\mathcal{L} = 0$. Since $P$ is the orthogonal projection on the tangent space, $W_{\star}$ is a local optimum, i.e., no directions exist in the tangent space $\mathcal{T}_{W_{\star}}\mathcal{M}$, which further decrease the loss. Here, $\mathcal{M}$ denotes the space of low-rank adapters, and $\mathcal{T}_{W_{\star}}\mathcal{M}$ represents the tangent space at the optimal adapter weight $W_{\star}$.}
\label{fig:CompareSGDandRGD}
\end{figure}

While the evolution defined in \cref{eq_fine_tune_grad_flow_dlrt} guarantees convergence to an optimal low-rank adapter, the presence of the $S^{-1}$ term on the right-hand side introduces stiffness in the gradient flow. This stiffness can significantly slow down convergence, especially when the singular values in $S$ vary greatly in magnitude. Robust solutions to address the stiffness problem and ensure convergence without being hindered by the $S^{-1}$ term have been proposed \cite{Schothoefer_2022, zangrando2023rank, schotthöfer2024federateddynamicallowranktraining}. However, these methods require multiple gradient tape evaluations per training update, which makes them computationally more expensive than traditional LoRA training techniques with simultaneous updates.

To overcome these limitations, we propose \ALGNAME{}, a novel training method for low-rank adapters 
that only requires a single gradient tape evaluation per update while ensuring convergence to an optimal low-rank solution, following the projected gradient flow in \cref{eq_fine_tune_grad_flow_dlrt}. This approach retains the computational efficiency of conventional LoRA methods while achieving comparable or even superior performance. By eliminating the need for multiple gradient tape evaluations, \ALGNAME{} offers a practical and scalable solution for training low-rank adapters effectively.


Before presenting the proposed training method, we illustrate different behaviours of different low-rank adaptation strategies using a toy example. Consider the problem of matching a rank-$r$ target matrix $W_{\text{target}} \in \mathbb{R}^{n \times n}$ with a low-rank adapter $W$, formulated as:
\begin{align}\label{eq_matrix_regression}
    \min_{W} \frac{1}{2} \norm{W_{\text{target}} - W}_F^2\, .
\end{align}
We compare the convergence behavior of six different training methods for $n = 5000$, $r = 5$, and a learning rate of $\lambda = 0.1$:

\begin{minipage}[c]{0.5\linewidth}
\begin{enumerate}[leftmargin=*,noitemsep,topsep=0em]
    \item Full fine-tuning (FT) (blue),
    \item DLRT from \citep{Schothoefer_2022} (orange),
    \item The proposed \ALGNAME{} method (green),
    \item Fixed rank LoRA from \cite{hu2021lora} (red),
    \item AdaLoRA from \citep{zhang2023adalora} (brown),
    \item Fixed rank  AdaLoRA (purple).
\end{enumerate}
\end{minipage}
\hfill
\begin{minipage}[c]{0.5\linewidth}
\centering \includegraphics[width=1.0\textwidth]{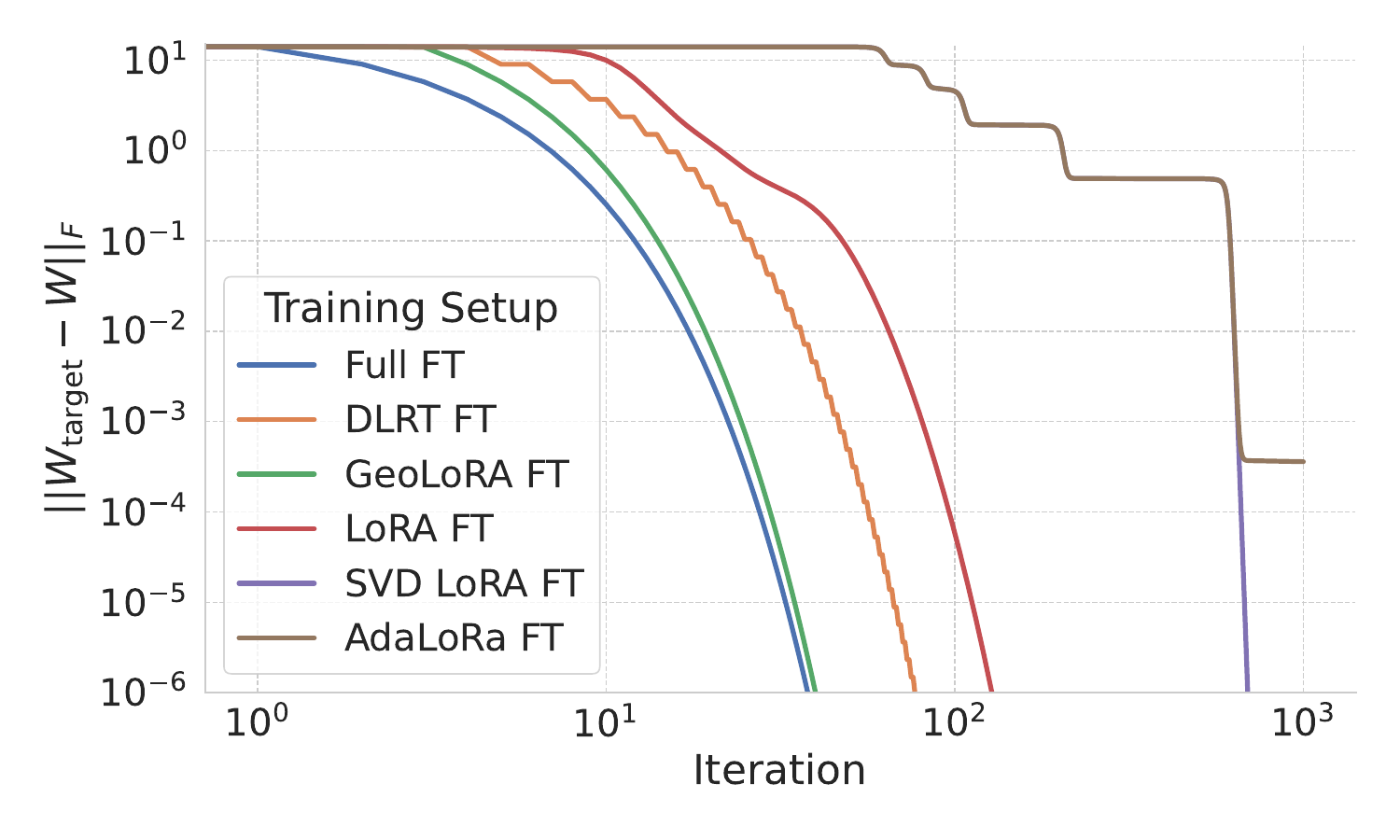}
\end{minipage}

In this experiment, the fixed rank approaches (4, 6) use a rank of 50. All adapters $W$ are initialized to zero, with $S_0 = 0$ for the SVD-based methods (2, 3, 5, 6) and $B = 0$ for LoRA-based methods (4).

The results show that the proposed \ALGNAME{} method (3) converges as quickly as full fine-tuning. In contrast, method (2) (DLRT) takes approximately twice as long due to the sequential updates of the basis and coefficient matrices\footnote{The loss plateaus appear since the loss value remains constant during a basis update and only decreases during a coefficient update. This accounts for the fact that two gradient tapes need to be computed.}. LoRA-type methods (4, 5, 6) exhibit slower convergence due to the suboptimality of the underlying gradient flow defined in \cref{eq_fine_tune_grad_flow}. AdaLoRA (5) solves the same gradient flow as method (6) but plateaus at a loss of $5 \times 10^{-4}$, corresponding to the regularization parameter for the terms $\Vert{U^\top U - I}\Vert_F^2 + \Vert{V^\top V - I}\Vert_F^2$ that enforce the orthonormality of $U$ and $V$. We had to fix the minimum rank to 5 for AdaLoRA to prevent stalling of the optimization due to rank underestimation, an issue not observed in methods (2) and (3), where rank augmentation avoided this problem. 

\section{The proposed method}\label{sec:method}

In this section, we introduce \ALGNAME{} (\ALGNAMELONG) a novel low-rank fine-tuning method that integrates \textbf{rank adaptivity}, \textbf{low-rank optimality}, and \textbf{memory and computational efficiency}. Our method builds upon the parallel geometric low-rank integrator originally designed for model order reduction in high-dimensional PDEs \citep{ceruti2023parallelrankadaptiveintegratordynamical}, and it is equipped with loss descent, approximation bounds, and convergence guarantees. 
%
Notably, it improves upon existing dynamical low-rank methods, e.g. \citep{schotthöfer2024federateddynamicallowranktraining, zangrando2023rank, Schothoefer_2022} by updating basis and coefficients \textit{in parallel} opposed to a \textit{sequential} basis update and coefficient step. Moreover, only a single backward pass per iteration step is required through a novel evaluation strategy of robust gradients, thus doubling the wall-time performance. \ALGNAME{} is, therefore, the first low-rank training method solving the optimal gradient flow \cref{eq_fine_tune_grad_flow_dlrt} with training times per iteration comparable to standard simultaneous descent approaches to low-rank adaptation such as LoRA and AdaLoRA \citep{hu2021lora,zhang2023adalora}. In particular, it improves upon these methods  
by providing robustness and convergence guarantees and demonstrating an overall improved performance and robustness to hyperparameters in numerical examples.


Starting from an initial factorization $U_0,V_0,S_0$ with initial rank $r_0$, where $S_0$ is diagonal and full-rank, \ALGNAME{} performs the following steps (also summarized in Algorithm~\ref{alg_efficient_TDLRT}):

\begin{algorithm}[t]
\DontPrintSemicolon
\SetAlgoLined
\SetKwInOut{Input}{Input}
\SetKwComment{Comment}{$\triangleright$\ }{}

\Input{Initial orthonormal bases $U,V\in\mathbb{R}^{n\times r}$ and diagonal $S\in\mathbb{R}^{r\times r}$;\;
$\tau$: singular value threshold for rank truncation;\;
$\lambda$: learning rate.
}
Evaluate $\mathcal{L}(USV^\top)$\tcc*{Forward evaluate}
$G_U\gets\nabla_{U}\mathcal{L}(USV^\top);\,G_S\gets\nabla_{S}\mathcal{L}(USV^\top);\,G_V\gets\nabla_{V}\mathcal{L}(USV^\top)$ \tcc*{Backprop}

$\left\{
\begin{array}{l}
S^\text{new} \gets \texttt{optimizer\_step}(S, G_S, \lambda)   \\
K^\text{new} \gets \texttt{optimizer\_step}(US, G_U S^{-\top}, \lambda) \\
L^\text{new} \gets \texttt{optimizer\_step}(VS^\top, G_V S^{-1}, \lambda) \quad 
\end{array}
\right.$ {\tcc*{in parallel}}

$\left\{
\begin{array}{l}
\newU\gets\texttt{  basis\_augmentation}(U, K^\textup{new}) \\
\newV\gets\texttt{  basis\_augmentation}(V, L^\textup{new})
\end{array}
\right.${\tcc*{in parallel}}
$\augS\gets 
\begin{bmatrix}
    S^\text{new}  & L^{\text{new},\top}\newV \\
     \newU^\top K^{\text{new}} & 0 \\
\end{bmatrix}\in\mathbb{R}^{2r\times 2r}
$\tcc*{Assemble new coefficient matrix} 
$U,S,V, S^{-1} \gets ${\tt truncation}$(\augS, [U \mid \newU], [V \mid \newV]  )$

\caption{Single iteration of \ALGNAME{}. \\The functions \texttt{optimizer\_step}, \texttt{basis\_augmentation}, and \texttt{truncation} are detailed in \Cref{alg_helper} in the appendix. }\label{alg_efficient_TDLRT}
\end{algorithm}

1. \textbf{Perform a (stochastic) gradient step} to compute the new variables  
$S^\text{new}\in\mathbb{R}^{r_0\times r_0}$, $L^{\text{new}}\in\mathbb{R}^{n\times r_0}$, and $K^{\text{new}}\in\mathbb{R}^{n\times r_0}$, as follows: 
\begin{align}\label{eq_kls}
\begin{aligned}
    &S^\text{new} = S_0 - \lambda \nabla_S\mathcal{L}(U_0 S_0 V_0^\top) \\
    &K^\text{new} = U_0S_0 - \lambda \nabla_U\mathcal{L}(U_0 S_0 V_0^\top)S_0^{-\top} \\
    &L^\text{new} = V_0S_0^\top - \lambda \nabla_V\mathcal{L}(U_0 S_0 V_0^\top)S_0^{-1}.
\end{aligned}
\end{align}
We will see in Theorem~\ref{theo_robust_error_bound} that using these variables mitigates the stiffness of the system in \cref{eq_fine_tune_grad_flow_dlrt} while approximating the optimal gradient flow. Note that the right-hand side gradients $\nabla_U\mathcal{L}$, $\nabla_V\mathcal{L}$, and  $\nabla_S\mathcal{L}$ can be evaluated with only one backward pass through the network using standard algorithmic differentiation techniques, halving the computational cost of existing geometric methods such as \citep{Schothoefer_2022, zangrando2023rank}. Evaluation of the inverse $S_0^{-1}$ induces no computational overhead since $S_0$ is diagonal at the start of each iteration. 

2.  \textbf{Augment the current bases} $U_0,V_0$ 
to twice their rank using the gradient dynamics of the loss, which is encoded in $K^{\text{new}}$ and $L^{\text{new}}$, i.e.
\begin{align}\label{eq_basis_aug}
    \augU = [U_0,\newU] = \textup{ortho}([U_0,K^{\text{new}}])\in\mathbb{R}^{n\times 2r_0}\quad \text{and}
\quad \augV = [V_0,\newV] = \textup{ortho}([V_0,L^{\text{new}}])\in\mathbb{R}^{n\times 2r_0}.
\end{align}
Here ``$\textup{ortho}$'' denotes a column orthonormalization procedure such as the QR-algorithm. 
This augmentation step provides the low-rank adapter with a larger search space to increase the rank of its adaptation if the initial rank-guess $r_0$ was insufficient to fully capture the problem. Doubling the rank implies that in $\log(n)$ training iterations any rank can be captured by a rank one initialization, eliminating the need for tuning $r$ as a hyperparameter, see Figure~\ref{fig_vit_rafer_init_r}. 

3. \textbf{Assemble the augmented coefficient matrix} 
\begin{align}\label{eq_S_step}
   \augS\gets 
\begin{bmatrix}
    S^\text{new}  & L^{\text{new},\top}\newV \\
     \newU^\top K^{\text{new}} & 0 \\
\end{bmatrix}\in\mathbb{R}^{2r_0\times 2r_0}
\end{align}
where we obtain the block entries $S^\text{new}$, $L^{\text{new}}$, and $K^{\text{new}}$ from \cref{eq_kls}. 

4. \textbf{Truncate redundant singular values} $s_i$ of $\augS$ and the corresponding singular vectors, i.e. basis functions of $\augU,\augV$, using the criterion
\begin{align}\label{eq_truncation_threshold}
    \sum_{i=r_1+1}^{2r} s_i^2  < \vartheta,
\end{align}
where $r_1$ is the new rank of the factorization and $\vartheta$ is a tresholding hyperparameter. The singular values $s_i$ are obtained via the SVD of $\augS = P\Sigma Q^\top\in\R^{2r_0\times 2r_0}$. Then  we determine the new factorization as $S_1 = \textup{diag}(s_1,\dots,s_{r1})\in\mathbb{R}^{r_1\times r_1}$, $U_1 = \augU P_{(1,\dots,r_1)}\in\mathbb{R}^{n\times r_1}$ and $V_1 = \augV Q_{(1,\dots,r_1)}\in\mathbb{R}^{n\times r_1}$.
{The truncation threshold} $\vartheta$ is chosen relative to the nuclear norm of the specific layer's current singular values, i.e. $\vartheta=\tau\Vert{\augS}\Vert_F^2$. Other norms, such as the $1$-norm of the singular values $s_i$, are possible as well. Thus, the truncation threshold determines how aggressively to prune each layer individually. Analogously, the following global threshold similar to the one used in e.g.\ \citep{zhang2023adalora,ghadiri2023approximately,Idelbayev_2020_CVPR} 
\begin{align}\label{eq_global_truncation_threshold}
   \sum_{\ell=1}^L \sum_{i=\ell+1}^{2r_\ell} s_{i,\ell}^2  < \frac{\tau}{1-\tau} \sum_{\ell=1}^L \sum_{i=1}^{r_{1,\ell}} s_{i,\ell}^2,
\end{align}
can be considered by summing the singular values across all the layers 
$\ell=1,\dots,L$.
To directly control the parameter budget, order $ s_{i,\ell}^2$ by descending by magnitude and selecting the largest ones first until either \cref{eq_global_truncation_threshold} is violated or the budget is depleted. 
\subsection{Parameter initialization}
\textbf{LoRA-type adapters}~\citep{hu2021lora} initilize the low rank matrices $B,A$ with zero initialization of $B$, and Gaussian initialization of $A$. This ensures that the fine-tuning indeed starts at the pretrained state of the network, i.e., $ \sigma(W_{\mathrm{pt}}\fx + \frac{\alpha}{r}A_{0}B_{0}^{\top}\fx)= \sigma(W_{\mathrm{pt}}\fx)$. 
For consistency with this initialization, the bases $U_0$ and $V_0$ can be initialized as random but orthonormal, whereas the coefficient matrix $S_0$ has zero-initialization. In this first solve of \cref{eq_kls}, we set $S_0^{-1}$ as the identity matrix. As a result, the first solve of \cref{eq_kls} is inconsistent with the optimal dynamics of \cref{eq_fine_tune_grad_flow_dlrt}. However, all following iterations evolve the low-rank trajectory according to the optimal gradient flow. Since in the first iterations of a LoRA fine-tuning, the adapter is typically close to the original solution but far from the fine-tuning optimum, this inconsistency is irrelevant to the overall convergence behavior of the method. Alternatively, the required gradients can be computed with three individual gradient tapes in the first iteration, which does not require the inversion of $S_0$.

The proposed method can readily be used for \textbf{dynamic low-rank compression} \citep{Schothoefer_2022, zangrando2023rank} of pre-trained networks, where we consider a layer $\fz = \sigma(W\fx)$,  and approximate $W\approx U_0S_0V_0^\top$. Here, the initial parameters $U_0,S_0,V_0$ are obtained by a truncated singular value composition of $W$. 

Finally, for \textbf{low-rank pre-training} of an untrained network with given architecture, i.e. predetermined layer dimensions $n$, but unknown rank $r$, the factors $U_0,V_0$ are initialized randomly, but orthonormal and $S_0$ is initialized randomly, but diagonal for easy initialization of $S_0^{-1}$. 

The basis augmentation step is able to truncate to any rank within one iteration and able to augment the basis to any rank within a logarithmic number of iterations, thus the method is robust w.r.t the choice of the initial rank. We provide an ablation study in \Cref{fig_vit_rafer_init_r}.

\subsection{Analysis}
In the following, we analyze \Cref{alg_efficient_TDLRT} under the general assumption that $\mathcal{L}$ is L-smooth with constant $L$ and bounded with constant $B$.

For brevity of exposition we denote $W^{r}_t = U_{t}S_{t}V_{t}^\top$ as the low-rank factorization at iteration $t$ evaluated with \Cref{alg_efficient_TDLRT}, whereas $W_t$ denotes the full-rank solution obtained by ``full fine-tuning'' with stochastic gradient descent. Further, we denote by $f(W^{r}_{t},\xi_t)$ the stochastic gradient of the network loss $\mathcal{L}$ w.r.t the low-rank weight $W^{r}_t$ at iteration $t$, obtained by batch-gradient descent. The i.i.d random variable $\xi_t$ models the randomness in the training data batch at iteration $t$.
Lastly, recall that $P(W^{r}_{t})Z$ denotes the orthogonal projection of the matrix $Z$ onto the tangent plane of the manifold of rank-$r$ matrices at the point $W^{r}_{t}$.

\textbf{\Cref{alg_efficient_TDLRT} is an optimizer on low-rank manifolds}: \Cref{theo_descend_direction} shows, that the proposed scheme with stochastic gradients indeed decreases the training loss in each iteration, while optimizing on a manifold, and \Cref{theo_convergence} yields stochastic convergence to a locally optimal stationary point.

\begin{theorem}[Stochastic descent estimate]\label{theo_descend_direction} \Cref{alg_efficient_TDLRT} with stochastic (mini-batch) gradients fulfills
\begin{align}\label{the0_2_result}
\begin{aligned}
   \mathbb{E}_{\xi_{t+1}}[\mathcal{L}(W^{r}_{t+1})]  \leq \mathcal{L}(W^{r}_{t}) -\lambda\left(1-\frac{L\lambda^2}{2}\right)\mathbb{E}_{\xi_1}[\Vert P(W^{r}_{t})f(W^{r}_{t},\xi_t) \Vert^2] +L\mathbb{E}_{\xi_1}[ \|W^{r}_{t+1}- \augWr_t \|]\,.
\end{aligned}
\end{align}
where $W^{r}_{{t}}$, $\augWr_{t}$, $W^{r}_{t+1}$ are the low-rank weight matrices at the start of iteration $t+1$, before, and after the truncation step, respectively.
\end{theorem}
The proof is provided in \Cref{app_descend_direction}. The above theorem yields a loss descent guarantee up to the two last terms on the right-hand side. The first term of the right hand side induces the step size criterion  $\lambda\leq \frac{2}{L}$, which resembles the step size criterion of full gradient descent, where the two right hand side terms read $ -\lambda(1-\frac{L{\lambda}}{2})\Vert f(W_{t}) \Vert^2$. This shows that the low-rank optimizer allows similar learning rates as a full fine-tuning setup, eliminating the need for the $\frac{\alpha}{r}$ scaling parameter of LoRA. The last term models the error introduced by the truncation step and is bounded by the user-determined cutoff threshold $\vartheta$, as $\mathbb{E}_{\xi_1}[ \|W^{r}_{t+1}- \augWr_t \|]\approx \vartheta$. As the solution stabilizes in rank, the error term vanishes, and we obtain the following main convergence result: 
\begin{theorem}[Convergence]\label{theo_convergence}
    Let $\mathcal{L}\geq 0$ and $W^{r}_{1},\dots,W^{r}_T$ be the solutions generated by \Cref{alg_efficient_TDLRT} over $T$ steps.  Let the learning rate sequence $\{\lambda_t\}$ satisfy the Robbins-Monro conditions
    \[
    \textstyle{\sum_t \lambda_t =+\infty \qquad \sum_t \lambda_t^2 < +\infty \, ,}
    \]
    and each step $\lambda_t$ the step size restriction $\lambda_t\leq \frac{2}{L}$.
Further assume $\sum_{t=1}^{T-1}\mathbb{E}[\|W^{r}_{t+1} - \augWr_{t} \|] \leq D < \infty$, i.e.  after some time, the solution $W^{r}_t$ is contained in a manifold of rank $r$. Then we have
       \begin{align*}
        \liminf_{T\rightarrow\infty} \mathbb{E}[\Vert P(W^{r}_{t})f(W^{r}_{t}) \Vert^2] = 0\,,
    \end{align*}
    where the expected value is taken over all $\xi_t$.
\end{theorem}
The proof is provided in \Cref{app_convergence}. Additionally, the solution trajectory of \Cref{alg_efficient_TDLRT} is close to the (full-rank) trajectory of the dynamical system 
\begin{align}\label{eq_cont_gradient_flow_full_rank}
    \dot W(t)=-\nabla_W \mathcal{L}(W(t)),
\end{align}
i.e., the gradient flow of full training or fine-tuning:
\begin{theorem}[Error-bound]\label{theo_robust_error_bound}
For an integer $k$, let $t=k\lambda$. Let $W(t)$ be the solution of \cref{eq_cont_gradient_flow_full_rank}, and let $W^r_t$ be the factorized low-rank solution after $k$ steps with \Cref{alg_efficient_TDLRT}.
    Assume that for any $Z$ in a neighborhood of $W(t)$, we have $\norm{(I-  P(Z))\nabla\mathcal{L}(Z)}<\varepsilon$, i.e., the gradient flow is close to $T_Z \mathcal M_r$. Then, 
\begin{equation}\label{eq:approx}
   \norm{W(t)-W^r_t}\leq  c_{1}\varepsilon + c_{2}\lambda +c_{3}\vartheta/\lambda\,.
\end{equation}
Moreover, let $W_{RF}(t)$ denote the solution of the Riemannian flow of \cref{eq_fine_tune_grad_flow_dlrt}. Then, 
\begin{equation}\label{eq:approx2}
   \norm{W_{RF}(t)-W^r_t}\leq  c_{4}\varepsilon + c_{2}\lambda +c_{3}\vartheta/\lambda
\end{equation}
where the constants $c_{1}$, $c_{2}$,  $c_{3}$,  $c_{4}$ depend only on $L$ and $B$.
\end{theorem}

The proof is provided in \Cref{app_robust_error_bound}. We refer to \Cref{app_gradient_trick} for an interpretation of \Cref{alg_efficient_TDLRT} as an integrator of the gradient flow of \cref{eq_cont_gradient_flow_full_rank}.


Finally, we point out that the single-layer case discussed so far is  not restrictive, and all the theoretical results above can be directly transferred to the multilayer setting by means of the following proposition:
\begin{proposition}[Global structure preservation]\label{prop:structure_preservation}
The application of \Cref{alg_efficient_TDLRT} for multiple LoRA layers corresponds to the numerical integration of an augmented single matrix system on the adjacency matrix of the computational graph
\[
\dot{ \mathcal W} = - P(\mathcal W)\Pi \nabla \mathcal L(\mathcal W)
\]
Where $\Pi$ is a linear projection that depends only on the structure of the neural network architecture. Moreover, the application of \Cref{alg_efficient_TDLRT} to this system, leads to the global truncation strategy proposed in \Cref{sec:method}.
\end{proposition}
The proof of \Cref{prop:structure_preservation} can be found in \Cref{app_full_network_lr} together with the relative derivation of the global truncation strategy. 

\section{Numerical Results}

\paragraph{DeBERTa for GLUE.} 
We evaluate the performance of \ALGNAME{} by fine-tuning the 183 million parameter transformer DeBERTaV3-base \citep{he2023debertav3improvingdebertausing} on the GLUE Benchmark \citep{wang2019gluemultitaskbenchmarkanalysis} and compare the results in \Cref{tab_results}. For details on the methods, implementation, hyperparameter choices, and benchmark setup, please refer to \Cref{app_glue}. In most cases, \ALGNAME{} outperforms other methods on the benchmark, achieving better metrics with significantly fewer trainable parameters. This reduction in trainable parameters allows \ALGNAME{} to process substantially more samples during training and evaluation compared to AdaLoRA.

\paragraph{Performance analysis.}
The proposed method from \Cref{alg_efficient_TDLRT} combines low-rank optimality guarantees with significant computational efficiency gains compared to existing low-rank optimization methods, as shown in \Cref{tab_method_comparison}. For a rank $r$ adapter, the computational cost of gradient evaluation (i.e., \cref{eq_kls}) is equivalent to that of AdaLoRA, which updates $U$, $S$, and $V$ directly, and is similar to a standard LoRA update. The cost of basis augmentation is $\mathcal{O}(n r^2)$ due to the QR decomposition in \cref{eq_basis_aug}, comparable to evaluating the orthonormality regularization terms in AdaLoRA. Rank truncation is performed via an SVD of $S$ at a cost of $\mathcal{O}(r^3)$, where typically $r \ll n$. The complexity analysis shows comparable per-iteration costs for LoRA, AdaLoRA, and \ALGNAME{}. In \Cref{tab_results}, we also report the number of iterations computed per second during training and evaluation for both \ALGNAME{} and AdaLoRA, demonstrating that \ALGNAME{} outperforms AdaLoRA across almost all GLUE benchmarks. We note that training and inference speed depend on both layer ranks and sequence lengths, and the performance difference is less pronounced for benchmarks with longer sequences.

\begin{table}[t]
    \centering
\caption{Method comparison for low-rank finetuning. We compare the computational cost of a single training step for an $n\times n$ layer matrix of rank $r$. In the table, ``local optimality'' refers to the property $P(W_\star)\nabla \mathcal L(W_\star) = 0$ for the computed adapter $W_\star$, as discussed in \Cref{sec:what_can_go_wrong}.}
 \resizebox{\textwidth}{!}
 {
\begin{tabular}{l  ccccc c  c }
\toprule
 Method & compute (per iteration) & memory (per iteration) & \# gradient evals. & rank adaptive & local optimality\\
\midrule
Full FT& $\landauO(n^2)$ & $\landauO(n^2)$ & 1 & - &   \cmark &\\
\midrule
\ALGNAME{}& $\landauO( 2nr + (2n+ 1)r^2 + r^3)$ & $\landauO(4nr + 3r^2)$ & 1 & \cmark &  \cmark\\
AdaLoRA \citep{zhang2023adalora}& $\landauO(2nr + (2n+ 1)r^2 + r^3)$ & $\landauO(2nr + 3r^2)$ & 1 & \cmark  &\xmark \\
DLRT \citep{Schothoefer_2022} & $\landauO(6nr + (2n + 5)r^2 + 9r^3)$ & $\landauO(4nr + 3r^2)$ & 3 & \cmark &   \cmark \\
LoRA \citep{hu2021lora} &$\landauO(2nr)$ &$\landauO(2nr)$ &1 & \xmark & \xmark \\
\bottomrule
\end{tabular}
}
\vspace{-3em}
    \label{tab_method_comparison}
\end{table}
\begin{table}[t]
\centering
\caption{DeBERTaV3-base fine-tuning on GLUE. We compare with full fine-tuning (Full FT), Houlsby adapter \citep{pmlr-v97-houlsby19a} (HAdapter),  Pfeiffer adapter \citep{pfeiffer2021adapterfusionnondestructivetaskcomposition} (PAdapter), LoRA \citep{hu2021lora}, AdaLoRA \citep{zhang2023adalora} and Bitfit\citep{zaken2022bitfitsimpleparameterefficientfinetuning}. We report target metrics and computational performance (higher is better) for the median of 5 runs using different random seeds. Best results per dataset are shown
in bold. Results for BitFit, HAdapter, PAdapter were taken from \citep{zhang2023adalora}.}
\label{tab_results}
\resizebox{\textwidth}{!}{%
\begin{tabular}{lcccccccc}
\toprule
Method (\textcolor{blue}{\# Params})&  MNLI & SST-2  & CoLA & QQP  & QNLI  & RTE  & MRPC  & STS-B  \\
  &(Acc) &(Acc) &(Mcc) &(F1)&(Acc)&(Acc)&(Acc)&(Corr)\\
\midrule
Full FT  (\textcolor{blue}{184M})  & 89.90 & 95.63 & 69.19 & 89.80 & 94.03 & 83.75 & 89.46 & 91.60  \\
\midrule
BitFit     (\textcolor{blue}{0.1M}) & 89.37 & 94.84 & 66.96 & 84.95 & 92.24 & 78.70 & 87.75 & 91.35  \\
HAdapter  (\textcolor{blue}{1.22M})  & 90.13 & 95.53 & 68.64 & 89.27 & 94.11 & 84.48 & 89.95 & 91.48 \\
PAdapter   (\textcolor{blue}{1.18M})  & 90.33 & 95.61 & 68.77 & 89.40 & \textbf{94.29} & 85.20 & 89.46 & 91.54 \\

LoRA r=8 (\textcolor{blue}{1.33M})  & 90.29 & 95.29 & 68.57 & 90.61 & 93.91& 85.5 & 89.75 & 89.10\\
AdaLoRA  (\textcolor{blue}{1.27M})  &\textbf{90.44}& 95.64  & 68.76 & \textbf{90.65} & 94.11 & \textbf{86.00} & 89.44 & 91.41 \\
\ALGNAME{} (\textcolor{blue}{reported individually})   &89.98 (\textcolor{blue}{0.7M})&\textbf{95.98} (\textcolor{blue}{1.17M})&\textbf{69.03} (\textcolor{blue}{0.98M})&90.53 (\textcolor{blue}{0.69M}) &94.23 (\textcolor{blue}{0.7M})&85.93 (\textcolor{blue}{1.19M})&\textbf{90.10} (\textcolor{blue}{0.75M}) & \textbf{91.58} (\textcolor{blue}{0.71M})\\
\midrule
\multicolumn{9}{l}{Evaluation and train time comparison}\\ 
\midrule 
AdaLoRA (eval/train) [it/sec] & 12.4/4.3& 17.6/6.7  &  24.6/8.1& 9.2/3.2 & 4.9/1.6 & 10.3/3.2 &9,9/3.1  & 21.1/\textbf{8.5}\\
\ALGNAME{} (eval/train) [it/sec] &\textbf{17.1/4.9}& \textbf{21.3/8.3} & \textbf{37.4/9.1} & \textbf{12.0/3.8} &\textbf{5.9/1.8}  & \textbf{13.2/3.7} & \textbf{12.6/3.7}  &\textbf{21.3/}8.3 \\
\bottomrule
\end{tabular}
}
\end{table}

\paragraph{Vision transformer for object classification.}
We compare \ALGNAME{} and AdaLoRA on fine-tuning the Vit-base-patch16-224  Vision Transformer, pre-trained on the Imagenet-1k dataset, and fine-tuned on Cifar10, Cifar100, and Tiny-Imagenet. Details on implementation and hyperparameters are provided in \Cref{app_vit}. \Cref{tab_vit} shows that \ALGNAME{}, with both global and local rank budgeting, achieves higher validation accuracy than AdaLoRA, while using fewer trainable parameters.

\begin{table}[t]
    \centering
    \begin{minipage}[t!]{0.67\textwidth}
\centering
\caption{Vit-base-patch16-224 fine-tuning on Cifar10, 100 and Tiny-Imagenet. We compare AdaLoRA to \ALGNAME{} with local and global budgeting reporting the median of 5 runs using different random seeds.}
\label{tab_vit}
\resizebox{\textwidth}{!}
{%
\begin{tabular}{lccccccc}
\toprule
Method &\multicolumn{2}{c}{ Cifar 10 [\%]} &\multicolumn{2}{c}{ Cifar 100 [\%]}  & \multicolumn{2}{c}{ Tiny-Imagenet  [\%]} \\
& \# Params &Acc [\%] &\# Params &Acc [\%] &\# Params &Acc [\%]\\
\midrule
AdaLoRA    & \textbf{0.47M} & 98.51 & 0.45M & 91.44  &0.9M & 87.2\\
\ALGNAME{}, local & \textbf{0.47M} & \textbf{98.55} &\textbf{0.35M}&\textbf{91.63}& 0.92M & \textbf{88.09} \\
\ALGNAME{}, global  & 0.48M & 98.51  & 0.47M & 91.62 & \textbf{0.75M} & 88.07 \\
\bottomrule
\end{tabular}
}
\vfill
 \end{minipage}
    \hfill
    \begin{minipage}[t!]{0.31\textwidth}
    \centering
        \caption{Stable Diffusion on Dreambooth benenchmark. We compare LoRA and \ALGNAME{} reporting the median of 5 runs.}
    \label{tab:lora_stable_diff}
    \resizebox{\textwidth}{!}
{%
\begin{tabular}{lcc}
\toprule
Method & Val. Loss & \# Params  \\
\midrule
 LoRA ($r = 5$) & $0.275$ &$3.0$ M \\
        LoRA ($r = 3$) & $0.281$ &$1.8$ M \\
        \midrule
        \ALGNAME{} ($\tau = 0.02$) &\textbf{0.242} &\textbf{2.6M} \\
        \ALGNAME{} ($\tau = 0.1$) & \textbf{0.257} &\textbf{1.4M} \\
\bottomrule
\end{tabular}
}
     \end{minipage}
\end{table}

\paragraph{Ablations.}
In \Cref{fig_vit_rafer_init_r}, we examine how the performance of \ALGNAME{} is influenced by the initial rank and learning rate. \Cref{fig_vit_rafer_init_r}(a, b) demonstrate that \ALGNAME{} dynamically recovers the intrinsic rank of the low-rank adaptation, regardless of the initial rank, highlighting the robustness of the method with respect to this hyperparameter. Notably, \ALGNAME{} can extend the adapter rank to full rank if necessary within logarithmic time, while truncating in constant time (in terms of optimization iterations). We provide a detailed discussion of the rank distribution across transformer layers in \Cref{app_vit}. Similarly, \Cref{fig_vit_rafer_init_r}(c, d) show that \ALGNAME{} is less sensitive to learning rate variations compared to AdaLoRA.

\paragraph{Dreambooth stable diffusion.}
Finally, we test \ALGNAME{} on fine-tuning Stable Diffusion \citep{rombach2021highresolution} using Dreambooth \citep{ruiz2023dreamboothfinetuningtexttoimage} on their original datasets. Implementation details are provided in \Cref{app_stable_diff}. In \Cref{tab:lora_stable_diff}, we compare the validation loss and number of parameters between LoRA and \ALGNAME{}, showing that \ALGNAME{} consistently achieves lower validation loss with fewer trainable parameters.

\begin{figure}[t]
    \centering
        \centering
        \begin{subfigure}[b]{0.4\textwidth}
            \centering
            \includegraphics[width=\textwidth]{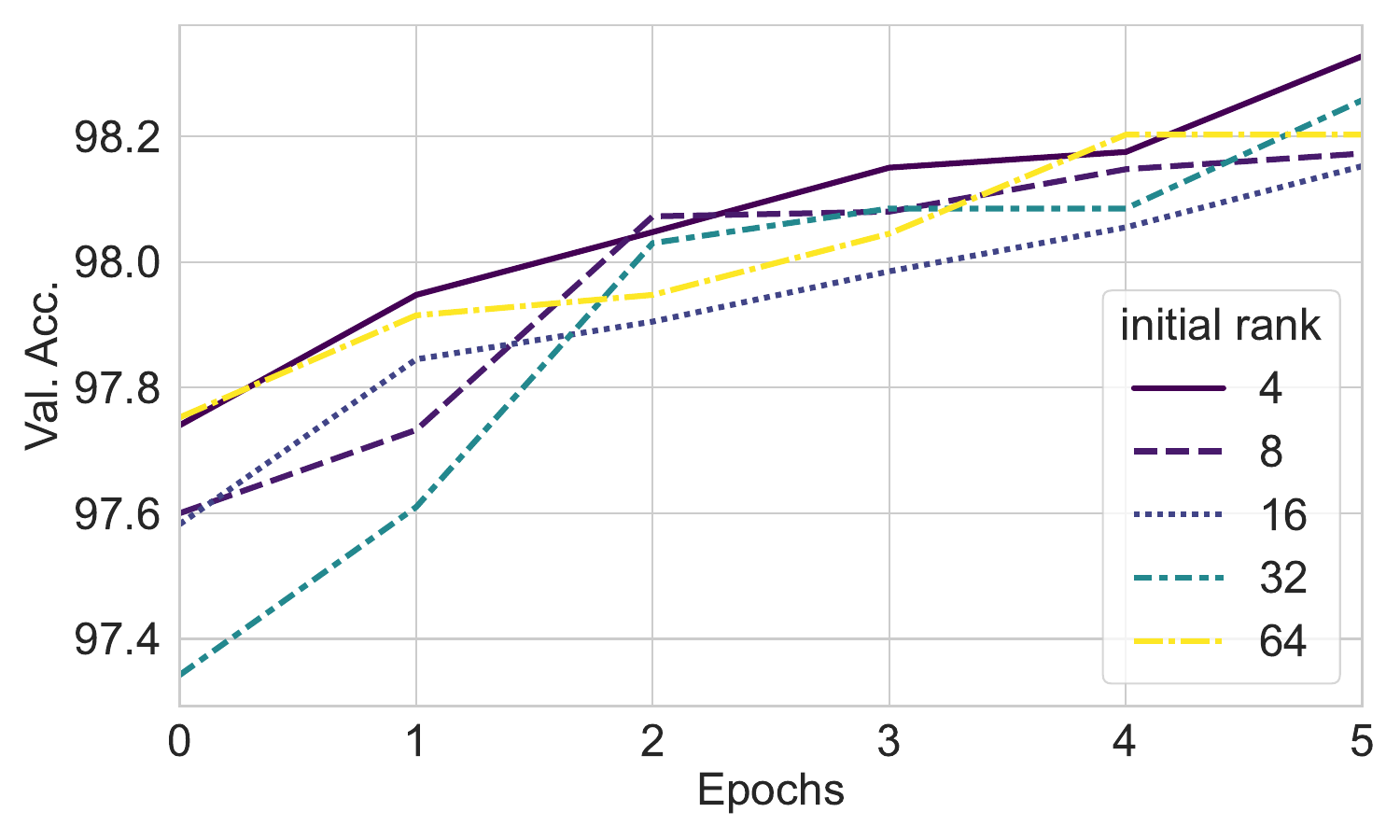}
            \caption{Validation accuracy over epochs.}
        \end{subfigure}
        \begin{subfigure}[b]{0.4\textwidth}
            \centering
            \includegraphics[width=\textwidth]{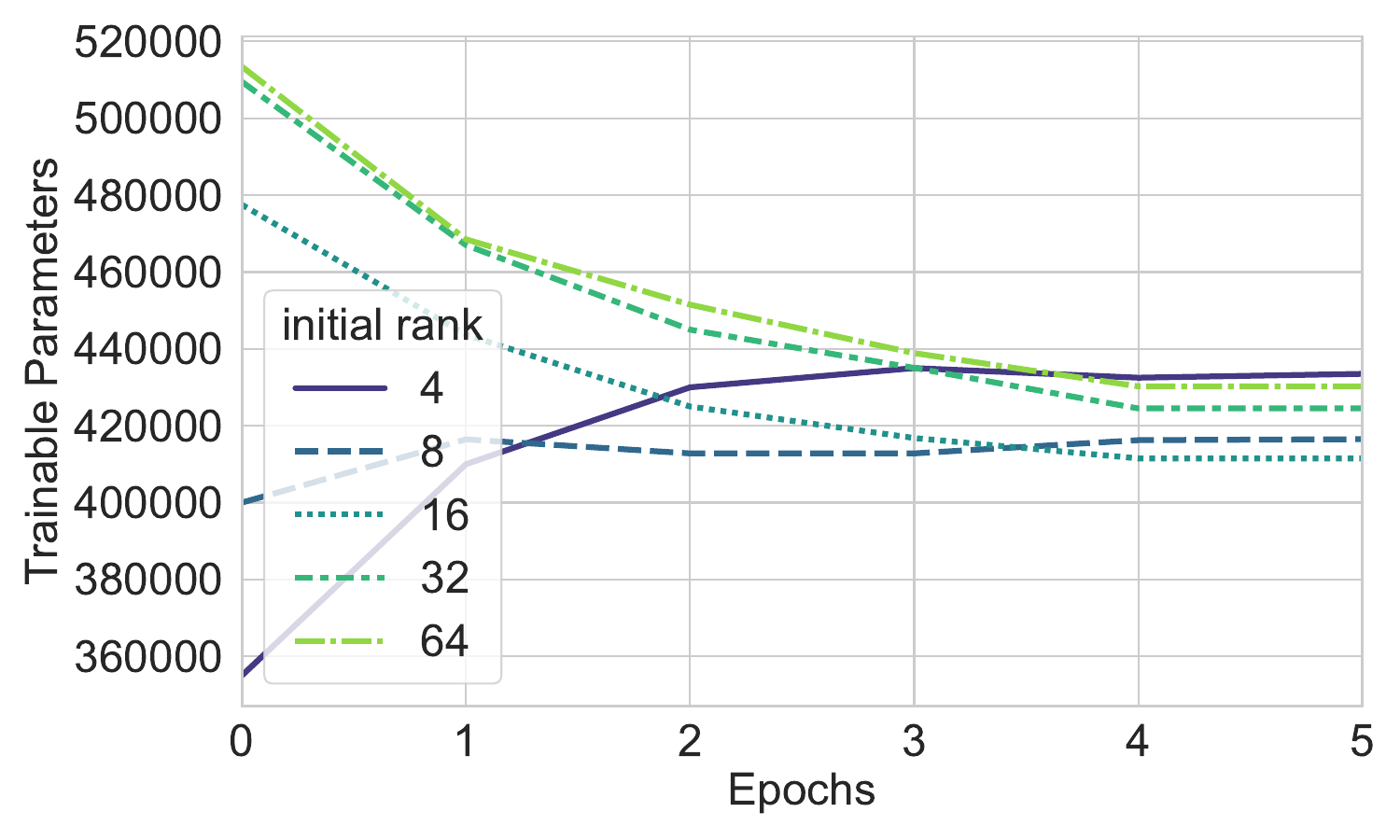}
            \caption{Trainable parameters over epochs. }
        \end{subfigure}\\
        \begin{subfigure}[b]{0.4\textwidth}
            \centering
            \includegraphics[width=\textwidth]{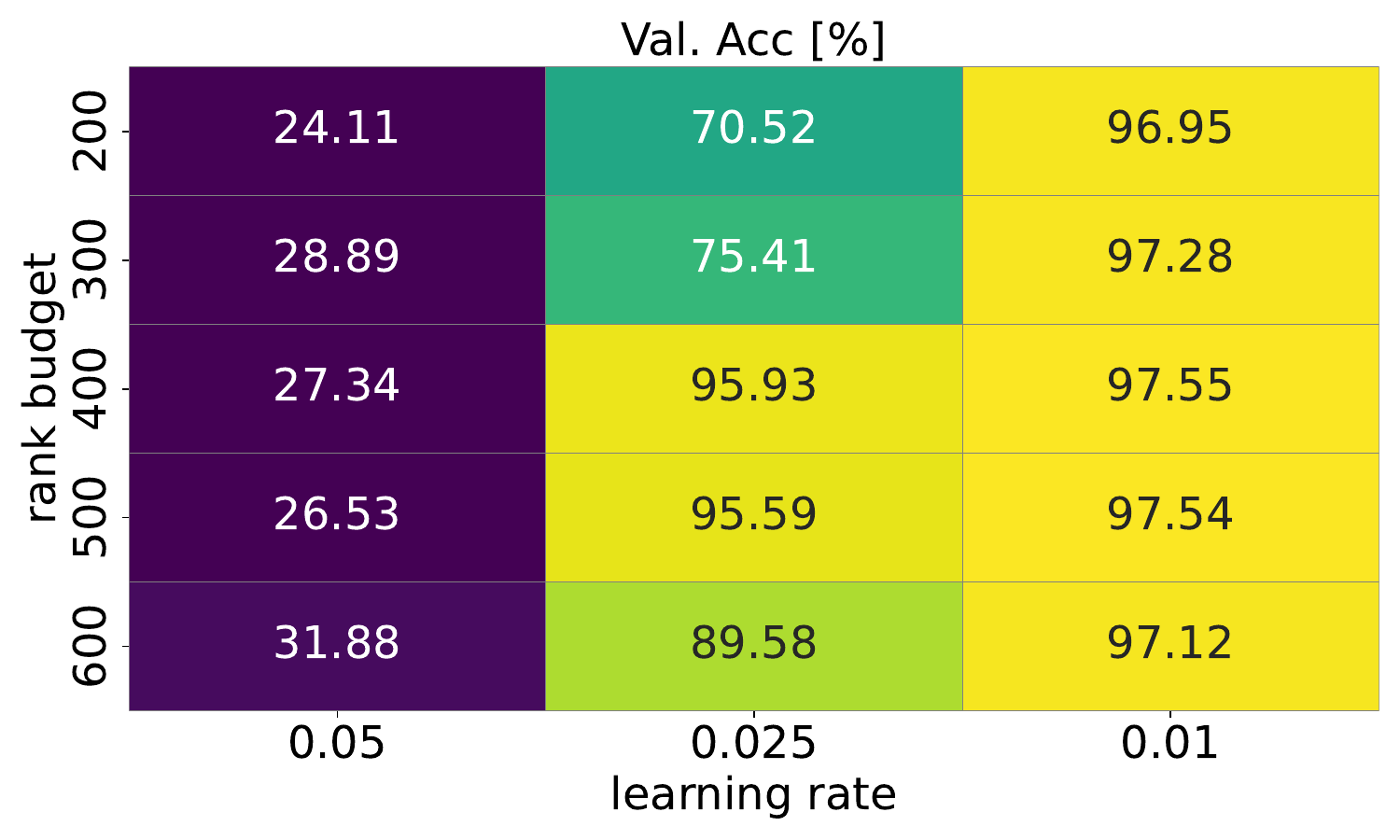}
            \caption{\ALGNAME{}}
        \end{subfigure}
        \begin{subfigure}[b]{0.4\textwidth}
            \centering
            \includegraphics[width=\textwidth]{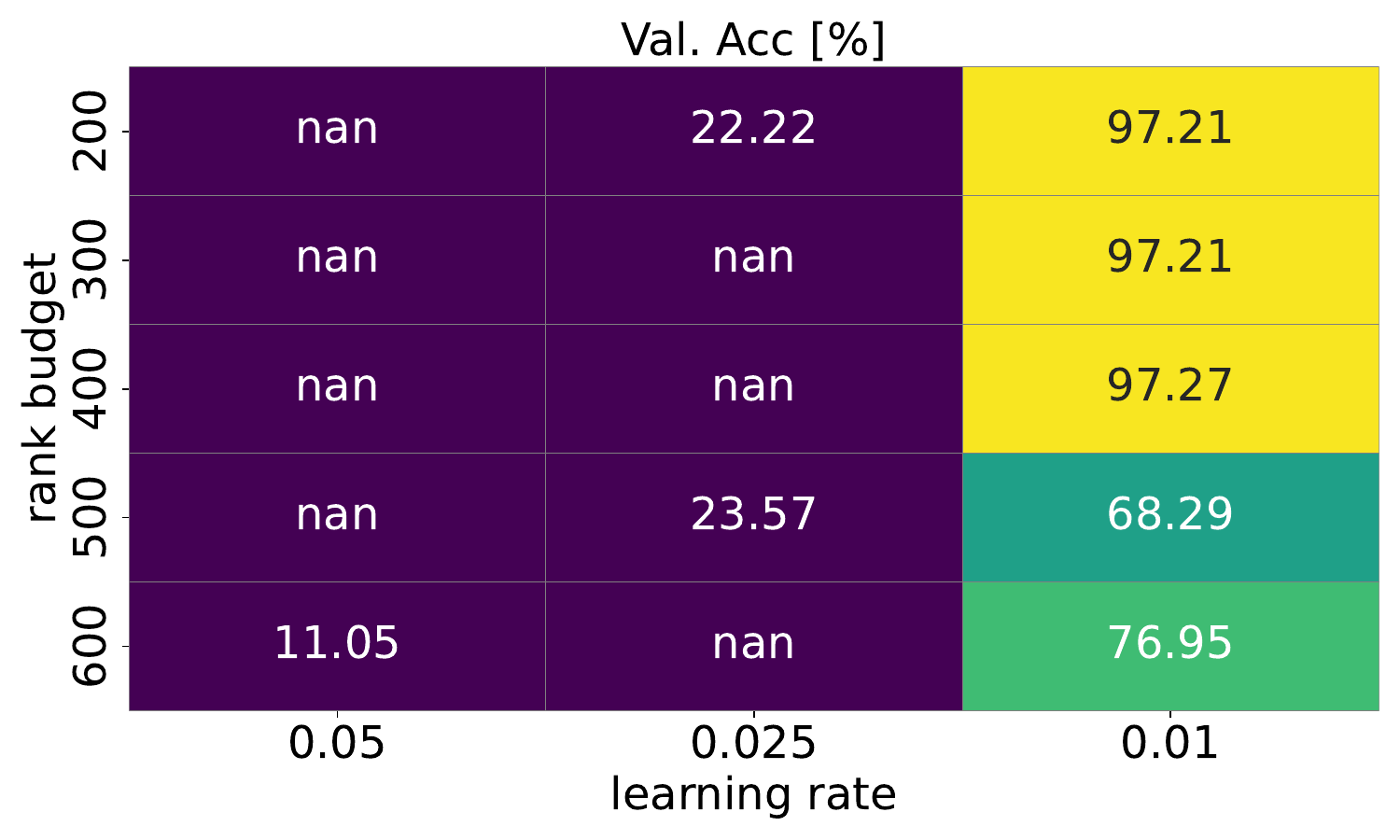}
                       \caption{AdaLoRA}
        \end{subfigure}
\caption{\textbf{Top panels (a, b):} \ALGNAME{}-adapted ViT-32b fine-tuned on Cifar10 with different initial layer ranks, using a learning rate of $1\mathrm{e}{-3}$ and $\tau=0.3$. The total number of trainable parameters converges to a similar steady state, regardless of the initial rank. The differences in validation accuracy between runs are smaller than the variance observed within individual setups. \textbf{Bottom panels (c, d):} \ALGNAME{}- and AdaLoRA-adapted ViT-32b fine-tuned on Cifar10 with different rank budgets and learning rates. Fields marked with \texttt{nan} indicate that training diverged within the first epoch. \ALGNAME{} demonstrates significantly greater robustness than AdaLoRA, particularly with high learning rates.}
 \label{fig_vit_rafer_init_r}
\end{figure}

\section{Conclusion}
We introduced \ALGNAME{} (\ALGNAMELONG), a novel adaptive low-rank fine-tuning method that combines computational efficiency with robustness. Based on geometric principles from dynamical low-rank approximation theory, the method comes with guarantees of convergence and local optimality. By leveraging a parallel update strategy of the low-rank adapters, the method requires only a single backward pass per iteration, achieving inference and training speed comparable or superior to existing baselines such as AdaLoRA, and much more efficient than previous geometric-aware strategies.
Our experiments on the GLUE benchmark, Vision Transformers, and Stable Diffusion demonstrate that \ALGNAME{} outperforms existing PEFT methods in both accuracy and efficiency, with fewer trainable parameters. These results, alongside strong theoretical guarantees, position \ALGNAME{} as a robust solution for efficient model adaptation.

\section{Funding Acknoledgements}{
The work of E. Zangrando was funded by
the MUR-PNRR project “Low-parametric machine learning”. \\
Francesco Tudisco is partially funded by the PRIN-MUR project MOLE and the PRIN-PNRR project FIN4GEO within the European Union's Next Generation EU framework, Mission 4, Component 2, CUP P2022BNB97. \\
The work of Steffen Schotthöfer is sponsored by the Applied Mathematics Progrm at the Office of Advanced Scientific Computing Research, U.S. Department of Energy, and performed at the Oak Ridge National Laboratory, which is managed by UT-Battelle, LLC under Contract No. DE-AC05-00OR22725 with the U.S. Department of Energy. The United States Government retains and the publisher, by accepting the article for publication, acknowledges that the United States Government retains a non-exclusive, paid-up, irrevocable, world-wide license to publish or reproduce the published form of this manuscript, or allow others to do so, for United States Government purposes. The Department of Energy will provide public access to these results of federally sponsored research in accordance with the DOE Public Access Plan (http://energy.gov/downloads/doe-public-access-plan).
}
\newpage

\bibliography{iclr2025_conference}

\begin{thebibliography}{34}
\providecommand{\natexlab}[1]{#1}
\providecommand{\url}[1]{\texttt{#1}}
\expandafter\ifx\csname urlstyle\endcsname\relax
  \providecommand{\doi}[1]{doi: #1}\else
  \providecommand{\doi}{doi: \begingroup \urlstyle{rm}\Url}\fi

\bibitem[Aghajanyan et~al.(2020)Aghajanyan, Zettlemoyer, and
  Gupta]{aghajanyan2020intrinsic}
Armen Aghajanyan, Luke Zettlemoyer, and Sonal Gupta.
\newblock Intrinsic dimensionality explains the effectiveness of language model
  fine-tuning, 2020.

\bibitem[Ceruti et~al.(2022)Ceruti, Kusch, and Lubich]{ceruti2021rank}
Gianluca Ceruti, Jonas Kusch, and Christian Lubich.
\newblock A rank-adaptive robust integrator for dynamical low-rank
  approximation.
\newblock \emph{BIT Numerical Mathematics}, 2022.
\newblock URL \url{https://doi.org/10.1007/s10543-021-00907-7}.

\bibitem[Ceruti et~al.(2023)Ceruti, Kusch, and
  Lubich]{ceruti2023parallelrankadaptiveintegratordynamical}
Gianluca Ceruti, Jonas Kusch, and Christian Lubich.
\newblock A parallel rank-adaptive integrator for dynamical low-rank
  approximation, 2023.
\newblock URL \url{https://arxiv.org/abs/2304.05660}.

\bibitem[Courbariaux et~al.(2016)Courbariaux, Hubara, Soudry, El-Yaniv, and
  Bengio]{courbariaux2016binarized}
Matthieu Courbariaux, Itay Hubara, Daniel Soudry, Ran El-Yaniv, and Yoshua
  Bengio.
\newblock Binarized neural networks: Training deep neural networks with weights
  and activations constrained to+ 1 or-1.
\newblock \emph{arXiv:1602.02830}, 2016.

\bibitem[Ghadiri et~al.(2023)Ghadiri, Fahrbach, Fu, and
  Mirrokni]{ghadiri2023approximately}
Mehrdad Ghadiri, Matthew Fahrbach, Gang Fu, and Vahab Mirrokni.
\newblock Approximately optimal core shapes for tensor decompositions.
\newblock In \emph{International Conference on Machine Learning}, pp.\
  11237--11254. PMLR, 2023.

\bibitem[Guo et~al.(2016)Guo, Yao, and Chen]{guo2016dynamic}
Yiwen Guo, Anbang Yao, and Yurong Chen.
\newblock Dynamic network surgery for efficient dnns.
\newblock \emph{Advances in neural information processing systems}, 29, 2016.

\bibitem[Hayou et~al.(2024)Hayou, Ghosh, and Yu]{hayou2024lora}
Soufiane Hayou, Nikhil Ghosh, and Bin Yu.
\newblock Lora+: Efficient low rank adaptation of large models, 2024.

\bibitem[He et~al.(2023)He, Gao, and
  Chen]{he2023debertav3improvingdebertausing}
Pengcheng He, Jianfeng Gao, and Weizhu Chen.
\newblock Debertav3: Improving deberta using electra-style pre-training with
  gradient-disentangled embedding sharing, 2023.
\newblock URL \url{https://arxiv.org/abs/2111.09543}.

\bibitem[He et~al.(2017)He, Zhang, and Sun]{he2017channel}
Yihui He, Xiangyu Zhang, and Jian Sun.
\newblock Channel pruning for accelerating very deep neural networks.
\newblock In \emph{Proceedings of the IEEE international conference on computer
  vision}, pp.\  1389--1397, 2017.

\bibitem[Hnatiuk et~al.(2024)Hnatiuk, Kusch, Kusch, Gauger, and
  Walther]{Hnatiuk}
Arsen Hnatiuk, Jonas Kusch, Lisa Kusch, Nicolas~R. Gauger, and Andrea Walther.
\newblock Stochastic aspects of dynamical low-rank approximation in the context
  of machine learning.
\newblock \emph{Optimization Online}, 2024.
\newblock \doi{https://optimization-online.org/?p=25971}.
\newblock URL \url{https://optimization-online.org/?p=25971}.

\bibitem[Houlsby et~al.(2019)Houlsby, Giurgiu, Jastrzebski, Morrone,
  De~Laroussilhe, Gesmundo, Attariyan, and Gelly]{pmlr-v97-houlsby19a}
Neil Houlsby, Andrei Giurgiu, Stanislaw Jastrzebski, Bruna Morrone, Quentin
  De~Laroussilhe, Andrea Gesmundo, Mona Attariyan, and Sylvain Gelly.
\newblock Parameter-efficient transfer learning for {NLP}.
\newblock In Kamalika Chaudhuri and Ruslan Salakhutdinov (eds.),
  \emph{Proceedings of the 36th International Conference on Machine Learning},
  volume~97 of \emph{Proceedings of Machine Learning Research}, pp.\
  2790--2799. PMLR, 09--15 Jun 2019.
\newblock URL \url{https://proceedings.mlr.press/v97/houlsby19a.html}.

\bibitem[Hu et~al.(2021)Hu, Shen, Wallis, Allen-Zhu, Li, Wang, Wang, and
  Chen]{hu2021lora}
Edward~J Hu, Yelong Shen, Phillip Wallis, Zeyuan Allen-Zhu, Yuanzhi Li, Shean
  Wang, Lu~Wang, and Weizhu Chen.
\newblock Lora: Low-rank adaptation of large language models.
\newblock \emph{arXiv preprint arXiv:2106.09685}, 2021.

\bibitem[Idelbayev \& Carreira-Perpinan(2020)Idelbayev and
  Carreira-Perpinan]{Idelbayev_2020_CVPR}
Yerlan Idelbayev and Miguel~A. Carreira-Perpinan.
\newblock Low-rank compression of neural nets: Learning the rank of each layer.
\newblock In \emph{Proceedings of the IEEE/CVF Conference on Computer Vision
  and Pattern Recognition (CVPR)}, June 2020.

\bibitem[Khodak et~al.(2021)Khodak, Tenenholtz, Mackey, and
  Fusi]{khodak2021initialization}
Mikhail Khodak, Neil Tenenholtz, Lester Mackey, and Nicolo Fusi.
\newblock Initialization and regularization of factorized neural layers.
\newblock In \emph{International Conference on Learning Representations}, 2021.

\bibitem[Koch \& Lubich(2007{\natexlab{a}})Koch and Lubich]{KochLubich07}
O.~Koch and C.~Lubich.
\newblock Dynamical low-rank approximation.
\newblock \emph{SIAM Journal on Matrix Analysis and Applications}, 29\penalty0
  (2):\penalty0 434--454, 2007{\natexlab{a}}.
\newblock ISSN 0895-4798.
\newblock \doi{10.1137/050639703}.
\newblock URL \url{https://doi.org/10.1137/050639703}.

\bibitem[Koch \& Lubich(2007{\natexlab{b}})Koch and Lubich]{koch2007dynamical}
Othmar Koch and Christian Lubich.
\newblock Dynamical low-rank approximation.
\newblock \emph{SIAM Journal on Matrix Analysis and Applications}, 29\penalty0
  (2):\penalty0 434--454, 2007{\natexlab{b}}.

\bibitem[Lialin et~al.(2023)Lialin, Shivagunde, Muckatira, and
  Rumshisky]{lialin2023relora}
Vladislav Lialin, Namrata Shivagunde, Sherin Muckatira, and Anna Rumshisky.
\newblock Relora: High-rank training through low-rank updates, 2023.

\bibitem[Mao et~al.(2024)Mao, Huang, Guan, Bao, Mo, and Xu]{Mao2024DoRAEP}
Yulong Mao, Kaiyu Huang, Changhao Guan, Ganglin Bao, Fengran Mo, and Jinan Xu.
\newblock Dora: Enhancing parameter-efficient fine-tuning with dynamic rank
  distribution.
\newblock 2024.
\newblock URL \url{https://api.semanticscholar.org/CorpusID:270062642}.

\bibitem[Molchanov et~al.(2017)Molchanov, Tyree, Karras, Aila, and
  Kautz]{molchanov2017pruning}
P~Molchanov, S~Tyree, T~Karras, T~Aila, and J~Kautz.
\newblock Pruning convolutional neural networks for resource efficient
  inference.
\newblock In \emph{International Conference on Learning Representations}, 2017.

\bibitem[Pfeiffer et~al.(2021)Pfeiffer, Kamath, Rücklé, Cho, and
  Gurevych]{pfeiffer2021adapterfusionnondestructivetaskcomposition}
Jonas Pfeiffer, Aishwarya Kamath, Andreas Rücklé, Kyunghyun Cho, and Iryna
  Gurevych.
\newblock Adapterfusion: Non-destructive task composition for transfer
  learning, 2021.
\newblock URL \url{https://arxiv.org/abs/2005.00247}.

\bibitem[Rombach et~al.(2021)Rombach, Blattmann, Lorenz, Esser, and
  Ommer]{rombach2021highresolution}
Robin Rombach, Andreas Blattmann, Dominik Lorenz, Patrick Esser, and Björn
  Ommer.
\newblock High-resolution image synthesis with latent diffusion models, 2021.

\bibitem[Ruiz et~al.(2023)Ruiz, Li, Jampani, Pritch, Rubinstein, and
  Aberman]{ruiz2023dreamboothfinetuningtexttoimage}
Nataniel Ruiz, Yuanzhen Li, Varun Jampani, Yael Pritch, Michael Rubinstein, and
  Kfir Aberman.
\newblock Dreambooth: Fine tuning text-to-image diffusion models for
  subject-driven generation, 2023.
\newblock URL \url{https://arxiv.org/abs/2208.12242}.

\bibitem[Sato(2021)]{sato2021riemannian}
Hiroyuki Sato.
\newblock \emph{Riemannian optimization and its applications}, volume 670.
\newblock Springer, 2021.

\bibitem[Schotth\"{o}fer et~al.(2022)Schotth\"{o}fer, Zangrando, Kusch, Ceruti,
  and Tudisco]{Schothoefer_2022}
Steffen Schotth\"{o}fer, Emanuele Zangrando, Jonas Kusch, Gianluca Ceruti, and
  Francesco Tudisco.
\newblock Low-rank lottery tickets: finding efficient low-rank neural networks
  via matrix differential equations.
\newblock In \emph{Advances in Neural Information Processing Systems}, 2022.
\newblock URL
  \url{https://proceedings.neurips.cc/paper_files/paper/2022/file/7e98b00eeafcdaeb0c5661fb9355be3a-Paper-Conference.pdf}.

\bibitem[Schotthöfer \& Laiu(2024)Schotthöfer and
  Laiu]{schotthöfer2024federateddynamicallowranktraining}
Steffen Schotthöfer and M.~Paul Laiu.
\newblock Federated dynamical low-rank training with global loss convergence
  guarantees, 2024.
\newblock URL \url{https://arxiv.org/abs/2406.17887}.

\bibitem[Valipour et~al.(2023)Valipour, Rezagholizadeh, Kobyzev, and
  Ghodsi]{valipour2023dylora}
Mojtaba Valipour, Mehdi Rezagholizadeh, Ivan Kobyzev, and Ali Ghodsi.
\newblock Dylora: Parameter efficient tuning of pre-trained models using
  dynamic search-free low-rank adaptation, 2023.

\bibitem[Wang et~al.(2019)Wang, Singh, Michael, Hill, Levy, and
  Bowman]{wang2019gluemultitaskbenchmarkanalysis}
Alex Wang, Amanpreet Singh, Julian Michael, Felix Hill, Omer Levy, and
  Samuel~R. Bowman.
\newblock Glue: A multi-task benchmark and analysis platform for natural
  language understanding, 2019.
\newblock URL \url{https://arxiv.org/abs/1804.07461}.

\bibitem[Wang et~al.(2021)Wang, Agarwal, and
  Papailiopoulos]{wang2021pufferfish}
Hongyi Wang, Saurabh Agarwal, and Dimitris Papailiopoulos.
\newblock {P}ufferfish: {C}ommunication-efficient models at no extra cost.
\newblock \emph{Proceedings of Machine Learning and Systems}, 3:\penalty0
  365--386, 2021.

\bibitem[Wanner \& Hairer(1996)Wanner and Hairer]{wanner1996solving}
Gerhard Wanner and Ernst Hairer.
\newblock \emph{Solving ordinary differential equations II}, volume 375.
\newblock Springer Berlin Heidelberg New York, 1996.

\bibitem[Wu et~al.(2016)Wu, Leng, Wang, Hu, and Cheng]{wu2016quantized}
Jiaxiang Wu, Cong Leng, Yuhang Wang, Qinghao Hu, and Jian Cheng.
\newblock Quantized convolutional neural networks for mobile devices.
\newblock In \emph{Proceedings of the IEEE conference on computer vision and
  pattern recognition}, pp.\  4820--4828, 2016.

\bibitem[Zaken et~al.(2022)Zaken, Ravfogel, and
  Goldberg]{zaken2022bitfitsimpleparameterefficientfinetuning}
Elad~Ben Zaken, Shauli Ravfogel, and Yoav Goldberg.
\newblock Bitfit: Simple parameter-efficient fine-tuning for transformer-based
  masked language-models, 2022.
\newblock URL \url{https://arxiv.org/abs/2106.10199}.

\bibitem[Zangrando et~al.(2024)Zangrando, Schotth{\"o}fer, Ceruti, Kusch, and
  Tudisco]{zangrando2023rank}
Emanuele Zangrando, Steffen Schotth{\"o}fer, Gianluca Ceruti, Jonas Kusch, and
  Francesco Tudisco.
\newblock Rank-adaptive spectral pruning of convolutional layers during
  training.
\newblock In \emph{Advances in Neural Information Processing Systems}, 2024.

\bibitem[Zhang et~al.(2023)Zhang, Chen, Bukharin, Karampatziakis, He, Cheng,
  Chen, and Zhao]{zhang2023adalora}
Qingru Zhang, Minshuo Chen, Alexander Bukharin, Nikos Karampatziakis, Pengcheng
  He, Yu~Cheng, Weizhu Chen, and Tuo Zhao.
\newblock Adalora: Adaptive budget allocation for parameter-efficient
  fine-tuning, 2023.

\bibitem[Zhao et~al.(2024)Zhao, Zhang, Chen, Wang, Anandkumar, and
  Tian]{zhao2024galore}
Jiawei Zhao, Zhenyu Zhang, Beidi Chen, Zhangyang Wang, Anima Anandkumar, and
  Yuandong Tian.
\newblock Galore: Memory-efficient llm training by gradient low-rank
  projection, 2024.

\end{thebibliography}
\bibliographystyle{iclr2025_conference}

\appendix
\newpage

\section{Algorithm for auxiliary functions}
We present the auxiliary function for \Cref{alg_efficient_TDLRT} in \Cref{alg_helper}.
\begin{algorithm}[t]

\DontPrintSemicolon
\SetAlgoLined

\vspace{.2em}
\SetKwProg{Def}{def}{:}{}
\Def{   {\tt  optimizer\_step}($P$: param, $G$: gradient, $\lambda$: learning rate)}{
$P^\textup{new}\gets P -\lambda G$\tcc*{May use momentum and weight decay} 
return $P^\textup{new}$
}
\vspace{.2em}
\SetKwProg{Def}{def}{:}{}
\Def{   {\tt  basis\_augmentation}($B$: old basis, $G_B$: basis dynamics)}{
$[B \mid \widetilde{B}] \gets  \texttt{qr}([B \mid G_B])$ \;
return  $\widetilde{B}$
}

\vspace{.2em}
\SetKwProg{Def}{def}{:}{}
\Def{   {\tt  truncation}($\augS$: augmented coefficient, $\augU$: augmented basis, $\augV$: augmented co-basis )}{
$P_{r_1}, \Sigma_{r_1}, Q_{r_1} \gets$ truncated \texttt{svd}$(\widetilde{S})$ with threshold  $\vartheta$ to new rank $r_1$\;
$U\gets   \augU P_{r_1}$;
$V\gets   \augV Q_{r_1}$ \tcc*{Basis update}
$S\gets \Sigma_{r_1}; S^{\textup{inv}}\gets \Sigma_{r_1}^{-1}$ \tcc*{Coefficient update with diagonal $\Sigma_{r_1}$}
return  $U,S,V,S^{\textup{inv}}$\;
}
\caption{Various auxiliary functions}\label{alg_helper}
\end{algorithm}
\section{Additional information for the numerical test cases}
\subsection{GLUE Benchmark}\label{app_glue}
\subsubsection{Dataset description}
We compare \ALGNAME{} to several fine-tuning methods from recent literature in the General Language Understanding Evaluation (GLUE) benchmark \citep{wang2019gluemultitaskbenchmarkanalysis}. 
The GLUE benchmark is a collection of diverse natural language understanding tasks designed to evaluate the performance of models in comprehending and processing human language. GLUE provides a comprehensive assessment by including tasks that cover a range of linguistic phenomena, such as textual entailment, sentiment analysis, sentence similarity, and more. The benchmark consists of nine different tasks:
\begin{itemize}[leftmargin=*,noitemsep,topsep=0em]
\item CoLA (Corpus of Linguistic Acceptability): Classifying whether a sentence is grammatically correct or not.
\item SST-2 (Stanford Sentiment Treebank): Sentiment analysis task where the goal is to classify the sentiment of a sentence as positive or negative.
\item MRPC (Microsoft Research Paraphrase Corpus): Identifying if two sentences are paraphrases of each other.
\item STS-B (Semantic Textual Similarity Benchmark): Measuring the degree of semantic similarity between two sentences on a scale from 1 to 5.
\item QQP (Quora Question Pairs): Determining if a pair of questions are semantically equivalent.
\item MNLI (Multi-Genre Natural Language Inference): Classifying the relationship between a pair of sentences (entailment, contradiction, or neutral).
\item QNLI (Question Natural Language Inference): Determining if a sentence provides a correct answer to a given question.
\item RTE (Recognizing Textual Entailment): Binary classification task for entailment and contradiction.
\item WNLI (Winograd Schema Challenge): Resolving pronoun reference ambiguity in sentences.
Specific Focus: MRPC (Microsoft Research Paraphrase Corpus)
\end{itemize}
We present the benchmark overview in \Cref{tab_glue_overview}.
\begin{table}[h]
\caption{Summary of GLUE benchmark tasks}
\label{tab_glue_overview}
\centering
\begin{tabular}{@{}l l r r r r l@{}}
\toprule
\textbf{Corpus} & \textbf{Task} & \textbf{\#Train} & \textbf{\#Dev} & \textbf{\#Test} & \textbf{\#Label} & \textbf{Metrics} \\ 
\midrule
\multicolumn{7}{c}{\textbf{Single-Sentence Classification (GLUE)}} \\ 
CoLA & Acceptability & 8.5k & 1k & 1k & 2 & Matthews corr \\
SST  & Sentiment & 67k & 872 & 1.8k & 2 & Accuracy \\ 
\midrule
\multicolumn{7}{c}{\textbf{Pairwise Text Classification (GLUE)}} \\ 
MNLI & NLI & 393k & 20k & 20k & 3 & Accuracy \\
RTE & NLI & 2.5k & 276 & 3k & 2 & Accuracy \\
QQP & Paraphrase & 364k & 40k & 391k & 2 & F1 \\
MRPC & Paraphrase & 3.7k & 408 & 1.7k & 2 & Accuracy \\
QNLI & QA/NLI & 108k & 5.7k & 5.7k & 2 & Accuracy \\
\midrule
\multicolumn{7}{l}{\textbf{Text Similarity (GLUE)}} \\ 
STS-B & Similarity & 7k & 1.5k & 1.4k & 1 & Pearson/Spearman cor \\
\bottomrule
\end{tabular}
\end{table}
To recapitulate, the F1 score is defined in dependence of precision score $P$ and recall score $R$.
The model precision $P$ is given by
\begin{align}
    P := \frac{P_T}{P_T+P_F},
\end{align}
where $P_T$ is the number of true positive and $P_F$ is the number of false positive examples.
The recall $R$ is the ratio
\begin{align}
    R := \frac{P_T}{P_T +N_F},
\end{align}
where $N_F$ are the false negatives. The F1 score combines these two metrics to
\begin{align}
    F1 := \frac{2P R}{P+R}\,.
\end{align}
\subsubsection{Reference implementations}
\textbf{Full finetuning (FT)}: This is the most common approach for model finetuning and transfer learning. Here, the model is initialized with pre-trained weights and all model parameters are updated with gradient descent. 

\textbf{Bitfit~\citep{zaken2022bitfitsimpleparameterefficientfinetuning}}: Here, the model is initialized with pre-trained weights, but only bias terms are updated with gradient descent. 

\textbf{Adapter tuning~\citep{pmlr-v97-houlsby19a,pfeiffer2021adapterfusionnondestructivetaskcomposition}}: Two-layer adapters are inserted between transformer blocks. In \citep{pmlr-v97-houlsby19a}, the adapter is inserted between the self-attention module and the feed-forward module and equipped with a residual connection. In \citep{pfeiffer2021adapterfusionnondestructivetaskcomposition}, the adapter is applied after the feed-forward module and the layer-norm module. To maintain conistency with the notation of \citep{zhang2023adalora}, we call the method of \citep{pmlr-v97-houlsby19a} HAdapter and the  method of \citep{pfeiffer2021adapterfusionnondestructivetaskcomposition} PAdapter. 

\textbf{LoRA~\citep{hu2021lora}}: As stated in \Cref{sec:what_can_go_wrong}, LoRA applies additive corrections to selected weight matrices, i.e.  $\fz = \sigma(W_{\mathrm{pt}}\fx + \frac{\alpha}{r}AB^\top\fx)$ for $A,B\in\mathbb{R}^{n\times r}$. We apply LoRA to key $W_k$, query $W_q$ and value $W_v$ matrices of all attention blocks, and to both feed-forward layers $W_{f_1}$ and $W_{f_2}$. 

\textbf{AdaLoRA~\citep{zhang2023adalora}}: As stated in \Cref{sec:what_can_go_wrong}, AdaLoRA applies additive corrections to selected weight matrices, i.e.   $\fz = \sigma(W_{\mathrm{pt}}\fx + \frac{\alpha}{r}USV^\top\fx)$
with arbitrary activation $\sigma$, frozen pre-trained weights $W_{\mathrm{pt}}\in\mathbb{R}^{n\times n}$,  rank $r$ adapter weights $U,V\in\mathbb{R}^{n\times r}$, $S\in\mathbb{R}^{r\times r}$. An SVD-based truncation mechanism is used to select layer ranks. Alternatively, the loss-sensitivity of singular vectors can be used for layer rank selection.  Just like LoRA, we apply AdaLoRA to key $W_k$, query $W_q$ and value $W_v$ matrices of all attention blocks, and to both feed-forward layers $W_{f_1}$ and $W_{f_2}$.

We use the implementation of \citep[Appendix C]{zhang2023adalora} to compute the results for the presented reference methods and use their reported hyper-parameter choices:
We compare the baselines under different budget levels, for example, given the total trainable
parameters as 0.3/0.6/1.2 million. In order to match the parameter budget, we select the hidden
dimensions of adapters from \{8, 16, 32, 64\}, set the rank r of LoRA as \{2, 4, 8\}, and choose the
final budget $b^{(T)}$ of AdaLoRA from \{144, 288, 576\}. Then we set $b^{(0)}$ as 1.5 times of $b^{(T)}$
for AdaLoRA and select the regularization coefficient $\gamma$ from \{0.1, 0.3, 0.5\}. We set the exponential moving average parameters $\beta_1$ and $\beta_2$ of AdamW
as their default value $0.85$. We select the learning rate from $\{5\rm{e}{-5}, 8 \rm{e}{-5}, 1\rm{e}{-4},2\rm{e}{-4}\}$.

\subsubsection{Implementation details}
We implement \ALGNAME{} as similar as possible as Adalora to achieve a fair comparison. That is, we add an adapter of the form $\fz = \sigma(W_{\mathrm{pt}}\fx + USV^\top\fx)$  to the key $W_k$, query $W_q$ and value $W_v$ matrices of all attention blocks, and to both feed-forward layers $W_{f_1}$ and $W_{f_2}$. For each adapter, we employ \Cref{alg_efficient_TDLRT} to update the layer weights and ranks.

In \Cref{tab_glue_hyperparam}, we display the hyper-parameter choices of \ALGNAME{}
\begin{table}[h!]
\centering
\caption{Hyper-parameter setup for the GLUE benchmark.}
\label{tab_glue_hyperparam}
\begin{tabular}{lccccccc}
\toprule
\textbf{Dataset} & \textbf{Learning Rate} & \textbf{Batch Size} & \textbf{\# Epochs} & $\tau$ & inital rank\\
\midrule
MNLI   & $5 \times 10^{-4}$  & 32 & 7  & 0.15 &10\\
RTE    & $1.2 \times 10^{-3}$ & 32 & 50 & 0.15 &10\\
QNLI   & $1.2 \times 10^{-3}$ & 32 & 5 & 0.15 &10 \\
MRPC   & $1 \times 10^{-3}$   & 32 & 30 & 0.15 &10\\
QQP    & $5 \times 10^{-4}$   & 32 & 5  & 0.15 &10\\
SST-2  & $8 \times 10^{-4}$   & 32 & 24 & 0.15 &10\\
CoLA   & $5 \times 10^{-4}$   & 32 & 25 & 0.15 &10\\
STS-B  & $2.2 \times 10^{-3}$ & 32 & 25 & 0.15 &10\\
\bottomrule
\end{tabular}
\end{table}

\subsection{Object classification benchmarks for the  vit-base-patch16-224 vision transformer}\label{app_vit}
We present in \Cref{tab_vit} results for finetuning the vit-base-patch16-224 vision transformer, which is pretrained on the imagenet-1k-dataset. The pretrained weights are downloaded from the torch-vision python package.
For both AdaLora and \ALGNAME{}, we augment the key, query, and value matrices from attention layers as well as the three fully connected layers of each transformer block with a low-rank adapter. The biases of each layer are trainable. Additionally, the classifier is augmented with a low-rank adapter. {The classifier is low-rank by construction, and we fix the rank as the number of classes.} 
We fine-tune the vision transformer on Cifar10, Cifar100 and Tiny-Imagenet. 

The hyperparameter settings to generate the results of \Cref{tab_vit}, \Cref{tab_ranks1} and \Cref{tab_rank2} are given in \Cref{tab_hperparam_cifar}.
 
\begin{table}[h!]
\centering
\caption{Hyper-parameter setup for fine-tuning vit-base-patch16-224 vision transformer  with \ALGNAME{}.}
\label{tab_hperparam_cifar}
\begin{tabular}{lccccccc}
\toprule
\textbf{Dataset} & \textbf{Learning Rate} & \textbf{Batch Size} & \textbf{\# Epochs} & $\tau$ & inital rank\\
\midrule
Cifar10   & $1 \times 10^{-3}$  & 256 & 5  & 0.2 &16\\
Cifar100    & $1 \times 10^{-3}$ & 256 & 5 & 0.25 &32\\
TinyImageNet   & $1 \times 10^{-4}$ & 256 & 5 & 0.15 &32 \\
\bottomrule
\end{tabular}
\end{table}

\Cref{tab_ranks1} and \Cref{tab_rank2} show the rank distribution across layers for both AdaLoRA and \ALGNAME{} with global budget, for learning rate $\lambda=1\rm{e}{-3}$ and $\lambda=1\rm{e}{-4}$ and budgets ranging from $b=200,\dots, 600$ total ranks for the network. 
 Both methods prefer to allocate higher ranks to the deeper layers of the vision transformer, and prefer fully-connected layers over attention layers. 
Both methods pefer the first fully connected layer of a transformer block over the second. Overall \ALGNAME{} tends to assert higher ranks to single layers, compared to AdaLora, that distributes ranks more heterogeneously.
The effects are more pronounced for smaller learning rates.

\begin{figure}[t]
    \centering

    \begin{subfigure}[b]{0.4\textwidth}
        \centering
        \includegraphics[width=\textwidth]{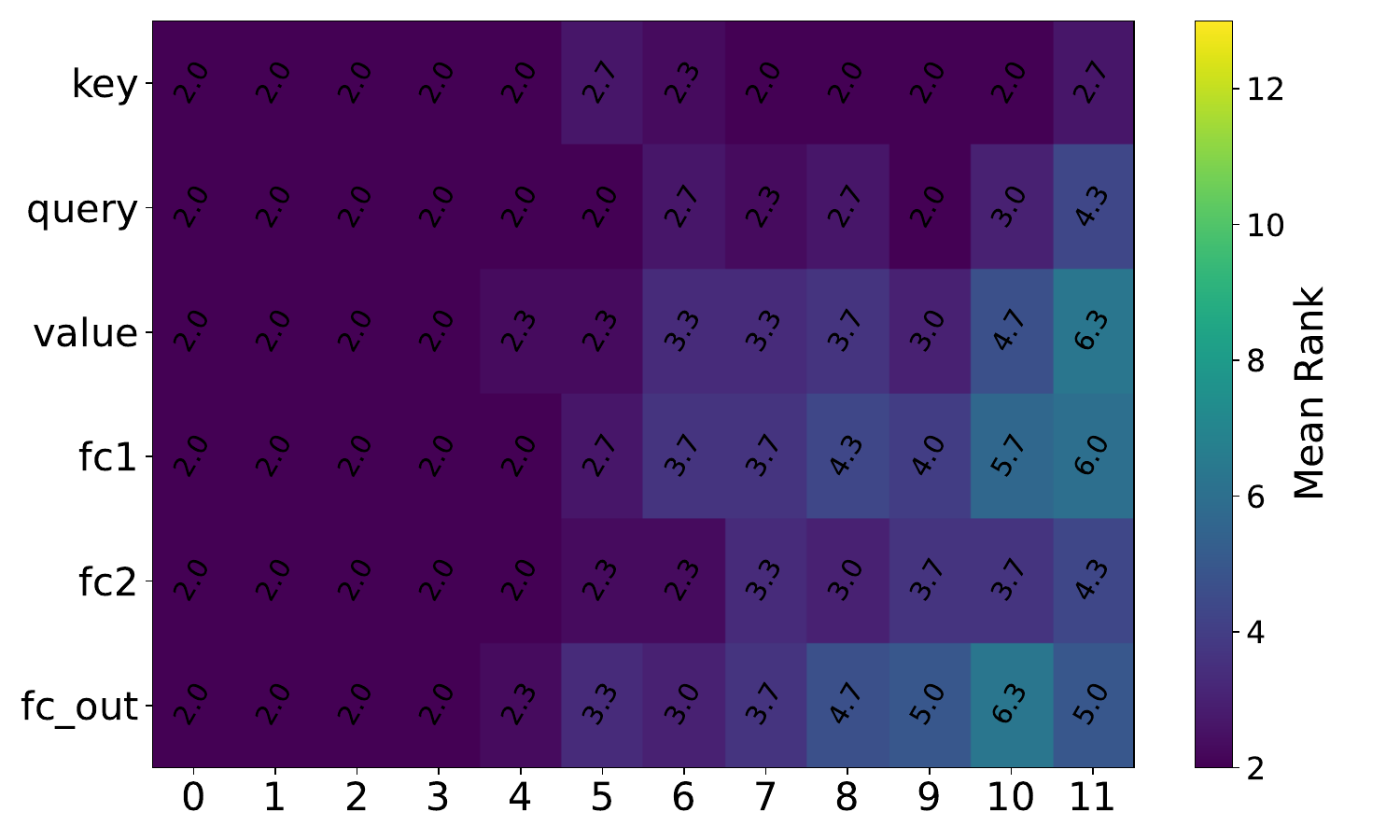}
        \caption{\ALGNAME{}: mean b=200}
    \end{subfigure}
    \begin{subfigure}[b]{0.4\textwidth}
        \centering
        \includegraphics[width=\textwidth]{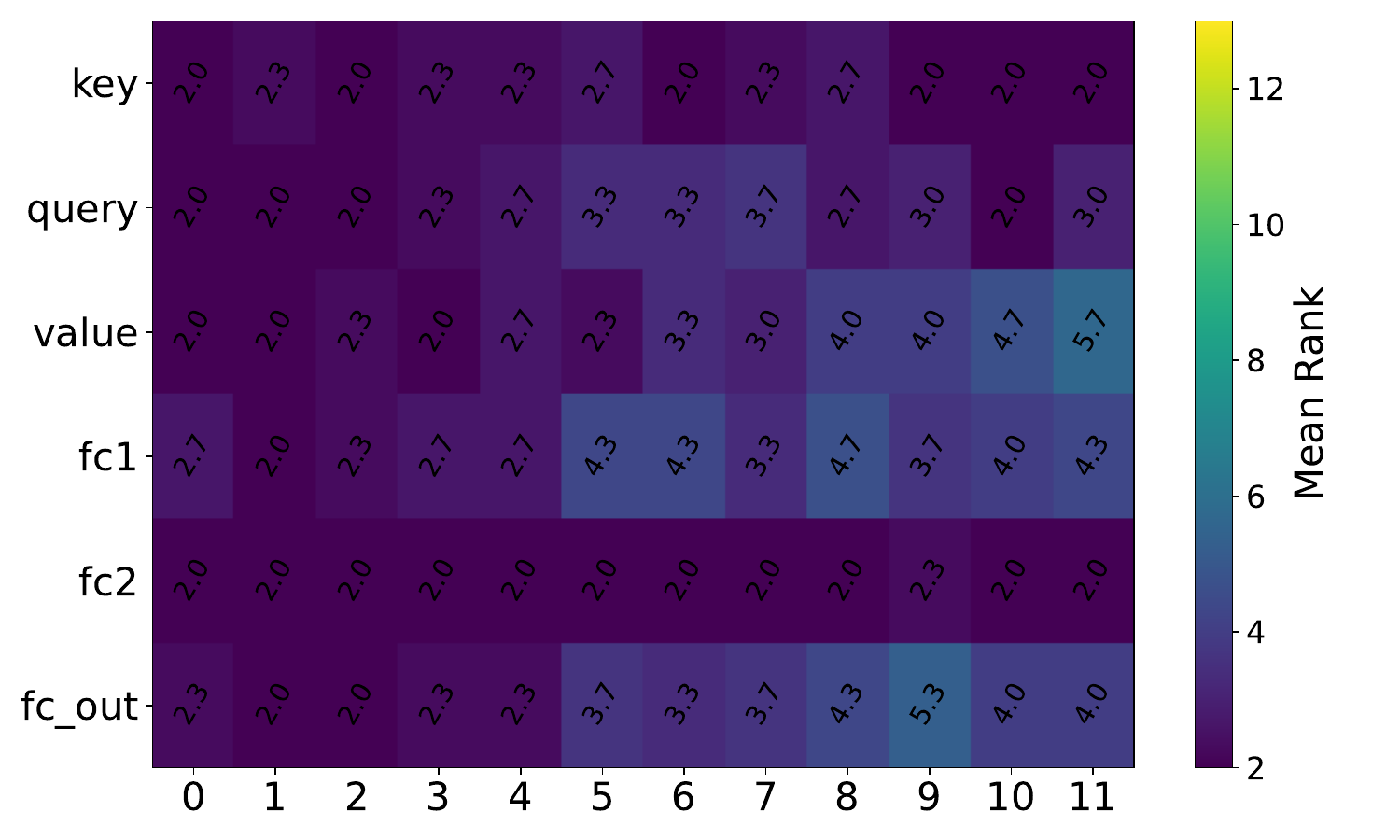}
        \caption{AdaLoRA: mean b=200}
    \end{subfigure}

    \begin{subfigure}[b]{0.4\textwidth}
        \centering
        \includegraphics[width=\textwidth]{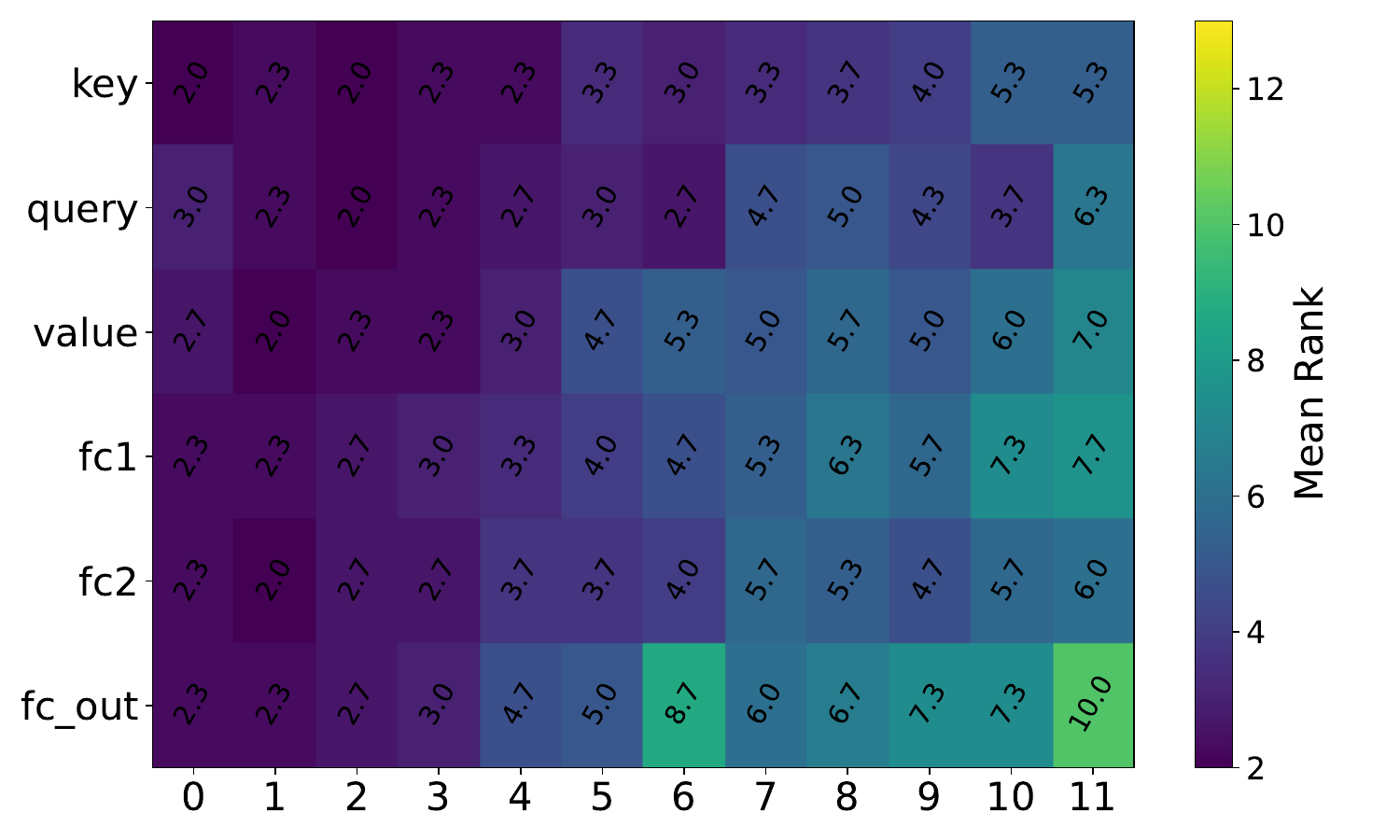}
        \caption{\ALGNAME{}: mean b=300}
    \end{subfigure}
    \begin{subfigure}[b]{0.4\textwidth}
        \centering
        \includegraphics[width=\textwidth]{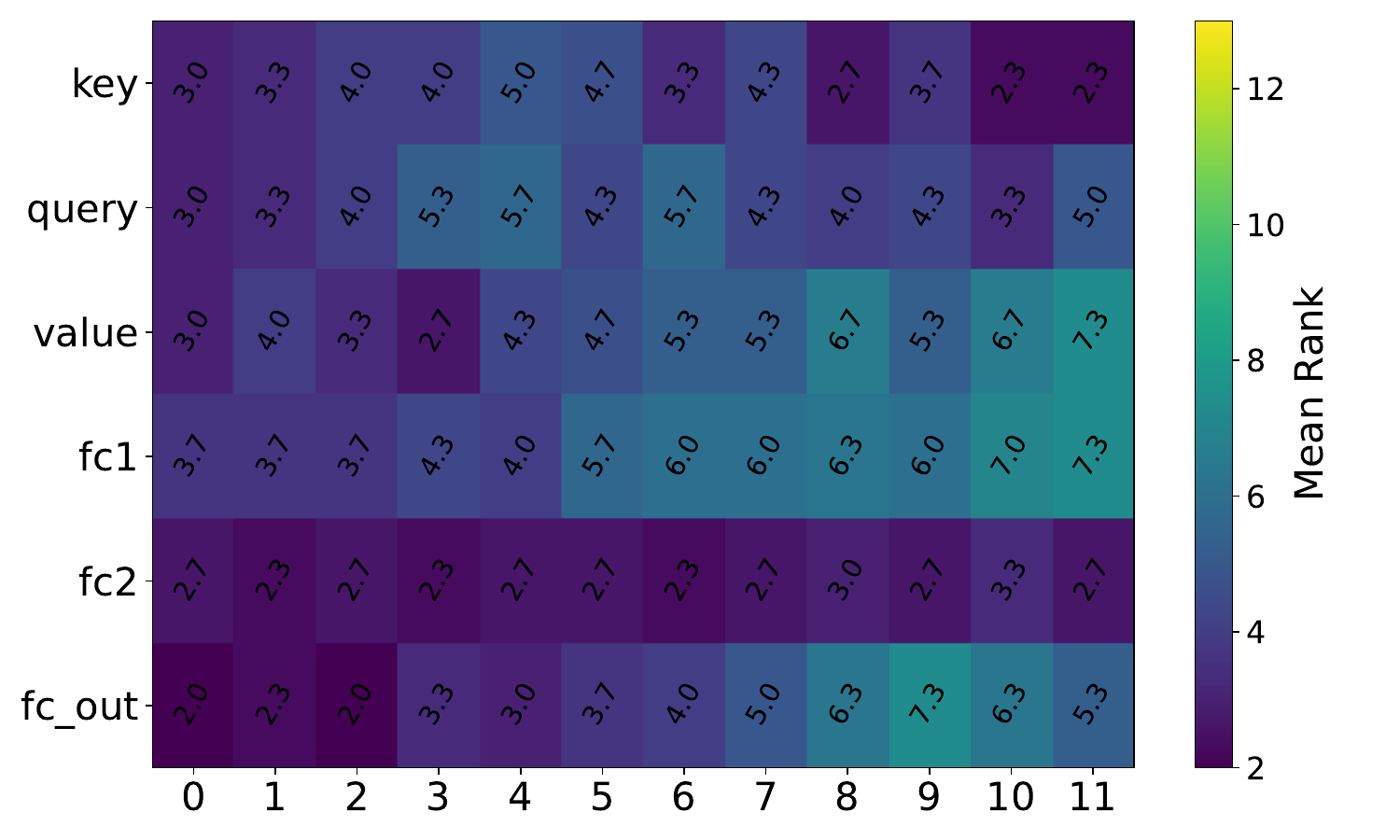}
        \caption{AdaLoRA: mean b=300}
    \end{subfigure}

    \begin{subfigure}[b]{0.4\textwidth}
        \centering
        \includegraphics[width=\textwidth]{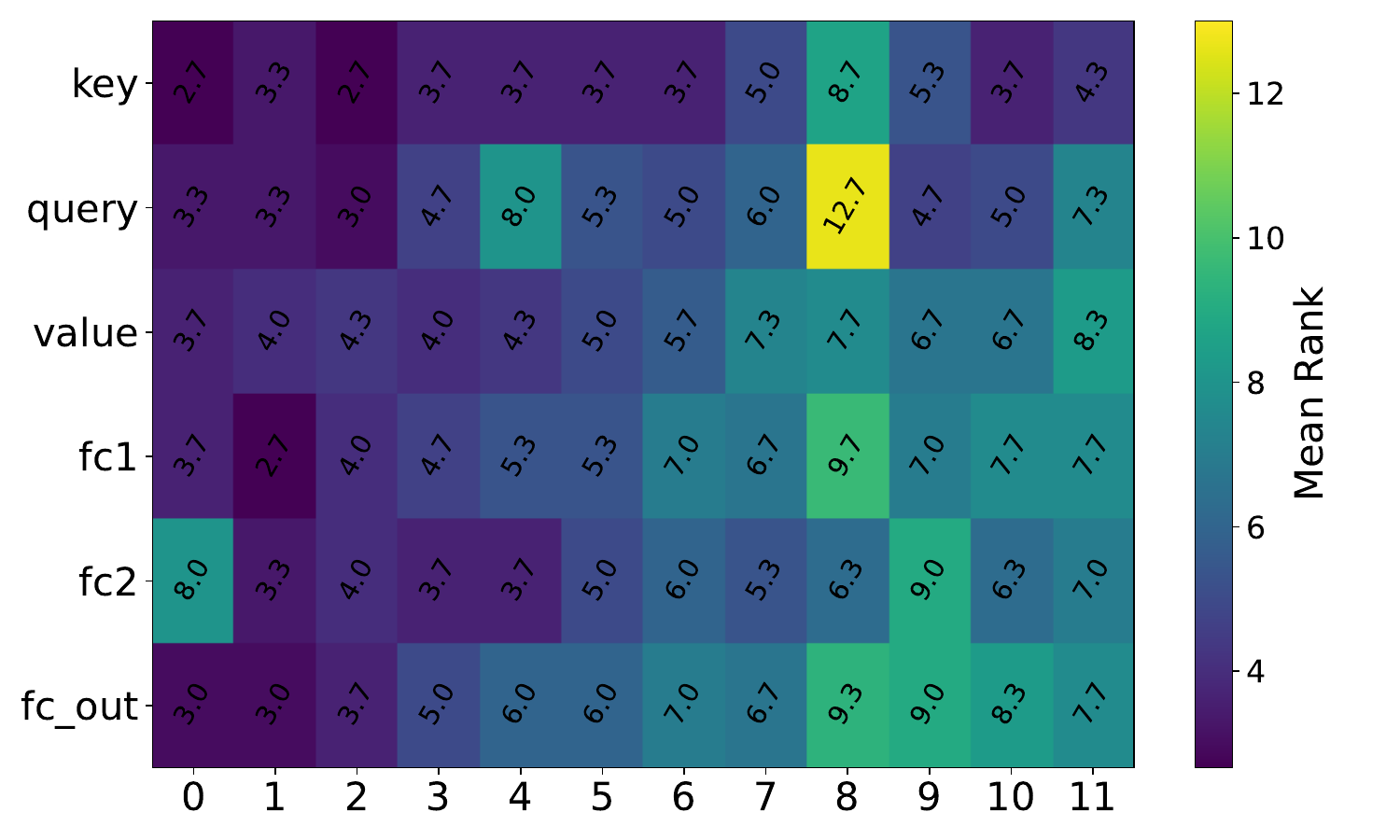}
        \caption{\ALGNAME{}: mean b=400}
    \end{subfigure}
    \begin{subfigure}[b]{0.4\textwidth}
        \centering
        \includegraphics[width=\textwidth]{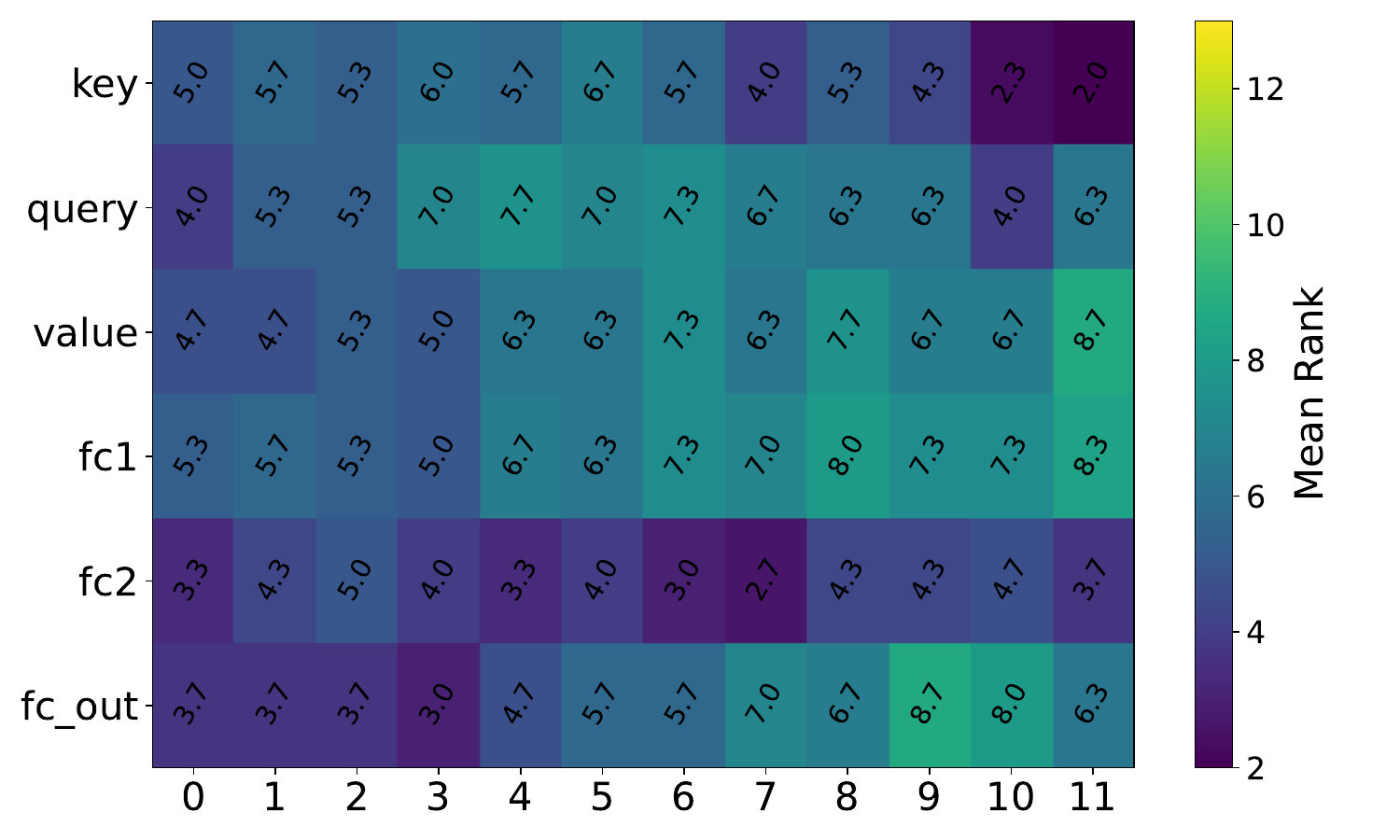}
        \caption{AdaLoRA: mean b=400}
    \end{subfigure}

    \begin{subfigure}[b]{0.4\textwidth}
        \centering
        \includegraphics[width=\textwidth]{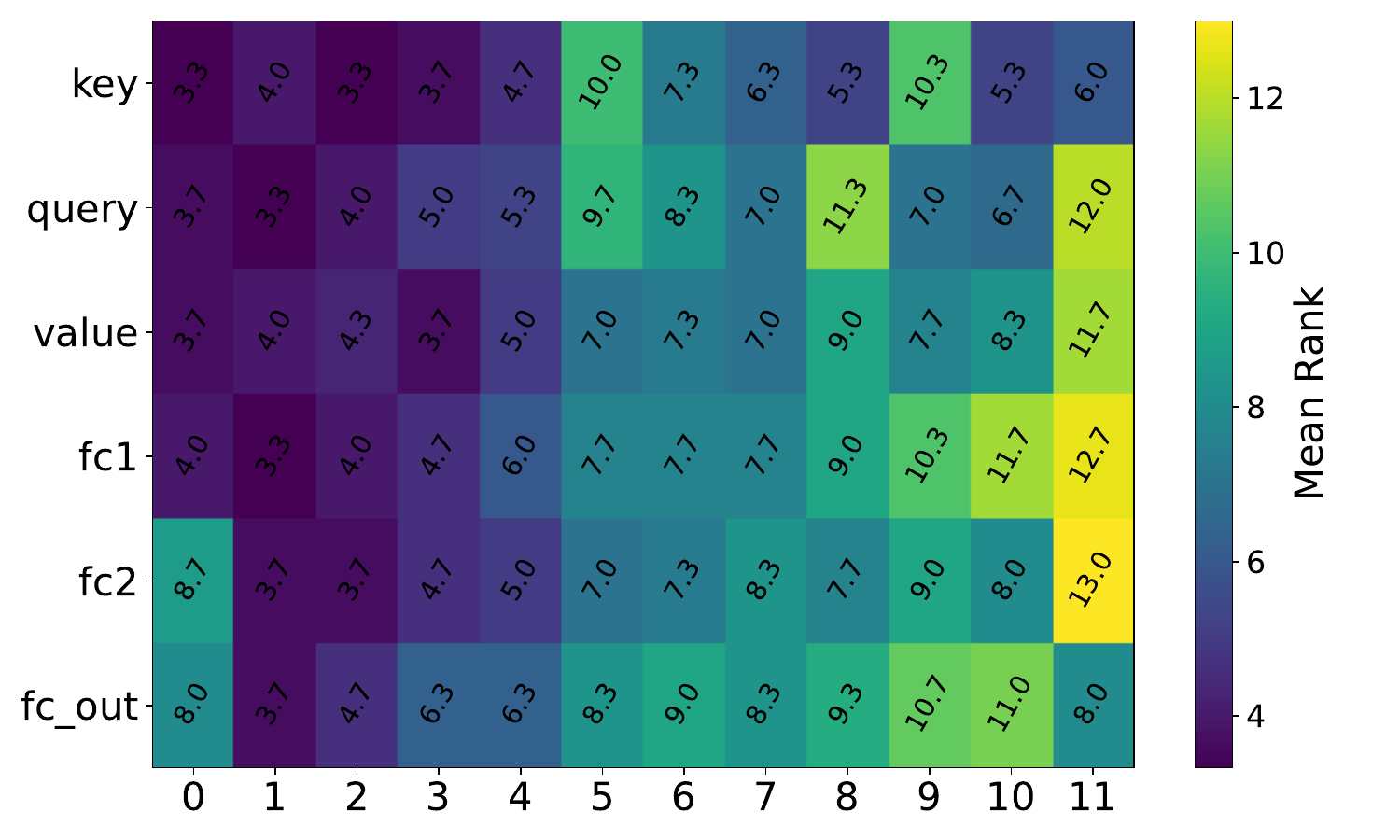}
        \caption{\ALGNAME{}: mean b=500}
    \end{subfigure}
    \begin{subfigure}[b]{0.4\textwidth}
        \centering
        \includegraphics[width=\textwidth]{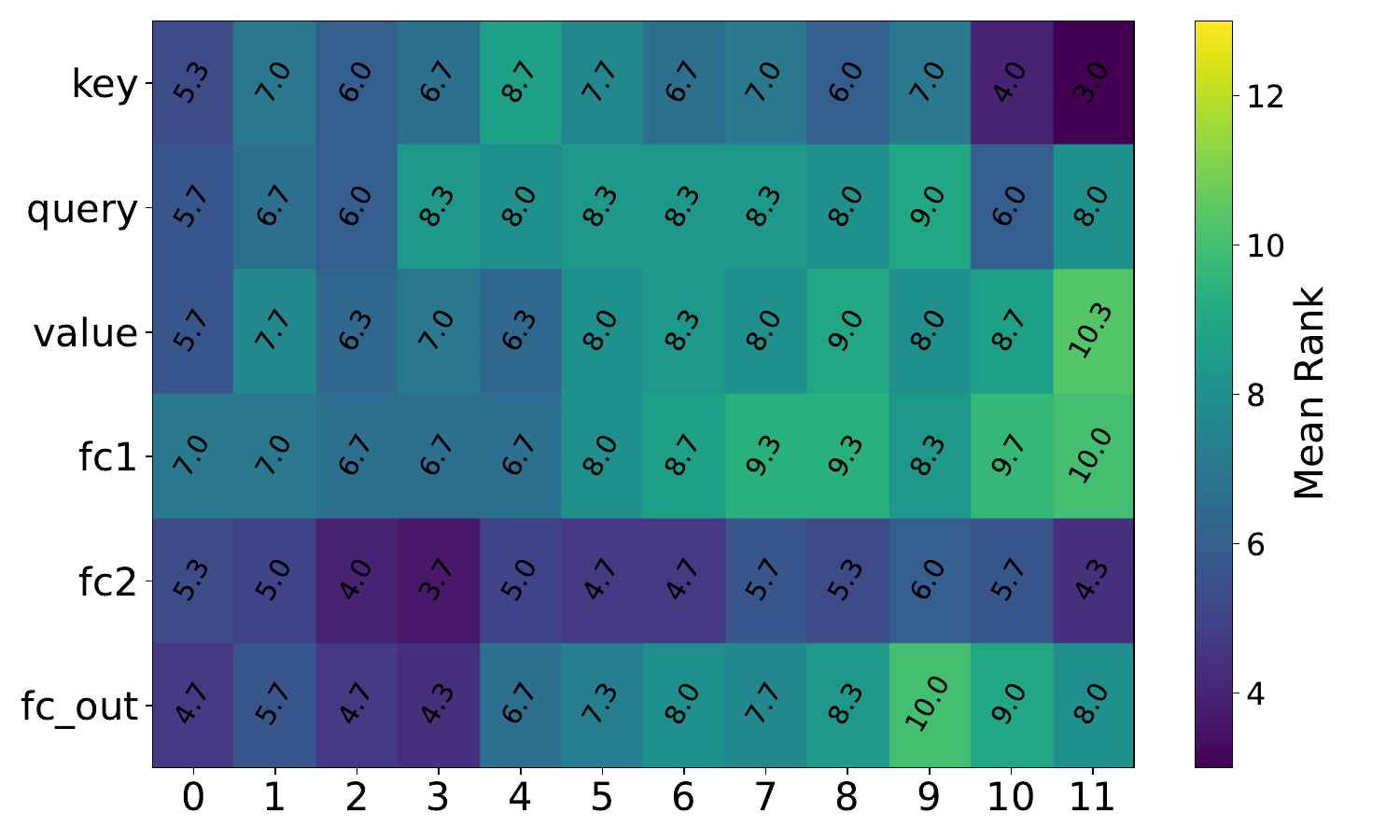}
        \caption{AdaLoRA: mean b=500}
    \end{subfigure}

    \begin{subfigure}[b]{0.4\textwidth}
        \centering
        \includegraphics[width=\textwidth]{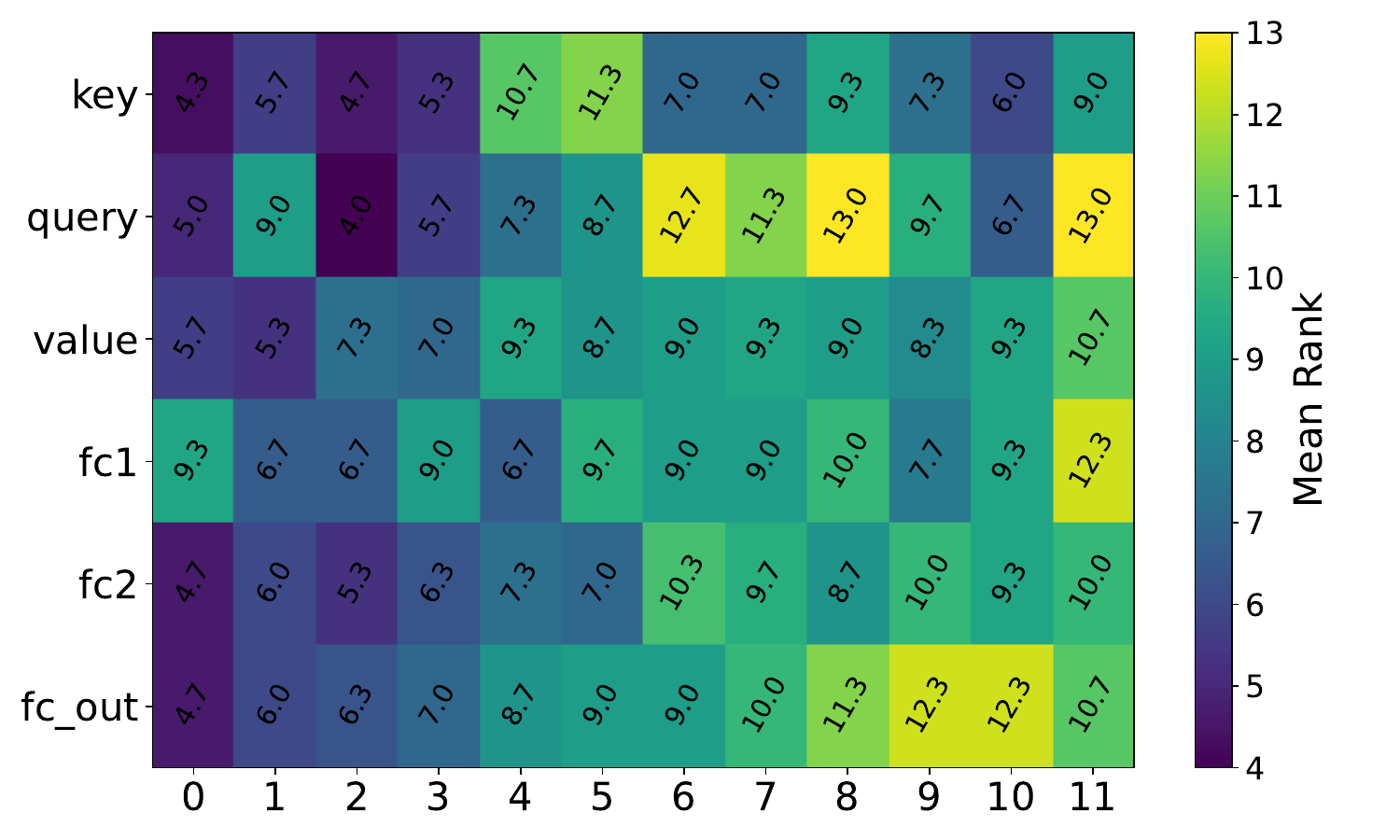}
        \caption{\ALGNAME{}: mean b=600}
    \end{subfigure}
    \begin{subfigure}[b]{0.4\textwidth}
        \centering
        \includegraphics[width=\textwidth]{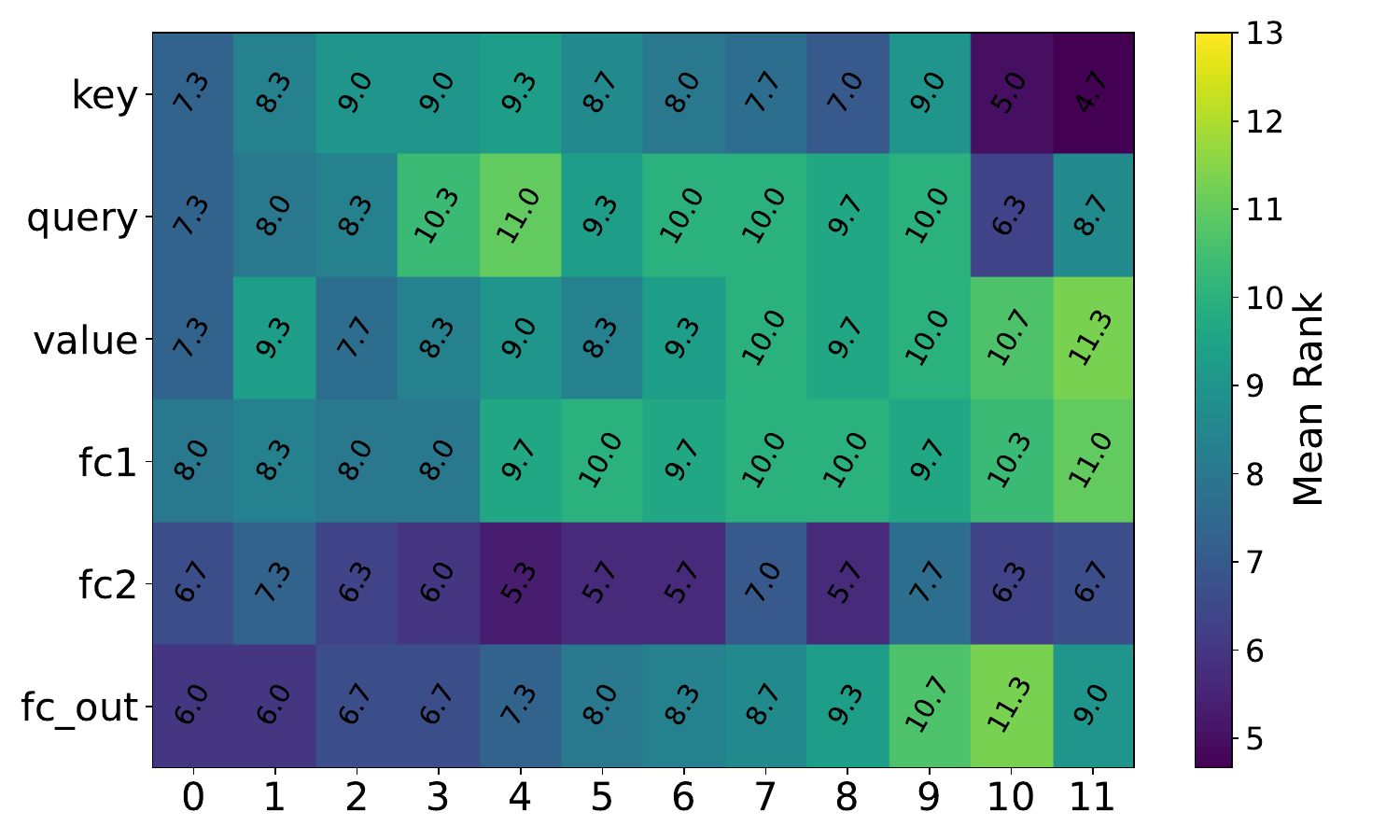}
        \caption{AdaLoRA: mean b=600}
    \end{subfigure}
\caption{Rank distribution of Vit-32b finetuned on Cifar10 for 5 epochs at learning rate $1\rm{e}{-3}$ using \ALGNAME{} and AdaLoRA. }\label{tab_ranks1}
\end{figure}

\begin{figure}[t]
    \centering

    \begin{subfigure}[b]{0.4\textwidth}
        \centering
        \includegraphics[width=\textwidth]{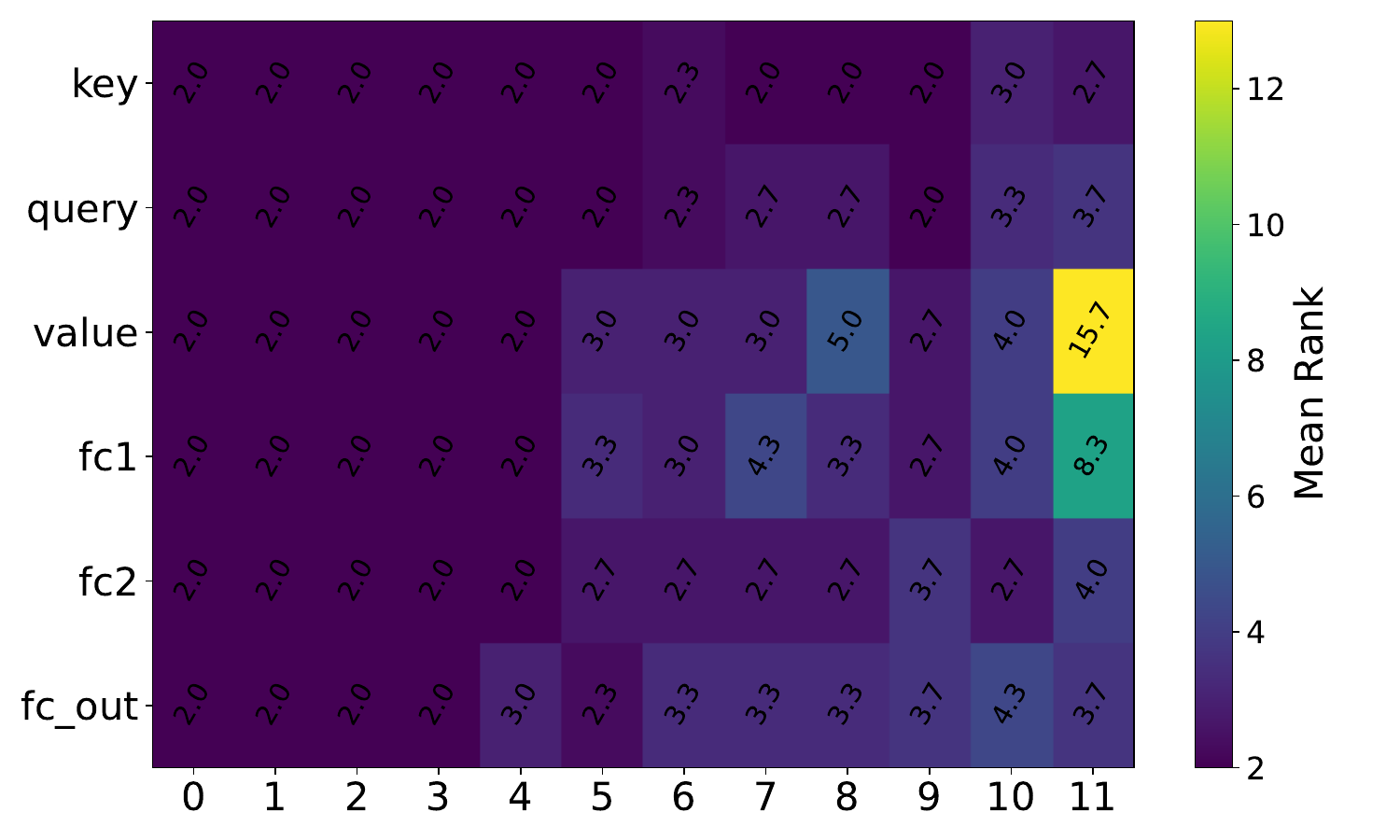}
        \caption{\ALGNAME{}: mean b=200}
    \end{subfigure}
    \begin{subfigure}[b]{0.4\textwidth}
        \centering
        \includegraphics[width=\textwidth]{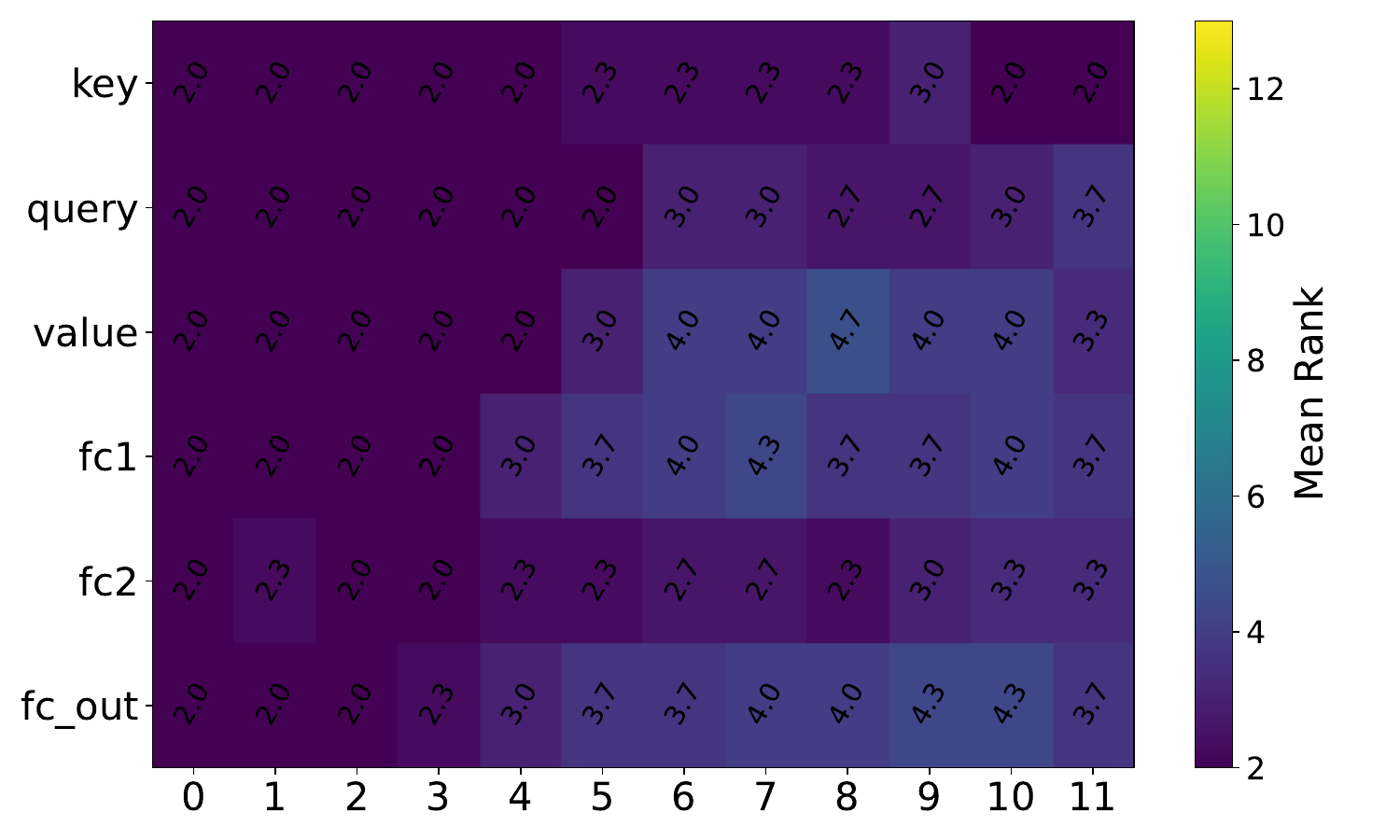}
        \caption{AdaLoRA: mean b=200}
    \end{subfigure}

    \begin{subfigure}[b]{0.4\textwidth}
        \centering
        \includegraphics[width=\textwidth]{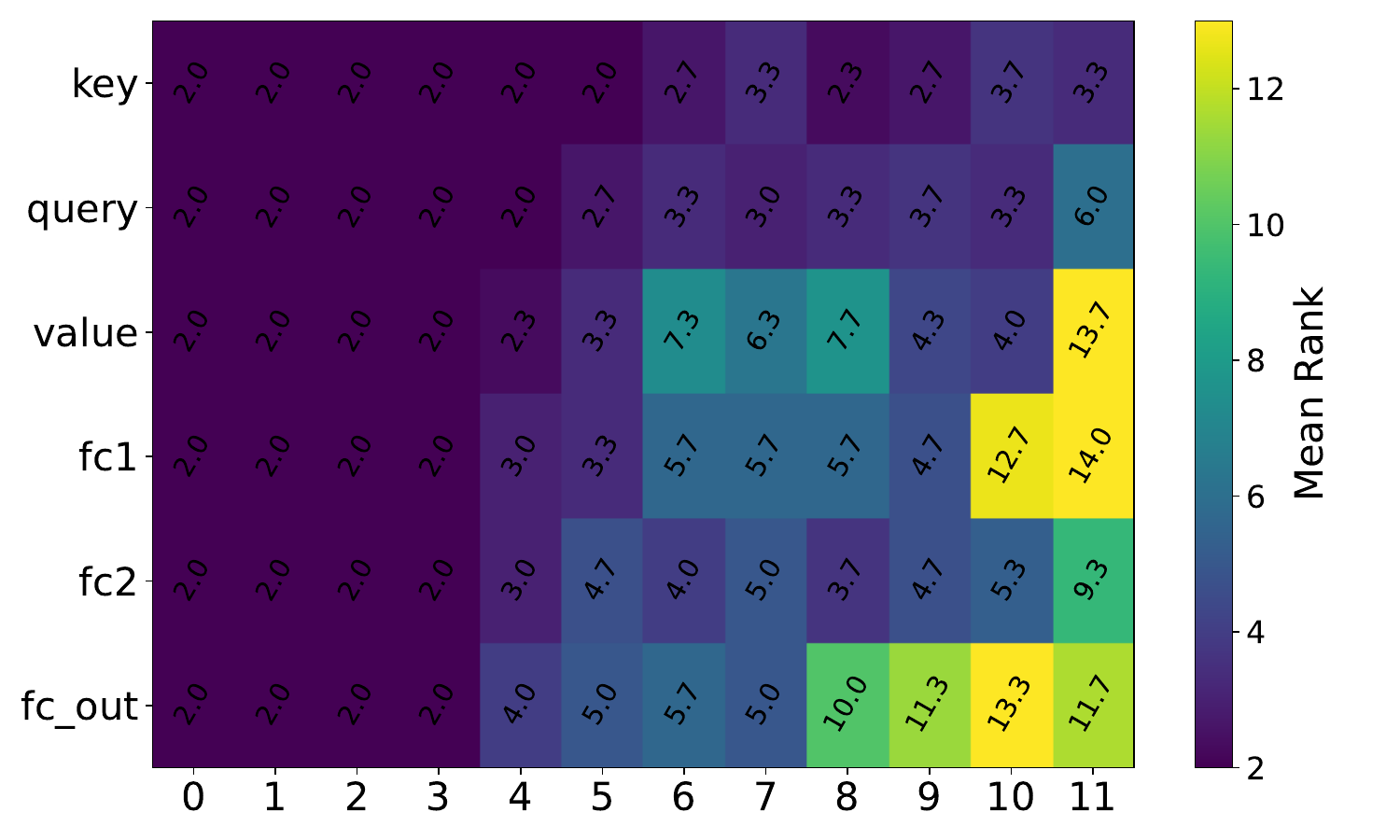}
        \caption{\ALGNAME{}: mean b=300}
    \end{subfigure}
    \begin{subfigure}[b]{0.4\textwidth}
        \centering
        \includegraphics[width=\textwidth]{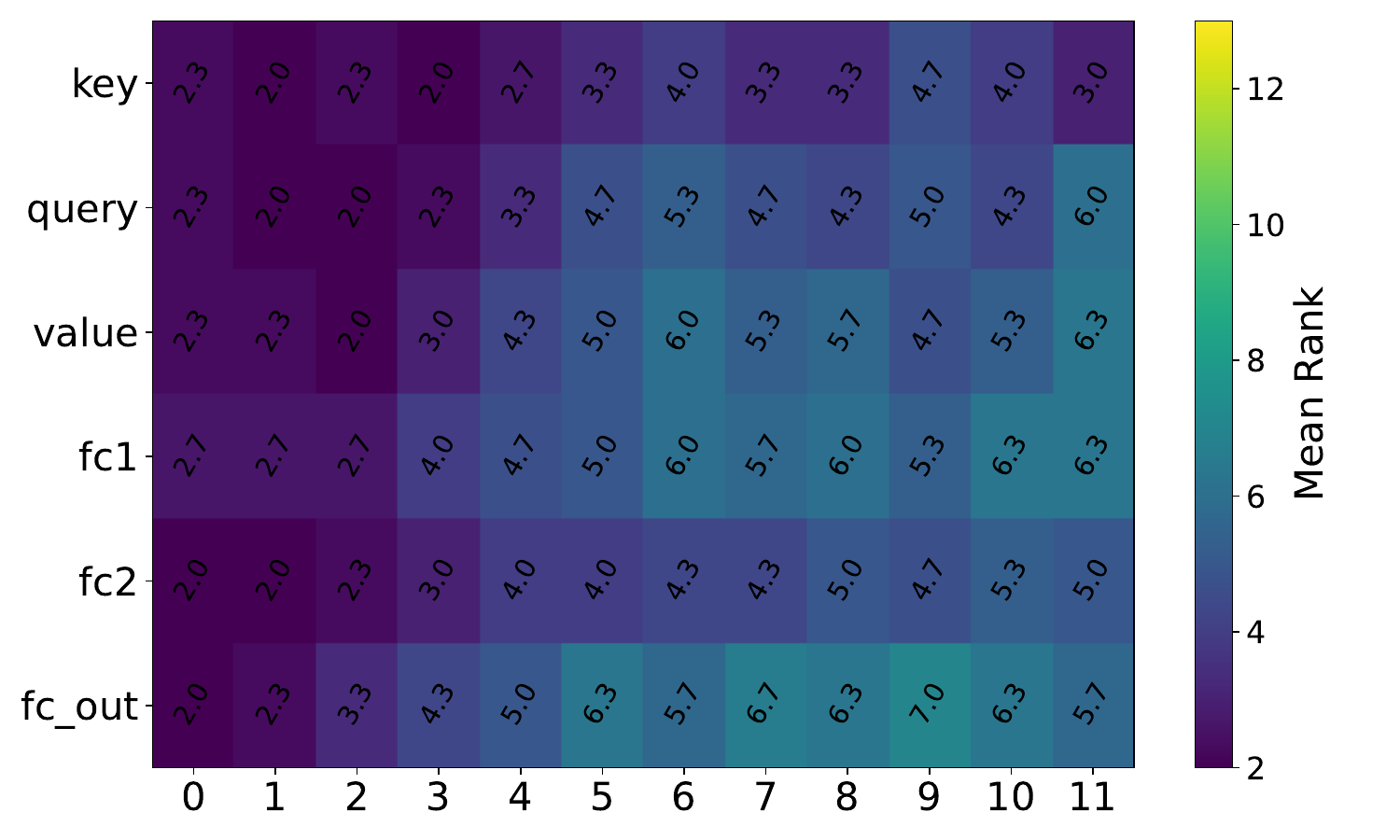}
        \caption{AdaLoRA: mean b=300}
    \end{subfigure}

    \begin{subfigure}[b]{0.4\textwidth}
        \centering
        \includegraphics[width=\textwidth]{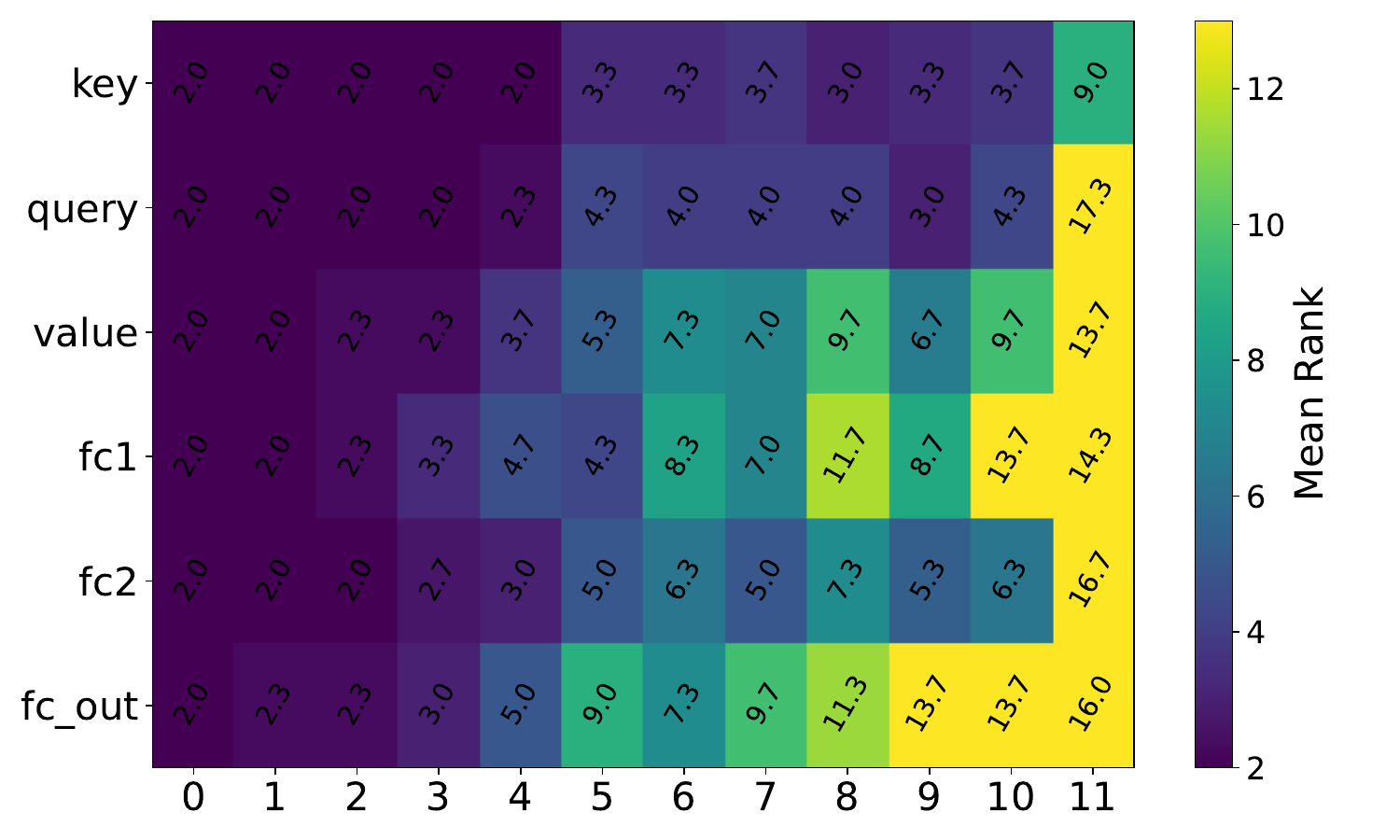}
        \caption{\ALGNAME{}: mean b=400}
    \end{subfigure}
    \begin{subfigure}[b]{0.4\textwidth}
        \centering
        \includegraphics[width=\textwidth]{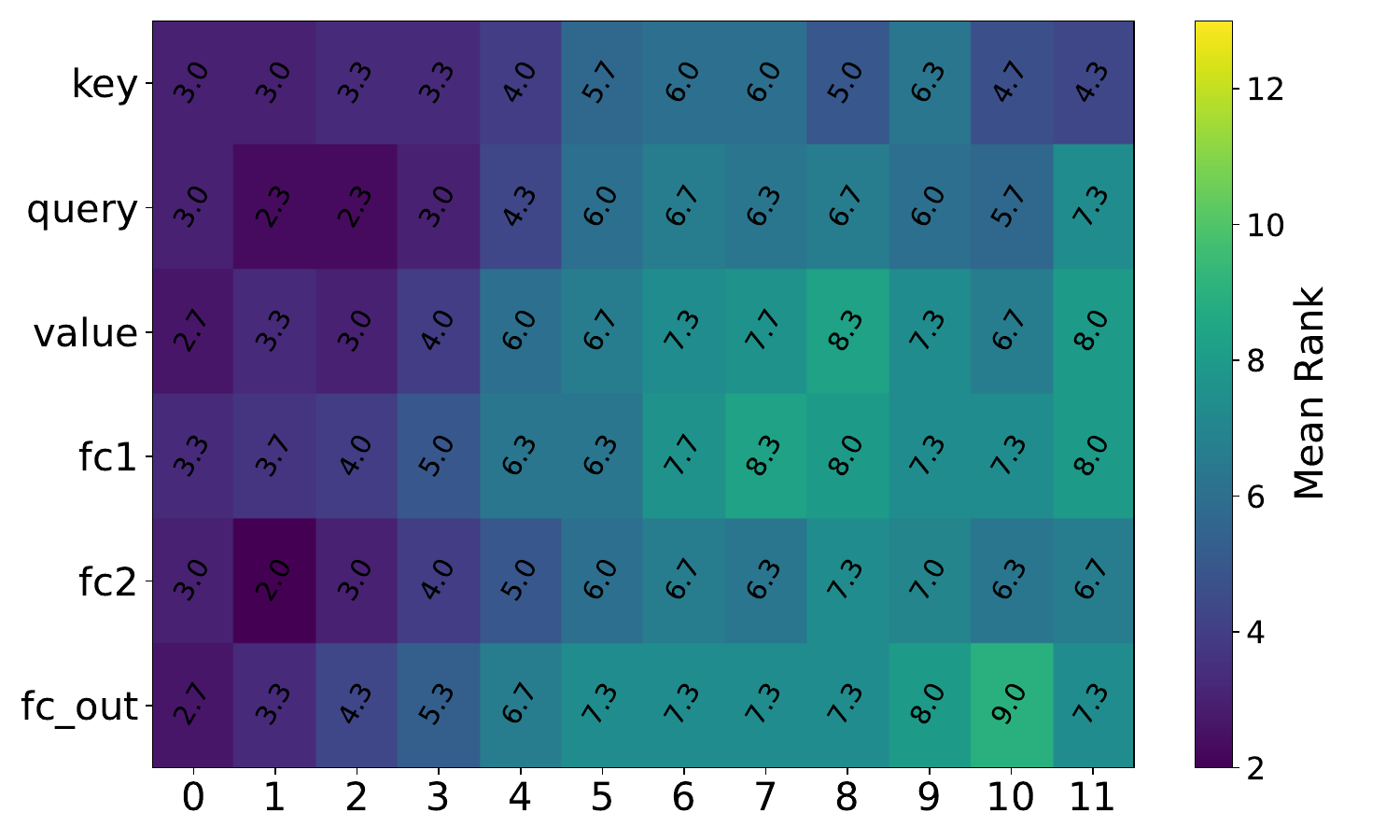}
        \caption{AdaLoRA: mean b=400}
    \end{subfigure}

    \begin{subfigure}[b]{0.4\textwidth}
        \centering
        \includegraphics[width=\textwidth]{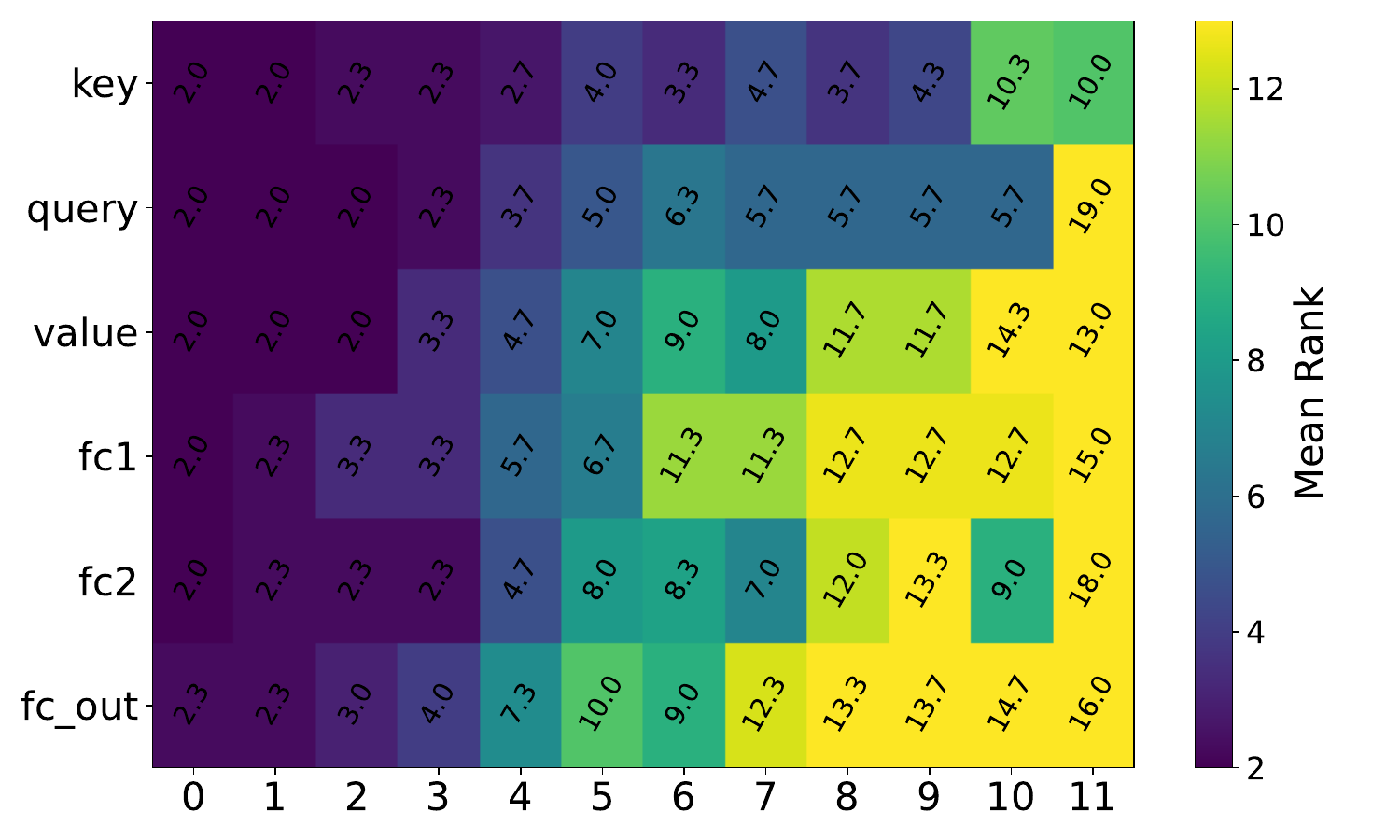}
        \caption{\ALGNAME{}: mean b=500}
    \end{subfigure}
    \begin{subfigure}[b]{0.4\textwidth}
        \centering
        \includegraphics[width=\textwidth]{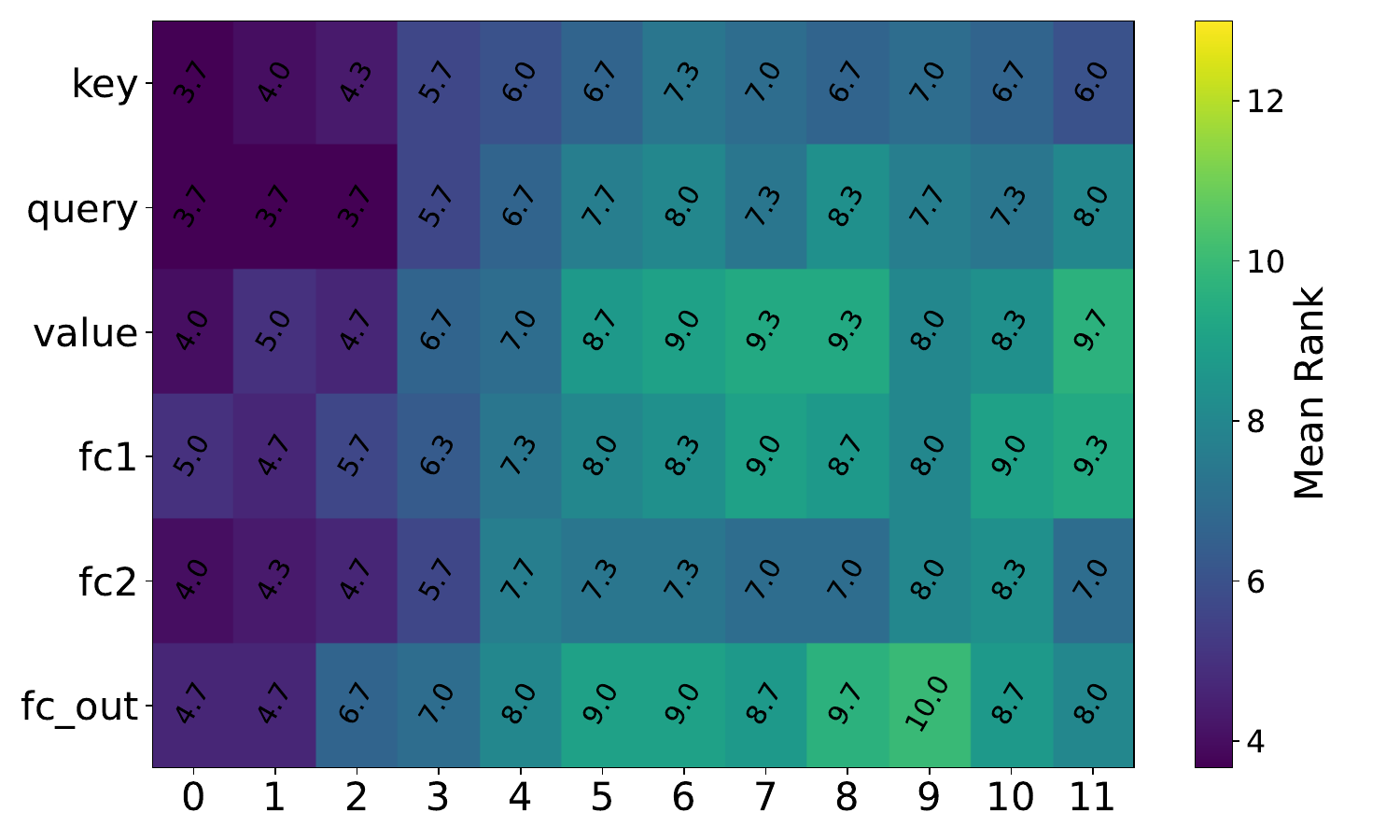}
        \caption{AdaLoRA: mean b=500}
    \end{subfigure}

    \begin{subfigure}[b]{0.4\textwidth}
        \centering
        \includegraphics[width=\textwidth]{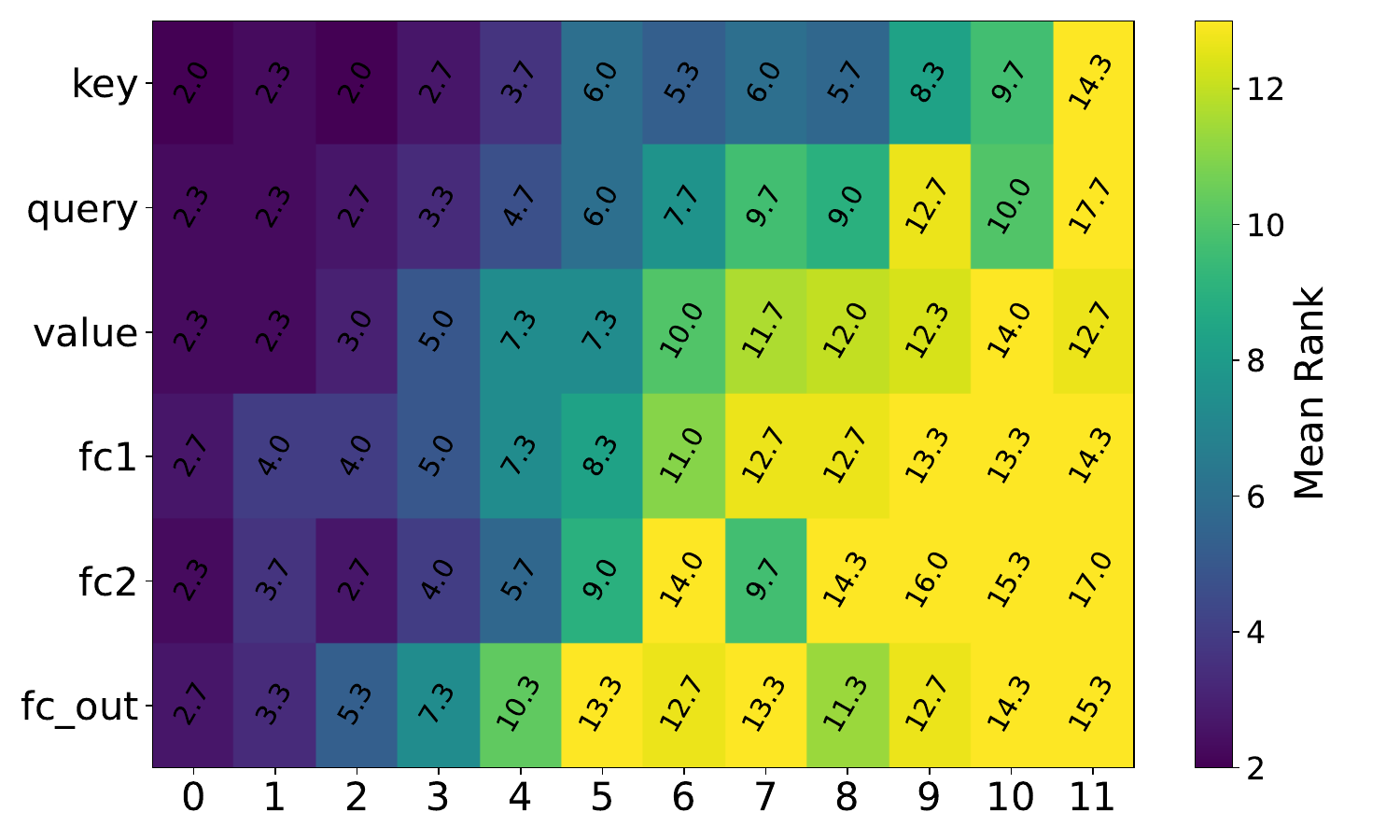}
        \caption{\ALGNAME{}: mean b=600}
    \end{subfigure}
    \begin{subfigure}[b]{0.4\textwidth}
        \centering
        \includegraphics[width=\textwidth]{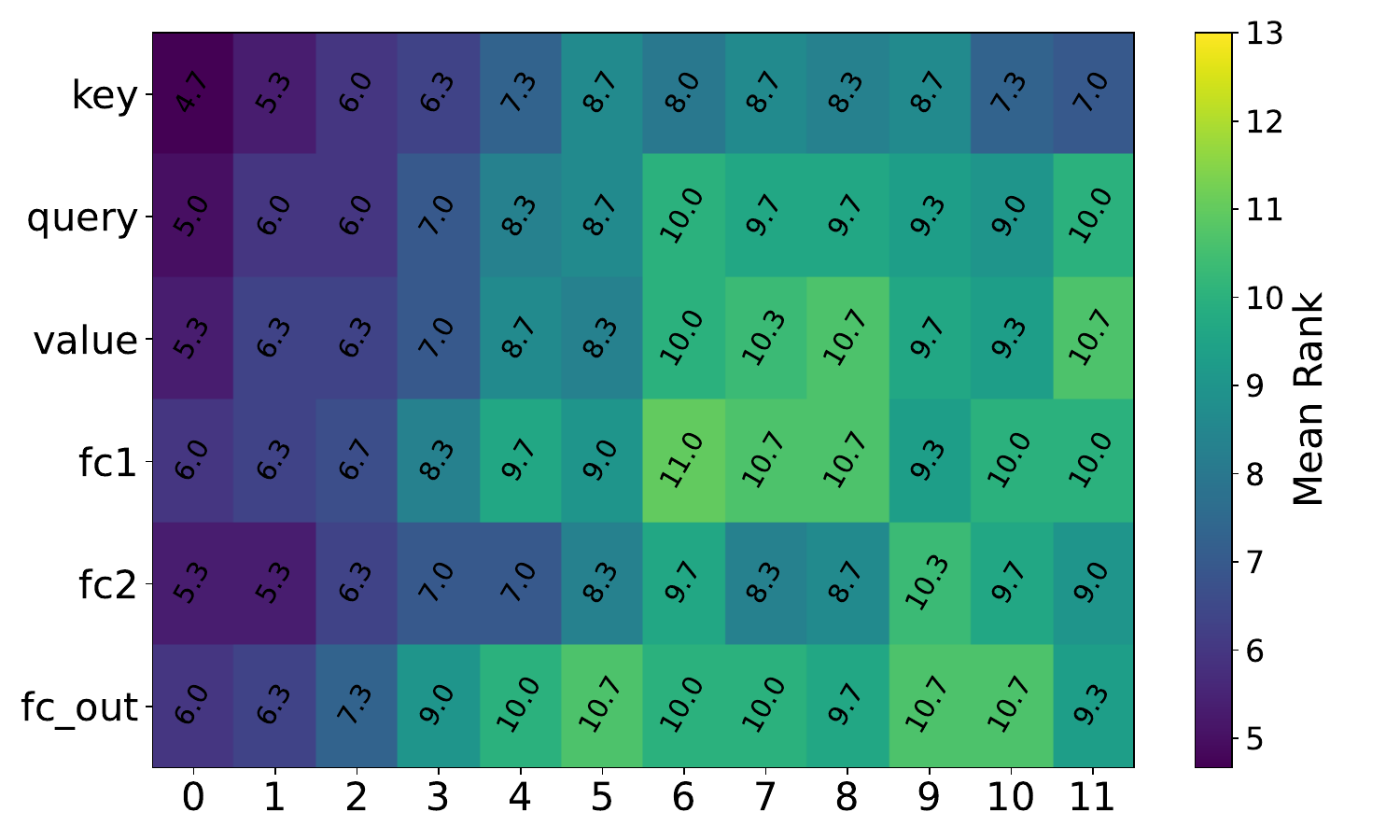}
        \caption{AdaLoRA: mean b=600}
    \end{subfigure}
\caption{Rank distribution of Vit-32b finetuned on Cifar10 for 5 epochs at learning rate $1\rm{e}{-4}$ using \ALGNAME{} and AdaLoRA. }\label{tab_rank2}
\end{figure}
\newpage
\subsection{Stable Diffusion on Dreambooth,}\label{app_stable_diff}
In this numerical example, we apply low-rank adapters to all linear and attention layers of the U-Net and the text 
encoder networks.  The hyperparameters for LoRA and \ALGNAME{} are the same, apart from the fact that we start 
with adapters of rank $8$ for both Unet and text encoder. We train for $5$ full epochs, using adamW as an optimizer, with 
$(\beta_1,\beta_2) = (0.9,0.999)$, initial learning rate $5 \times 10^{-6}$ and weight decay set to $10^{-2}$.

\newpage
\section{Overview for the numerical analysis}
\subsection{Notation}
We provide an overview of the notation used throughout the main manuscript and the appendix.
\begin{itemize}[leftmargin=*,noitemsep,topsep=0em]
\item $W\in\mathbb{R}^{n\times n}$ is the full-rank weight matrix of a neural network layer or adapter.
\item $Z\in\mathbb{R}^{n\times n}$ is an arbitrary matrix.
\item $\mathcal{M}_r = \set{Z\in\mathbb{R}^{n \times n}: \text{rank}(Z)=r}$ is a manifold of rank $r$ matrices.
\item  $\mathcal{T}_{Z} \mathcal{M}_{r}$ is the tangent space of $\mathcal{M}_r$ at $Z$ for any $Z\in\mathbb{R}^{n\times n}$.
\item  $W^r =USV^\top\in \mathcal{M}_r$ is a rank-$r$ approximation of a matrix $W$.
\item  $\augWr =\augU\augS\augV^\top\in \mathcal{M}_r$ is a rank-$r$ approximation of a matrix $W$ with augmented basis.
\item $U,V\in\mathbb{R}^{n\times r}$ is the orthonormal basis and co-basis $\mathcal{M}_r$.
\item $\augU = [U_0,\newU]\in\mathbb{R}{n\times 2r}$ is the augmented basis. (Analogously for $\augV$.)
\item $U_0\in\mathbb{R}^{n\times r}$ is the basis at the beginning of the iteration.  (Analogously for $V_0$.)
\item $\newU\in\mathbb{R}^{n\times r}$ is the basis augmentation, obtained by $[U_0,\newU] = \textup{ortho}([U_0,K^{\text{new}}])\in\mathbb{R}^{n\times 2r_0}$. (Analogously for $\newV$.)
\item $S\in\mathbb{R}^{r\times r}$ is the coefficient matrix so assemble the low-rank approximation $W^r$ from $U,V$.
\item  $P(Z)$ is the orthogonal projection onto $\mathcal{T}_{Z}\mathcal{M}_r$. 
\item  $P_U=UU^\top$ is the orthogonal projection onto the range of orthonormal $U\in\mathbb{R}^{n\times r}$.
\item  $P_V=VV^\top$ is the orthogonal projection onto the range of orthonormal $V\in\mathbb{R}^{n\times r}$.
\item When applied to vectors, $\norm{\cdot}$ denotes the Euclidean norm ($\ell_2$-norm). When applied to matrices, $\norm{\cdot}$ denotes the Frobenius norm.
\item $\mathcal{L}(W;\xi)$ denotes the loss function dependent on weight matrix $W$ and data sample randomness $\xi$. Commonly abbreviated by $\mathcal{L}(W)$.
\item $f(W;\xi)=-\nabla_W\mathcal{L}(W;\xi)$ is the negative stochastic loss gradient w.r.t $W$. Commonly abbreviated by $f(W)$.
\item  $F(W)=\mathbb{E}_\xi [f(W,\xi)]$ the expectation of the random loss gradient, called the deterministic gradient.
\end{itemize}
\subsection{Recap of commonly used properties}\label{app_recap}
We recapitulate repeatedly used properties of the mathematical objects and notations introduced above. 
\begin{itemize}[leftmargin=*,noitemsep,topsep=0em]
\item Per definition, we have for any $Z\in\mathbb{R}^{n\times n}$
    \begin{align}\label{def_projection}
        P(W^r)Z = UU^\top Z +  Z VV^\top - UU^\top Z VV^\top
    \end{align}
\item Since $\mathcal{T}_{Z}\mathcal{M}_r$ is a subspace of $\mathbb{R}^{n\times n}$ for any $Z\in\mathcal{M}_r$ we can decompose the gradients $F$ and $f$ into $F(Z)=M(Z)+R(Z)$ and $f(Z)=f(Z)+r(Z)$ for any $Z\in\mathcal{R}^{n\times n}$, where $M(Z),m(Z)\in\mathcal{T}_{Z}\mathcal{M}_r$.
    \end{itemize}
\subsection{Global assumptions}
The following provides a comprehensive overview of the global assumptions for made in the analysis section and proofs of the provided theorems. The assumptions are common in literature, see e.g. \citep{Hnatiuk}
\begin{assumption}\label{ass_1}
    There is an $\varepsilon>0$ such that $\norm{R(Z)},\norm{r(Z)}\leq\epsilon$ for all $Z\in\mathcal{M}_r$
\end{assumption}
\begin{assumption}\label{ass_2}
   $F$ and $f$ are bounded by a constant $B>0$ and L-continuous w.r.t. $\norm{\cdot}$ with constant $L>0$.
\end{assumption}
\begin{assumption}\label{ass_3}
   There is a constant $C>0$ such that $\norm{F(Z)-f(Z)}\leq C$.
\end{assumption}
\begin{assumption}\label{ass_4}
    At initial time, we assume that the difference of a full-rank weight matrix $W_0$ and its low-rank counterpart $W^r_0$ is bounded by $\norm{W_0-Y_0}\leq \delta$ for $\delta>0$.
\end{assumption}
\begin{assumption}\label{ass_5}
    For all times, we have w.l.o.g $\mathcal{L}(t)>0$.
\end{assumption}

\section{Descent direction}\label{app_descend_direction}
We first state a few auxiliary lemmas, which provide common inequalities that will be used in the following analysis.
\begin{lemma}{\citep[Lemma 5.2]{Hnatiuk}}\label{lem_loss_function_est}
    For any two matrices $Y_1,Y_2\in\mathbb{R}^{n\times n}$ and an $L$-smooth $\mathcal{L}$ with constant $L$     it holds
    \begin{align}
        \mathcal{L}(Y_1)-\mathcal{L}(Y_2)\leq -\inner{Y_1-Y_2,f(Y_2)} + \frac{L}{2}\norm{Y_1-Y_2}^2,
    \end{align}
    where $f(Y)=-\nabla_Y\mathcal{L}(Y)$. Furthermore, it holds 
    \begin{align}
        \mathcal{L}(Y_1)-\mathcal{L}(Y_2)\leq -\inner{Y_1-Y_2,F(Y_2)} + \frac{L}{2}\norm{Y_1-Y_2}^2,
    \end{align}
    where $F(Y)=-\mathbb{E}[\nabla_Y\mathcal{L}(Y)]$.
\end{lemma}

The following results are primarily based on \citep{Hnatiuk} and use the reformulation of truncated terms as proposed in \citep{zangrando2023rank}. For ease of notation, we use $f(W^{r}_{{t}})=f(W^{r}_{{t}},\xi_t)$. Hence, randomness is not explicitly stated in our notation. Note that in this case, the factorized solution $W^{r}_{{1}} = U_1 S_1 V_1^{\top}$ is random since it depends on $f(W^{r}_{{1}})$. When using expected values, we explicitly write down the corresponding random variable. That is, $\mathbb{E}_{\xi}[\cdot]$ is the expected value for a random variable $\xi$. We denote the random variable in step $T$ as $\xi_{T}$ and denote $\mathbb{E}[\cdot] := \mathbb{E}_{\xi_1, \cdots, \xi_{T}}[\cdot]$.

\begin{theorem}(Restatement of \Cref{theo_descend_direction})
   \Cref{alg_efficient_TDLRT} with stochastic (mini-batch) gradients fulfills
\begin{align}\label{the0_2_result}
  \mathbb{E}_{\xi_{t+1}}[\mathcal{L}(W^{r}_{t+1})]  \leq \mathcal{L}(W^{r}_{t}) -\lambda\left(1-\frac{L\lambda^2}{2}\right)\mathbb{E}_{\xi_1}[\Vert P(W^{r}_{t})f(W^{r}_{t},\xi_t) \Vert^2] +L\mathbb{E}_{\xi_1}[ \|W^{r}_{t+1}- \augWr_t \|]\,.
\end{align}
where $W^{r}_{{t}}$, $\augWr_{t}$, $W^{r}_{t+1}$ are the low-rank weight matrices at the start of iteration $t+1$, before, and after the truncation step, respectively. The step size is given by $\lambda$. 
\end{theorem}
\begin{proof}
    Without loss of generality, we restrict ourselves to time steps $t=0$ and write $f(W_0^r)$ shorthand for $f(W^{r}_{t=0},\xi_t)$.
By definition of the coefficient matrix assembly in \cref{eq_S_step}, we get respectively
\begin{itemize}
    \item  $\widetilde U\widetilde U^{\top}f(W^{r}_{0})V_0V_0^{\top}$  for the right hand side of the $S^{\textup{new}}$ block
    \item $ U_0U_0^{\top}f(W^{r}_{0})\widetilde V\widetilde V^{\top}$  for the right hand side of the $L^{\textup{new}}$ block
      \item $ \widetilde U\widetilde U^{\top}f(W^{r}_{0})V_0V_0^{\top}$  for the right hand side of the $K^{\textup{new}}$ block
      \item and zero for the lower right block. 
\end{itemize}
Since the augmented bases are orthonormal, we can write for $W^{r}_{0} = U_0S_0V_0$
\begin{align*}
    \augWr_{0}  \overset{(\ref{eq_S_step})}{=}\,& W^{r}_{0} + \lambda U_0U_0^{\top}f(W^{r}_{0})V_0V_0^{\top} + \lambda \widetilde U\widetilde U^{\top}f(W^{r}_{0})V_0V_0^{\top} + \lambda U_0U_0^{\top}f(W^{r}_{0})\widetilde V\widetilde V^{\top} \\
    =\,& W^{r}_{0} - \lambda U_0U_0^{\top}f(W^{r}_{0})V_0V_0^{\top} + \lambda \augU\augU^{\top}f(W^{r}_{0})V_0V_0^{\top} + \lambda U_0U_0^{\top}f(W^{r}_{0})\augV\augV^{\top}\\
    =\,& W^{r}_{0} - \lambda  U_0U_0^{\top}f(W^{r}_{0})V_0V_0^{\top} + \lambda  f(W^{r}_{0})V_0V_0^{\top} + \lambda U_0U_0^{\top}f(W^{r}_{0})\\
    \overset{ (\ref{def_projection})}{=}\,& W^{r}_{0} + \lambda  P(W^{r}_{0})f(W^{r}_{0}).
\end{align*}

By \Cref{lem_loss_function_est} we have
\begin{align}\label{eq:decdir1}
    \mathcal{L}(\augWr_{0}) - \mathcal{L}(W^{r}_{0}) \leq - \langle f(W^{r}_{0}), \augWr_{0} -W^{r}_{0} \rangle + \frac{L}2 \Vert \augWr_{0} - W^{r}_{0} \Vert^2.
\end{align}
Therefore, plugging the above equation into \cref{eq:decdir1} yields
\begin{align}\label{eq:helper_0}
    \mathcal{L}(\augWr_{0}) - \mathcal{L}(W^{r}_{0}) \leq\,& - \lambda\langle f(W^{r}_{0}),P(W^{r}_{0})f(W^{r}_{0}) \rangle + \frac{L\lambda^2}{2}\Vert P(W^{r}_{0})f(W^{r}_{0}) \Vert^2 \\
    =\,& - \lambda\langle P(W^{r}_{0})f(W^{r}_{0}),P(W^{r}_{0})f(W^{r}_{0}) \rangle + \frac{L\lambda^2}{2}\Vert P(W^{r}_{0})f(W^{r}_{0}) \Vert^2 \,\\
    =\,& - \lambda\left(1-\frac{L\lambda^2}{2}\right)\Vert P(W^{r}_{0})f(W^{r}_{0}) \Vert^2 \,.
\end{align}
where the second line is obtained by definition of the orthogonal projection.
Comparing the loss before $\augWr$ and after $W^{r}_{1}$ truncation yields for some $s\in(0,1)$ using the mean value theorem and the Cauchy-Schwarz inequality,
\begin{equation}\label{eq:drift_truncation}
\mathcal{L}(W^{r}_{1}) \leq \mathcal{L}(\augWr_0+ \langle \nabla \mathcal{L}(sW^{r}_{1}+(1-s)\augWr),W^{r}_{1}-\augWr_0 \rangle \leq \mathcal{L}(\augWr_0)+L \|W^{r}_{1}- \augWr_0 \|.
\end{equation}
Plugging \cref{eq:drift_truncation} into \cref{eq:helper_0} then gives
\begin{align*}
      \mathcal{L}(W^{r}_{1}) - \mathcal{L}(W^{r}_{0}) \leq 
    - \lambda\left(1-\frac{L\lambda^2}{2}\right)\Vert P(W^{r}_{0})f(W^{r}_{0}) \Vert^2  +L \|W^{r}_{1}- \augWr_0 \|,
\end{align*}
where $L$ is the Lipschitz constant of $F$.
Hence, taking the expected value yields
\begin{align*}
    \mathbb{E}_{\xi_1}[\mathcal{L}(W^{r}_{1})] \leq \mathcal{L}(W^{r}_{0}) -\lambda\left(1-\frac{L\lambda^2}{2}\right)\mathbb{E}_{\xi_1}[\Vert P(W^{r}_{0})f(W^{r}_{0}) \Vert^2] +L\mathbb{E}_{\xi_1}[ \|W^{r}_{1}- \augWr_0 \|]\,.
\end{align*}
\end{proof}
\section{Convergence}\label{app_convergence}
\begin{theorem}{(Restatement of \Cref{theo_convergence})}
    Let $\mathcal{L}\geq 0$ and $W^{r}_{1},\dots,W^{r}_T$ be the solutions generated by \Cref{alg_efficient_TDLRT} over $T$ steps. Let the learning rate sequence $\{\lambda_t\}$ satisfy the Robbins-Monro conditions:
    \[
    \textstyle{\sum_t \lambda_t =+\infty \qquad \sum_t \lambda_t^2 < +\infty \, .}
    \]
Further assume $\sum_{t=1}^{T-1}\mathbb{E}[\|W^{r}_{t+1} - \augWr_{t} \|] \leq D < \infty$, i.e.  after some time, the solution $W^{r}_t$ is contained in a manifold of rank $r$. 
Then we have
       \begin{align*}
        \liminf_{T\rightarrow\infty} \mathbb{E}[\Vert P(W^{r}_{T})f(W^{r}_{T}) \Vert^2] = 0\,,
    \end{align*}
    where the expected value is taken over all $\xi_t$.
\end{theorem}
\begin{proof}
    By taking the expected value over $\xi_{1},\dots,\xi_{T}$ in \cref{the0_2_result} and denoting the corresponding expected value as $\mathbb{E}[\cdot]$ we get
    \begin{align*}
        \mathbb{E}[\mathcal{L}(W^{r}_{t+1})] - \mathbb{E}[\mathcal{L}(W^{r}_{t})] \leq  -&{\lambda_t}\mathbb{E}[\Vert P(W^{r}_{t})f(W^{r}_{t}) \Vert^2] + \frac{L{\lambda_t}^2}{2}\mathbb{E}[\Vert P(W^{r}_{t})f(W^{r}_{t}) \Vert^2]\\
        +&L \mathbb{E}[\|W^{r}_{t+1} - \augWr_{t} \|] \, \\
        = &-\lambda_t\left(1-\frac{L{\lambda_t}}{2}\right)\mathbb{E}[\Vert P(W^{r}_{t})f(W^{r}_{t}) \Vert^2]  + L \mathbb{E}[\|W^{r}_{t+1} - \augWr_{t} \|] \,.
    \end{align*}
    Using a telescoping sum until $t=T$ then yields
    \begin{align*}
        - \mathcal{L}(Y_{{0}}) \leq \mathbb{E}[\mathcal{L}(W^{r}_{t})] - \mathcal{L}(Y_{{0}}) \leq -\,&\sum_{t=1}^{T-1}\lambda_t\left(1-\frac{L{\lambda_t}}{2}\right)\mathbb{E}[\Vert P(W^{r}_{t})f(W^{r}_{t}) \Vert^2] \\
        &+ L \sum_{t=1}^{T-1}\mathbb{E}[\|W^{r}_{t+1} - \augWr_{t} \|]\,.
    \end{align*}
    Rearranging gives
    \begin{align*}
        \sum_{t=1}^{T-1}\lambda_t \left(1-\frac{L{\lambda_t}}{2}\right)\mathbb{E}[\Vert P(W^{r}_{t})f(W^{r}_{t}) \Vert^2] &\leq \mathcal{L}(Y_{{0}}) + L \sum_{t=1}^{T-1}\mathbb{E}[\|W^{r}_{t+1} - \augWr_{t+1} \|]\,. \\
       & \leq \mathcal{L}(Y_{{0}}) + L D\,.
    \end{align*}
    Using the assumptions $\Vert P(W^{r}_{t})f(W^{r}_{t}) \Vert \leq B$ and $\sum_{t=1}^{T-1}\mathbb{E}[\|W^{r}_{t+1} - \augWr_{t+1} \|]\leq D$.
    Now, when $T\rightarrow \infty$, then the right-hand side remains bounded, implying that
    \begin{align*}
        \liminf_{T\rightarrow\infty} \mathbb{E}[\Vert P(W^{r}_{t})f(W^{r}_{t}) \Vert^2] = 0\,.
    \end{align*}
\end{proof}

\section{Efficient evaluation of the right hand side of the low-rank dynamics}\label{app_gradient_trick}

{\Cref{alg_efficient_TDLRT} creates a trajectory in the low-rank parameter space, that robustly follows the full-rank solution of the gradient flow of the neural network training}. 
In particular, \Cref{theo_gradient_trick} yields a time-continuous representation of  \Cref{alg_efficient_TDLRT}.
\begin{theorem}\label{theo_gradient_trick}
The evolution equations \cref{eq_kls} are explicit Euler discretizations of a dynamical system which is equivalent to 
\begin{align}\label{eq_cont_kls}
\begin{aligned}
    \dot S &=-\nabla_S\mathcal{L}(U_0S(t)V_0^\top), \qquad S(t=0) = S_0, \\
    \dot K &=-\nabla_K\mathcal{L}(K(t)V_0^\top), \qquad K(t=0) = U_0S_0,\\
    \dot L &=-\nabla_L\mathcal{L}(U_0L(t)^\top), \qquad L(t=0) = S_0^\top V_0,
\end{aligned}
\end{align}
where $\mathcal{L}$ is the stochastic loss given random data samples.
\end{theorem}

\begin{proof}
    Consider the continuous time dynamics of $\dot K$, where we omit explicit time dependence on $U,S,V$ and $K$ for the sake of brevity, i.e.,
\begin{align} \label{eq:Ki_reduced}
    \begin{aligned}
    \dot{K}
        &= \dot{\left(US\right)} \\ 
        &= \dot{U} S + U \dot{S} \\
        &  \stackrel{(\ref{eq_fine_tune_grad_flow_dlrt})}{=}
        -(I-U U^{\top})\nabla_{W}\mathcal{L}(USV^\top)V S^{-1}S - U U^{\top} \nabla_{W}\mathcal{L}(USV^\top) V  \\
        &= - (I -U U^{\top})\nabla_{W}\mathcal{L}(USV^\top)V - U U^{\top}\nabla_{W}\mathcal{L}(USV^\top) V \\
       &= (U U^{\top}- I)\nabla_{W}\mathcal{L}(USV^\top)V - U U^{\top}\nabla_{W}\mathcal{L}(USV^\top) V \\
       &=-\nabla_{W}\mathcal{L}(USV^\top)V
    \end{aligned}
\end{align}
Further, using the chain rule, we observe
\begin{align*}
        \nabla_{U}\mathcal{L}(USV^\top) =\nabla_W\mathcal{L}(USV^\top)\nabla_{U}(USV^\top)=\nabla_W\mathcal{L}(USV^\top)VS^\top\,.
    \end{align*}
Thus, $-\nabla_{U}\mathcal{L}(USV^\top)S^{-\top} = -\nabla_W\mathcal{L}(USV^\top)V= \dot K$. Lastly we have by the chain rule $\dot K =-\nabla_W\mathcal{L}(USV^\top)V  =-\nabla_K\mathcal{L}(USV^\top)$, which yields 
\begin{align*}
    \dot K =-\nabla_{U}\mathcal{L}(USV^\top)S^{-\top}  =-\nabla_K\mathcal{L}(KV^\top)\,.
\end{align*}
Analogously we obtain for $\dot L$
\begin{align*}
    \dot L =-\nabla_{V}\mathcal{L}(USV^\top)S^{-1}  =-\nabla_L\mathcal{L}(UL^\top)\,,
\end{align*}
which concludes the proof.
\end{proof}

Note that using an explicit Euler time discretization for \cref{eq_cont_kls} directly yields \cref{eq_kls}, the update step of \ALGNAME.

\section{Robust error bound of the low-rank system}\label{app_robust_error_bound}
We show the robust error bound for \Cref{alg_efficient_TDLRT} applied to a single layer, and then extend the result to a network containing multiple layers treated with \Cref{alg_efficient_TDLRT}.

\begin{theorem}{(Restatement of \Cref{theo_robust_error_bound})}
For an integer $k$, let $t=k\lambda$. Let $W(t)$ be the solution of \cref{eq_cont_gradient_flow_full_rank}, and let $W^r_t$ be the factorized low-rank solution after $k$ steps with \Cref{alg_efficient_TDLRT}.
    Assume that for any $Z$ in a neighborhood of $W(t)$, we have $\norm{(I-  P(Z))\nabla\mathcal{L}(Z)}<\varepsilon$, i.e., the gradient flow is close to $T_Z \mathcal M_r$. Then, 
\begin{equation}\label{eq:approx}
   \norm{W(t)-W^r_t}\leq  c_{1}\varepsilon + c_{2}\lambda +c_{3}\vartheta/\lambda\,.
\end{equation}
Moreover, let $W_{RF}(t)$ denote the solution of the Riemannian flow of \eqref{eq_fine_tune_grad_flow_dlrt}. Then, 
\begin{equation}\label{eq:approx2}
   \norm{W_{RF}(t)-W^r_t}\leq  c_{4}\varepsilon + c_{2}\lambda +c_{3}\vartheta/\lambda
\end{equation}
where the constants $c_{1}$, $c_{2}$,  $c_{3}$,  $c_{4}$ depend only on $L$ and $B$. 
\end{theorem}

\begin{proof} Let us first investigate the local error. That is, we choose the solution at a given time $t_0$ of the full-rank gradient flow of \cref{eq_cont_gradient_flow_full_rank}, denoted as $W(t_0)$, as a given iteration of \ALGNAME, which we denote as $W^r_0$. Hence, $W(t_0) = W^r_0 =: W_0\in\mathcal{M}_r$. We are then interested in bounding the distance between the full-rank flow at $t_1=t_0+\lambda$ to the \ALGNAME~solution after a single iteration with learning rate $\lambda$. To simplify notation, we denote $\widehat U = [U_0 | \widetilde U]\in\mathbb{R}^{n\times 2r_0}$, $\widehat V = [V_0 | \widetilde V]\in\mathbb{R}^{n\times 2r_0}$ and denote the projections onto these augmented basis vectors as $P_{\augU}=\augU\augU^\top$ and $P_{\augV}=\augV\augV^\top$. Moreover, $c$ denotes a generic constant that only depends on $L$ and $B$. It is important to note that this constant does not depend on $S_k^{-1}$, since we never perform Taylor expansions of the individual low-rank factors.

Let us denote the augmented solution of \ALGNAME{} before truncation as $\augWr = \augU\augS\augV^\top$. Similarly, $W^r_1$ is the truncated solution after iteration~$1$. Then, the local error is bounded by
\begin{align*}
     \norm{W(t_1)-W^r_1} \leq\,& \norm{W(t_1) - P_{\augU} W(t_1) P_{\augV}} + \\
     &\norm{ P_{\augU} W(t_1) P_{\augV} -\augWr } + \norm{\augWr - W^r_1}.
\end{align*}
In the following, we bound the three norms individually in three corresponding steps.

\textbf{Step 1} - Bounding $ \norm{W(t_1) - P_{\augU} W(t_1) P_{\augV}}$: 
Using the triangle inequality, we obtain
    \begin{align*}
        \norm{W(t_1) - P_{\augU} W(t_1) P_{\augV}} \leq&  \norm{W(t_1) - P_{\augU} W(t_1) }+ \norm{P_{\augU}W(t_1)(I - P_{\augV})}\\
         = &\norm{(I - P_{\augU}) W(t_1) }+ \norm{W(t_1)(I - P_{\augV})},
    \end{align*}
    using orthonormality of $\augU$.
    
\textbf{First term:} Consider the first term with the dynamics $\dot W(t)=f(W)$ in mind,
    \begin{align*}
& \norm{(I - P_{\augU})  W(t_1)} \\
        \overset{\textup{(I)}}{\leq}\,& \norm{ (I - P_{\augU}) (W_0  +\lambda f(W_0))} + c\lambda^2 \\
        \overset{}{\leq}\,& \norm{ (I - P_{\augU}) (W_0  -\lambda P(W_0)f(W_0) + \lambda(I-P(W_0))f(W_0))} +c \lambda^2 \\
        \leq\,& \norm{ (I - P_{\augU}) W_0}  + \lambda \norm{(I - P_{\augU})P(W_0)f(W_0)} +\lambda \norm{ (I - P_{\augU})(I-P(W_0))f(W_0))} + c\lambda^2 \\
        \overset{\textup{(II)}}{=}\,&\lambda \norm{(I - P_{\augU})P(W_0)f(W_0)} +\lambda \norm{ (I - P_{\augU})(I-P(W_0))f(W_0))} + c\lambda^2 \\
        \overset{\textup{(III)}}{\leq}\,& \lambda\norm{(I - P_{\augU})  P(W_0)f(W_0)}  + \lambda\varepsilon + c\lambda^2 \\
        \overset{\textup{(IV)}}{\leq}\,& \lambda\norm{(I - P_{\augU}) f(W_0)\augV\augV^\top} + \lambda\varepsilon + c\lambda^2.
        \end{align*}
    using Taylor expansion in (I), $W_0\in\mathcal{M}_r$ in (II), \Cref{ass_1} in (III), and \cref{def_projection} in (IV).

By construction of the basis augmentation, we obtain 
\begin{align}\label{eq_helper_2}
    (I - P_{\augU})K^{\mathrm{new}}=(I - P_{\augU}) U_0S_0 = 0.
\end{align}
From \cref{eq_helper_2} we can directly conclude that $\norm{ (I - P_{\augU})  f(W_0)V_0V_0^{\top} } =0$.
Thus we obtain
\begin{align*}
    \lambda\norm{(I - P_{\augU}) f(W_0)\augV\augV^\top} =& \lambda\norm{(I - P_{\augU}) f(W_0)V_0V_0^\top} + \lambda\norm{(I - P_{\augU}) f(W_0)\newV\newV^\top}\\
    &  \overset{}{\leq} \lambda\epsilon,
\end{align*}
where we used for the second term that $\newV$ is in the orthogonal complement of $V_0$.
Hence,  
    \begin{align*}
       \norm{(I-P_{\augU})  W(t_1) } \leq  c\lambda^2 + \lambda \varepsilon\,.
    \end{align*}
    
\textbf{Second term}: The same derivation for the co-range using the evolution for $L(t)$ yields     
\begin{align*}
    \norm{W(t_1)(I - P_{\augV} ) } \leq c\lambda^2 +  \lambda\varepsilon.
\end{align*}

\textbf{Step 2 } - Bounding $\norm{ P_{\augU} W(t_1) P_{\augV} -\augWr }$: We have by the assembly of the augmented $S$ matrix in \cref{eq_S_step},
\begin{align*}
\augWr = \augU \augS\augV^{\top} = U_0 S^{\mathrm{new}}V_0^{\top} + \newU \newU^\top K^{\mathrm{new}}V_0^{\top} + U_0 L^{\mathrm{new},\top}\newV\newV^\top,
\end{align*}
from which we obtain the error bound between the projected $W(t_1)$ and $\augWr$:
    \begin{align*}
        \norm{P_{\augU} W(t_1) P_{\augV} -\augWr} \leq\,& \norm{P_{\augU} W(t_1) P_{\augV} -U_0 S^{\mathrm{new}}V_0^{\top} + \newU \newU^\top K^{\mathrm{new}}V_0^{\top} + U_0 L^{\mathrm{new},\top}\newV\newV^\top} \\    
      \overset{(I)}{\leq}\,&   \norm{U_0^{\top} W(t_1)V_0 - S^{\mathrm{new}} } +  \norm{ \newU^{\top} W(t_1)V_0 - \newU^{\top} K^{\mathrm{new}} } \\
        \,&+ \norm{U_0^{\top} W(t_1)\newV - L^{\mathrm{new},\top}\newV}+ \norm{\newU^{\top} W(t_1)\newV}\,.
    \end{align*}
    where we use orthonormality of $\augU$,$\augV$ in (I).
    All terms on the right-hand side can be bounded by $\lambda^2$ and $\varepsilon$ terms: 
    
  \textbf{First term:} We have
    \begin{align*}
       \norm{U_0^{\top} W(t_1)V_0 - S^{\mathrm{new}} } \overset{\textup{(I)}}{=}\,& \norm{\int_{t_0}^{t_1} U_0^{\top}(f(W(t))  - f(W_0))V_0   \, dt }\\
       \overset{\textup{(II)}}{\leq}\,& \int_{t_0}^{t_1} \norm{f(W(t))  - f(W_0)}  \, dt \\
      \overset{\textup{(III)}}{=}\,&\int_{t_0}^{t_1} \norm{f(W(t_0))  - f(W_0)}  \, dt + c\lambda^2 \\
        \overset{\textup{(IV)}}{=}\,&c\lambda^2 
    \end{align*}
    where we use in (I) $S^{\mathrm{new}} = S_0 -U_0^\top\nabla_W\mathcal{L}(W_0;\xi)V_0=  -U_0^\top f(W_0)V_0$. We use the orthonormality of $U_0,V_0$ in (II), perform a Taylor expansion of the full-rank flow in (III), and finally use that $W(t_0)=W^S(t_0)$ in (IV).
 
  \textbf{Second and third term:}   We have
    \begin{align*}
        \norm{\newU^\top W(t_1)V_0 - \newU^\top K^{\mathrm{new}} } \overset{\textup{(I)}}{\leq} \,& \int_{t_0}^{t_1}  \norm{\newU^{\top}(f(W(t))  - f(W_0))V_0 }  \, dt \\
       \overset{\textup{(II)}}{\leq}  \,& \int_{t_0}^{t_1}  \norm{f(W(t_0))  - f(W_0)}  \, dt + c\lambda^2 \\
       {=}\,& c\lambda^2\,,
    \end{align*}
where we use the K-step of \ALGNAME{} in (I) and a Taylor expansion of the full-rank flow in (II). $\norm{ U_0^\top W(t_1)\newV -  L^{\mathrm{new},\top}\newV }$ can be bounded analogously.

  \textbf{Fourth term:}    Lastly, we obtain for the fourth term,
        \begin{align*}
       \norm{\newU^\top W(t_1)\newV }=&\,  \norm{\newU^\top W(t_0)\newV + \int_{t_0}^{t_1} \newU^\top f(W(t))\newV \, dt}\\
      \overset{\textup{(I)}}{\leq} \,& \int_{t_0}^{t_1} \norm{\newU^\top f(W(t))\newV}  \, dt\\
        \leq\,& \int_{t_0}^{t_1} \norm{\newU^\top f(W(t_0))\newV}  \, dt + c\lambda^2 
        \overset{\textup{(II)}}{\leq}\lambda \varepsilon +  c\lambda^2\,.
    \end{align*}
     with $\newU^\top W(t_0)\newV =0$ by the construction of the augmented matrix $\augS$ used in (I), and in (II), we use \Cref{ass_1}.

     \textbf{Step 3} - Bounding of $ \norm{\augWr - W^r_1}$: By construction of the truncation step we directly obtain
     \begin{align*}
          \norm{\augWr - W^r_1}\leq \vartheta
     \end{align*}

     In conclusion, we obtain for a single iteration of \Cref{alg_efficient_TDLRT}
\begin{align*}
     \norm{W(t_1)-W^r_1} \leq\,& \norm{W(t_1) - P_{\augU} W(t_1) P_{\augV}} + \\
     &\norm{ \augU\augU^{\top} W(t_1) P_{\augV} -\augWr } + \norm{\augWr - W^r_1} \\
     \leq\,& \widetilde c_1\lambda\epsilon + \widetilde c_2\lambda^2 + \vartheta
\end{align*}
    To conclude, the global error in the training epochs follows by using the Lipschitz continuity of the gradient flow: We move from the local error in time to the global error in time by a standard ODEs argument of Lady Windermere’s fan~\citep[\S II.3]{wanner1996solving}; With $t = k\lambda$ and denoting the adapter computed with \ALGNAME{} at iteration $k$ as $W^r_t$ we then have
    \begin{align*}
     \norm{W(t)-W^r_t}  \leq c_1\epsilon  + c_2\lambda + c_3\vartheta/\lambda\,.
     \end{align*}
     This bounds the distance between the full-rank flow and \ALGNAME. The result trivially extends to the Riemannian flow of \eqref{eq_fine_tune_grad_flow_dlrt}. Denote by $W_{\mathrm{RF}}(t)$ the solution of the Riemannian flow $\dot W_{\mathrm{RF}}(t) = -P(W_{\mathrm{RF}}(t))\nabla_W \mathcal{L}(W_{\mathrm{RF}}(t))$. Then, since $\norm{W(t)-W_{\mathrm{RF}}(t)}  \leq c\epsilon$, it directly follows that 
     \begin{align*}
        \norm{W_{\mathrm{RF}}(t)-W^r_t}  \leq c_4\epsilon  + c_2\lambda + c_3\vartheta/\lambda\,.
     \end{align*}
\end{proof}

\section{Visualization of the stiffness of the basic low-rank system}\label{app_example_stiff}

Consider \Cref{eq_matrix_regression} in the case for $n=20$. We set the target matrix 
\begin{align*}
   W= \begin{bmatrix}
    0 & 15& 0 & \dots \\
    -2 & 0& 0 & \dots \\
    0 & 0& 0 & \dots \\
    \vdots & \vdots &\vdots  & \ddots \\
\end{bmatrix}\in\mathbb{R}^{20\times 20},
\end{align*}
which has rank $r=2$ and singular values $\sigma_1=15$ and $\sigma_2=2$
We compare SVD-lora, AdaLora, and \ALGNAME{}, both with an ansatz of form $W_{\textup{ans}}=USV^\top$ initialized as
\begin{align*}
   U,V= \begin{bmatrix}
    I \\
    0 \\
\end{bmatrix}\in\mathbb{R}^{20\times 4},\qquad S=\begin{bmatrix}
    10 & 0 &0&0 \\
    0 & 1e-2 &0&0 \\
    0 & 0 &1e-4&0 \\
    0 & 0 &0&1e-6 \\
\end{bmatrix}\in\mathbb{R}^{4\times 4}
\end{align*}
where $U$,$V$ are orthonormal, and the $S$ matrix has a fast decaying singular spectrum. 
 
AdaLora and \ALGNAME{} use a relative singular value truncation threshold $\tau=0.15$ for rank truncation. 
We found that learning rate $\lambda=0.178$ is the maximal learning rate before AdaLora and SVD-Lora become unstable, whereas \ALGNAME{} allows for arbitrary large learning rates, and we set $\lambda=0.1$. 
We present the trajectories of the S-matrix elements of the corresponding methods in  \Cref{fig_visualization} for up to $1000$ iterations or until single precision accuracy is reached.
As seen in \Cref{fig_visualization}, AdaLora and SVD-Lora exhibit heavy oscillations in the trajectories of the $S$-matrix elements - leading to slow convergence. Adalora - although using orthonormalization by regularization of the low-rank basis is not able to stabilize the training, leading to overestimation of the rank, which is $r=5$ at final time and a final loss value of $1.6$. Similarly SVD-Lora exhibits even stronger oscillations and is not able to find the right matrix approximation. In contrast, \ALGNAME{} identifies the correct rank $r=2$ and the corresponding correct singular values $15$ and $2$.

\begin{figure}[h!]
    \centering
    \begin{subfigure}[b]{0.3\textwidth}
        \includegraphics[width=\textwidth]{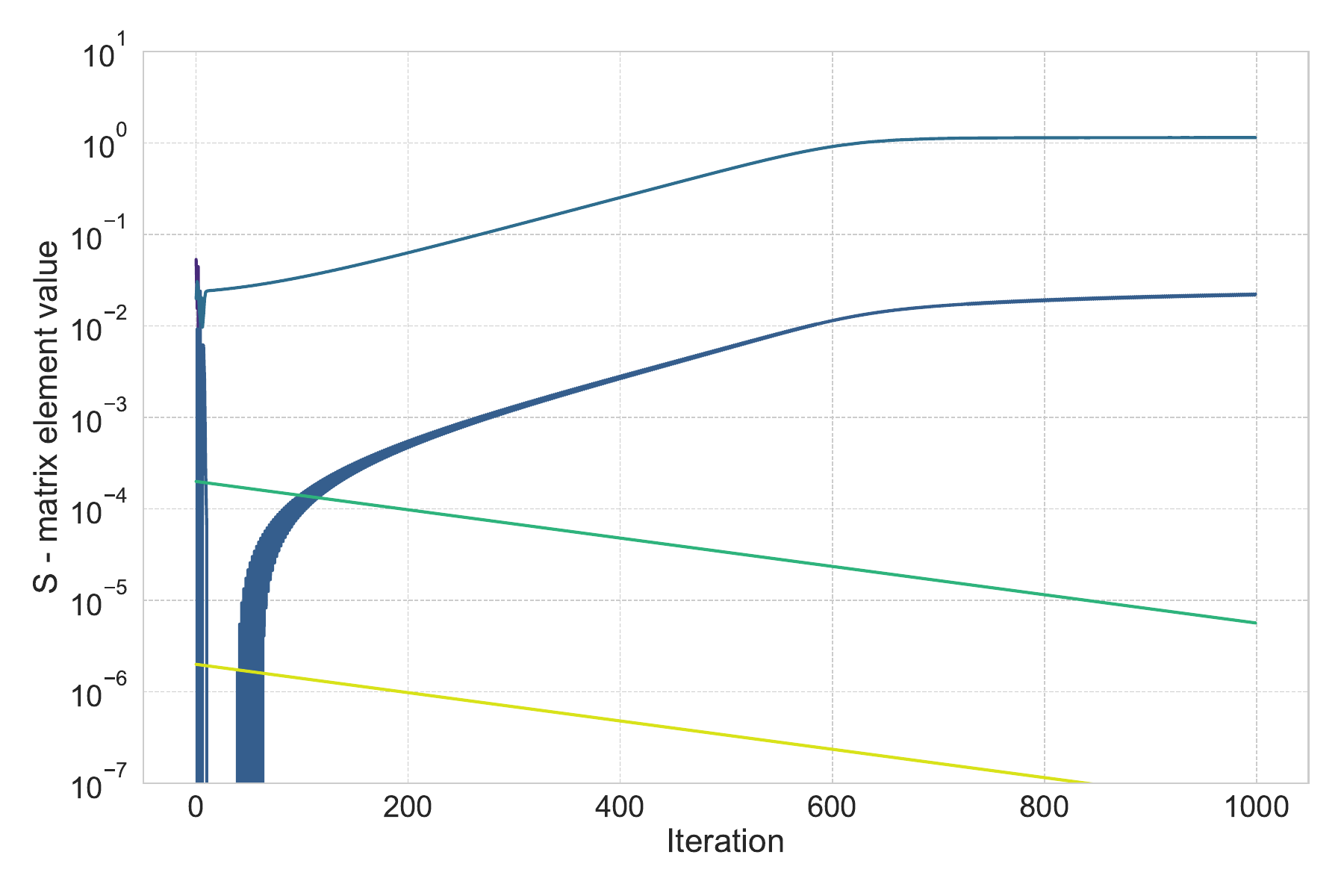}
        \caption{ SVD-Lora, $S$ matrix}
        \label{fig:sub3}
    \end{subfigure}
      \hfill 
    \begin{subfigure}[b]{0.3\textwidth}
        \includegraphics[width=\textwidth]{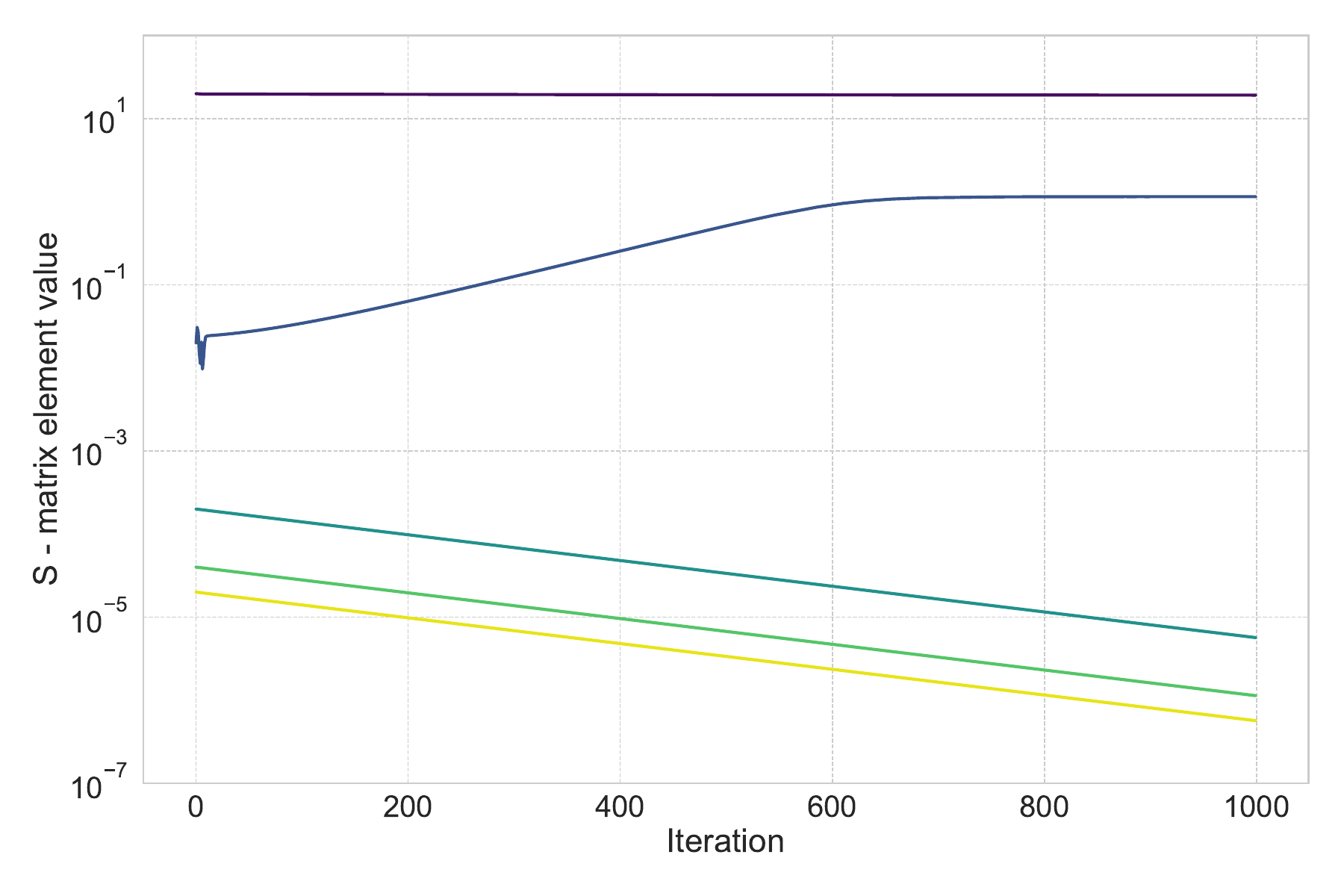}
        \caption{ AdaLora, $S$ matrix}
        \label{fig:sub3}
    \end{subfigure}
    \hfill
      \begin{subfigure}[b]{0.3\textwidth}
        \includegraphics[width=\textwidth]{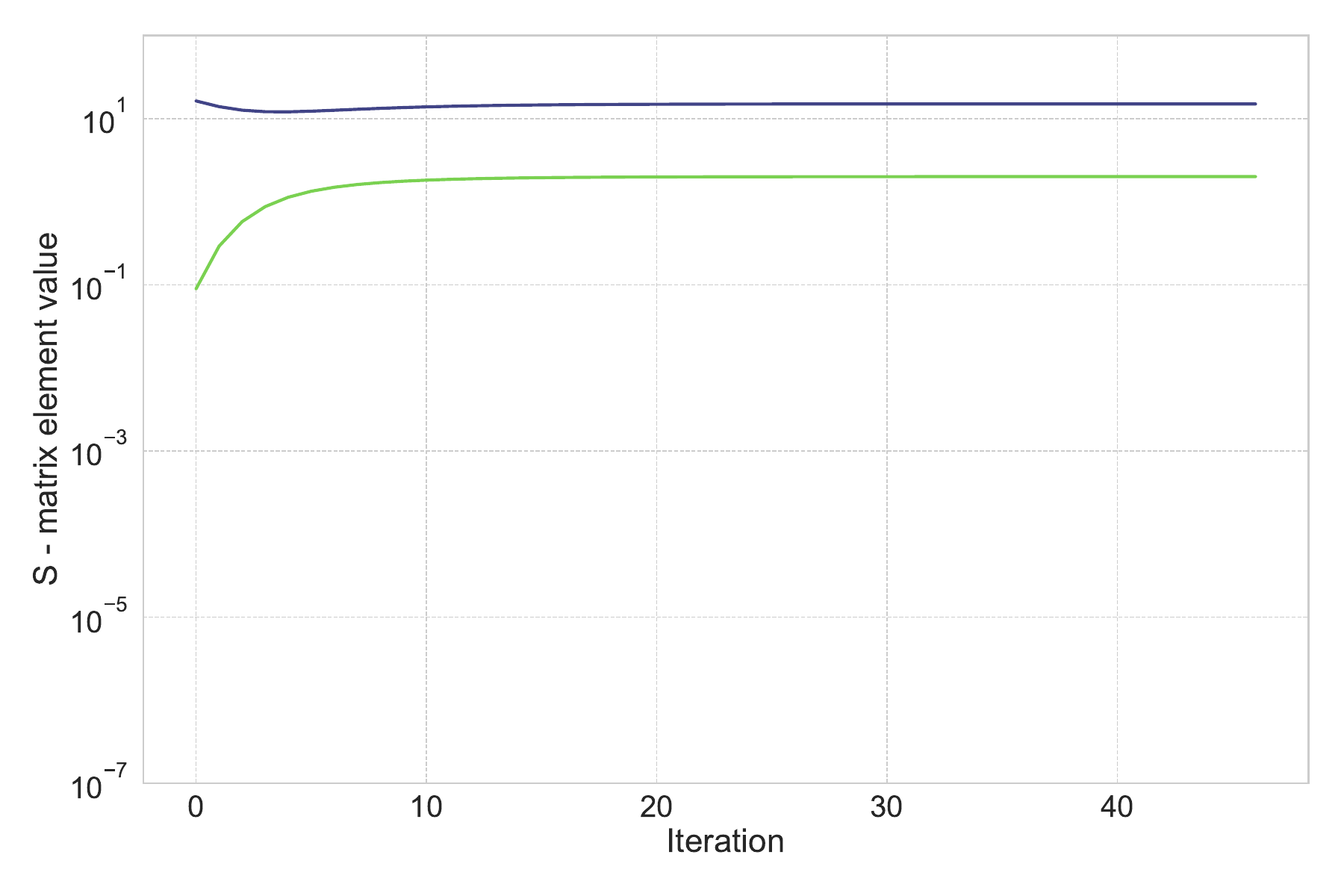}
        \caption{ \ALGNAME{}, $S$ matrix}
        \label{fig:sub3}
    \end{subfigure}
    \caption{Time-trace of the matrix elements of SVD-Lora (a) AdaLora (b) and the proposed method \ALGNAME{} (c) to solve \Cref{eq_matrix_regression}. SVD-Lora was trained with learning rate $\lambda=0.00178$, which is the largest learning rate for which the optimization remained stable, \ALGNAME{} allows larger learning rates, set to $\lambda=0.1$. \ALGNAME{} converges fast to single precision accuracy, whereas SVD-LORA still has a loss value of $1.7$ after $1000$ iterations, due to the heavy oscillations in it's $S$ matrix trajectory (a). Adalora reduces the oscillations, however incorrectly identifies the rank and fails to converge due to the influence of the additional singular values.}
    \label{fig_visualization}
\end{figure}
\section{Structure preservation}\label{app_full_network_lr}
The goal of this section is to clarify the formulation of \Cref{alg_efficient_TDLRT} in relation with the previous related literature. In particular, we want to show that the proposed algorithm can be seen as an efficient structure preservation formulation of a projected gradient flow \citep{KochLubich07} for training neural networks. In this section, to achieve full generality, we will denote with $Y_i$ either the pretrained matrices or the low-rank adapters. \\
As already mentioned in the previous section, gradient descent can be seen an forward Euler discretization of the gradient system
\[
\dot Y_i = -\nabla_{Y_i} \mathcal L(Y_1,\dots,Y_L), \,\,\,i = 1,\dots,L
\]
The neural network $f_{Y_1,\dots,Y_L}$ naturally induces a weighted graph, where nodes are neurons and weights are connections among them and for which the adjacency matrix can be written as:
\[
\mathcal Y:=\left[
\begin{array}{c|c|c|c|c}
0 & Y_1 & 0   & \cdots & 0 \\
\hline
0 & 0   & Y_2 & \cdots & 0 \\
\hline
\vdots & \vdots & \vdots & \ddots & \vdots \\
\hline
0 & 0   & 0   & \cdots & Y_{L} \\
\hline
0 & 0   & 0   & \cdots & 0
\end{array}
\right] \in \mathbb{R}^{|\mathcal{N}|\times |\mathcal{N}|}
\]
where $|\mathcal N | = \sum_{i=1}^L d_i$ is the total number of neurons of the neural network.\\
The matrix $\mathcal Y$ now represents the adjacency matrix of the computational graph, and the block structure is given by the layers. A model with a general full adjacency matrix $\mathcal Y$, would in general have non-zero connections between two generic layers $i,j$, descriebed by the block $Y_{ij}$. Let's consider the model
\[
f_{\mathcal Y}(x) = z^{L}(x),\,\,\ z^0(x) = x,\,\,\, z^{\ell+1} =  \sigma_i\Bigl(\sum_{i}Y_{i,\ell+1}z^i\Bigr)
\]
Notice that for $\mathcal Y$ upper diagonal, the previous model would be a feedforward network.
Given this observation, under the assumption that $f_{\mathcal Y}(x)$ is well defined as an eventual fixed point, we can now see the loss function as a function of the full adjacency matrix, with an abuse of notation we will call it again $\mathcal L(\mathcal Y)$. Usual training would superimpose the sparse graph with the same structure of $\mathcal Y$, but let's consider for a moment the gradient flow
\[
\dot{\mathcal Y} = -\nabla \mathcal{L}(\mathcal Y)
\]
Clearly, the flow does not preserve the sparsity of the adjacency matrix $\mathcal Y$, even for sparse initial conditions. Using the theory developed in \citep{KochLubich07} directly on this gradient flow would lead to neural networks with a non-feedforward topology. Moreover, given the size of $|\mathcal N|$ for modern neural networks, it can be expensive to compute QR or SVD decomposition of the basis matrices. 
Luckily, the sparsity structure is a simple linear constraint represented by the mask matrix
\[
\mathcal M:=\left[
\begin{array}{c|c|c|c|c}
0 & 11^\top & 0   & \cdots & 0 \\
\hline
0 & 0   & 11^\top & \cdots & 0 \\
\hline
\vdots & \vdots & \vdots & \ddots & \vdots \\
\hline
0 & 0   & 0   & \cdots & 11^\top \\
\hline
0 & 0   & 0   & \cdots & 0
\end{array}
\right] \in \mathbb{R}^{|\mathcal{N}|\times |\mathcal{N}|}
\]
and the linear operator $\Pi(A) = \mathcal{M} \odot A$. \\ 

A system preserving the sparsity pattern is given naturally by the ODE
\[
\dot{\mathcal Y} = -\Pi \nabla \mathcal L(\mathcal Y)
\]
However, it is not obvious that by projecting this last system on the manifold of rank-r matrices $\mathcal M_r$ the block structure is preserved. 
Fortunately, it is indeed the case, described by the following lemma:
\begin{proposition}(Block structure preservation of the flow) \\
Consider the gradient flow with sparse initial condition
\[
\dot {\mathcal Y} = -P(\mathcal Y)\Pi\nabla \mathcal L(\mathcal Y), \,\,\, \mathcal Y(0) = \mathcal Y_0 \in range(\Pi)
\]
Then $\mathcal Y(t) \in range(\Pi)$ for all $t\geq 0$.
\end{proposition}
\begin{proof}
    It is necessary and sufficient to prove that $P(\mathcal Y(t))\Pi\nabla \mathcal L(\mathcal Y(t)) \in range(\Pi)$ for all $t \geq 0$, i.e. that $\Pi P(\mathcal Y(t))\Pi\nabla \mathcal L(\mathcal Y(t)) = P(\mathcal Y(t))\Pi\nabla \mathcal L(\mathcal Y(t))$. The key to prove this is to observe that for $ Z \in range(\Pi)$, we have $P(\mathcal Y) Z \in range(\Pi)$. In fact, given $Z \in range(\Pi)$, we can write a SVD of $Z$ as
    \[
    Z = \left[
\begin{array}{c|c|c|c|c}
 0 & U_1 & 0   & \cdots & 0 \\
\hline
0 &  0   & U_2 & \cdots & 0 \\
\hline
\vdots & \vdots & \vdots & \ddots & \vdots \\
\hline
0 & 0   & 0   &  0 & U_L \\
\hline
I & 0   & 0   & \cdots &  0
\end{array}
\right]
\left[
\begin{array}{c|c|c|c|c}
 0 & 0 & 0   & \cdots & 0 \\
\hline
0 &  S_1   & 0 & \cdots & 0 \\
\hline
\vdots & \vdots & \vdots & \ddots & \vdots \\
\hline
0 & 0   & 0   &  S_{L-1} & 0  \\
\hline
0 & 0   & 0   & \cdots &  S_L
\end{array}
\right]
\left[
\begin{array}{c|c|c|c|c}
 I & 0 & 0   & \cdots & 0 \\
\hline
0 &  V_1^\top   & 0 & \cdots & 0 \\
\hline
\vdots & \vdots & \vdots & \ddots & \vdots \\
\hline
0 & 0   & 0   &  V_{L-1}^\top & 0  \\
\hline
0 & 0   & 0   & \cdots &  V_L^\top
\end{array}
\right]
    \]
\end{proof}
and we have $UU^\top,VV^\top \in range(\Pi)$. Thus, by direct calculation we can show that $UU^\top Z,ZVV^\top, UU^\top Z VV^\top \in range(\Pi)$ and thus $P(\mathcal Y)Z = UU^\top Z + Z VV^\top -UU^\top Z VV^\top \in range(\Pi)$. Since $\Pi \nabla \mathcal L(\mathcal Y(t)) \in range(\Pi)$ by construction for all $t \geq 0$, we get the desider result.
Thanks to this last proposition, following again the line of work in \citep{KochLubich07}, it is possible to restrict the parameterization in the tangent space to a block-structured one as in \Cref{prop:structure_preservation}. In this way, we get the following coherence theorem:
\begin{proposition}
    Consider the gradient flow with sparse initial condition
\[
U = \dot {\mathcal Y} = -P(\mathcal Y)\Pi\nabla \mathcal L(\mathcal Y), \,\,\, \mathcal Y(0) = \mathcal Y_0 \in range(\Pi)
\]
Consider now the parametrization $\mathcal Y = U S V^\top$ with
\[
U = \left[
\begin{array}{c|c|c|c|c}
 0 & U_1 & 0   & \cdots & 0 \\
\hline
0 &  0   & U_2 & \cdots & 0 \\
\hline
\vdots & \vdots & \vdots & \ddots & \vdots \\
\hline
0 & 0   & 0   &  0 & U_L \\
\hline
I & 0   & 0   & \cdots &  0
\end{array}
\right], S= \left[
\begin{array}{c|c|c|c|c}
 0 & 0 & 0   & \cdots & 0 \\
\hline
0 &  S_1   & 0 & \cdots & 0 \\
\hline
\vdots & \vdots & \vdots & \ddots & \vdots \\
\hline
0 & 0   & 0   &  S_{L-1} & 0  \\
\hline
0 & 0   & 0   & \cdots &  S_L
\end{array}
\right], V = \left[
\begin{array}{c|c|c|c|c}
 I & 0 & 0   & \cdots & 0 \\
\hline
0 &  V_1   & 0 & \cdots & 0 \\
\hline
\vdots & \vdots & \vdots & \ddots & \vdots \\
\hline
0 & 0   & 0   &  V_{L-1} & 0  \\
\hline
0 & 0   & 0   & \cdots &  V_L
\end{array}
\right]
\]
where $U_i^\top U_i = I,V_i^\top V_i = I$. Then, by imposing the Gauge conditions $\dot{U_i}^\top U_i = 0, \dot{V_i}^\top V_i = 0$, the projected flow $\dot{\mathcal Y} = -P(\mathcal Y)\Pi \nabla \mathcal L (\mathcal Y)$ can be rewritten in block fashion as follows:
\[
\begin{aligned}
\dot S_i(t) &= -U_i^\top(t)\nabla_{Y_i}\mathcal{L}(U(t) S(t)V(t)^\top)V_i(t), \\
\dot U_i(t) &= -\left(I - P_{U_i(t)}\right)\nabla_{Y_i}\mathcal{L}(U(t)S(t)V(t)^\top) V_i(t)S_i(t)^{-1}, \\
\dot V_i(t) &= -\left(I - P_{V_i(t)}\right)\nabla_{Y_i}\mathcal{L}(U(t)S(t)V(t)^\top) U_i(t)S_i(t)^{-\top},\,\,\,\, i = 1,\dots,L
\end{aligned}
\]
\end{proposition}
\begin{proof}
    Thanks to the previous proposition, we know that the variation $ P(\mathcal Y)\Pi \nabla \mathcal L(\mathcal Y) \in range(\Pi)$ for all $t\geq 0$. Then, we have $\mathcal Y(t) \in range(\Pi)$ for all times, and thus we can decompose it using a block SVD as described in the statement of the proposition. Moreover, by the self-adjointness of $\Pi$, Galerkin condition can be written as:
    \[
    \langle \dot{\mathcal Y} + \nabla \mathcal L(\mathcal Y),q \rangle=\langle \dot U SV^\top + U\dot S V^\top +US\dot V^\top + \nabla \mathcal L(\mathcal Y),q \rangle = 0,\,\,\, \forall q \in T_{\mathcal Y} \mathcal M_r \cap range(\Pi)
    \]
    Since $q \in T_{\mathcal Y} \mathcal M_r \cap range(\Pi)$, we can represent it as $q = \delta U S V^\top + U\delta S V^\top +US\delta V^\top$, with $\delta U,\delta V,\delta S$ with the same block structure of $U,S$ and $V$. By writing the last conditions on a basis of $T_{\mathcal Y} \mathcal M_r \cap range(\Pi)$, we get
    \[
    \begin{aligned}
        & \langle \dot U SV^\top + U\dot S V^\top +US\dot V^\top + \nabla \mathcal L(\mathcal Y),\delta U SV^\top \rangle = 0 \\
        & \langle \dot U SV^\top + U\dot S V^\top +US\dot V^\top + \nabla \mathcal L(\mathcal Y),U \delta SV^\top \rangle = 0 \\
        & \langle \dot U SV^\top + U\dot S V^\top +US\dot V^\top + \nabla \mathcal L(\mathcal Y),U S\delta V^\top \rangle = 0 \\
    \end{aligned}
    \]
    Thanks to the Gauge conditions $\dot U^\top U = 0, \dot V^\top V = 0$ and to the properties of the Frobenius inner product, the last system becomes
    \[
    \begin{aligned}
        & \langle \dot U SS^\top + \nabla \mathcal L(\mathcal Y) V S^\top ,\delta U \rangle = 0 \\
        & \langle  U^\top U\dot S V^\top V + U^\top \nabla \mathcal L(\mathcal Y)V,\delta S \rangle = \langle  \dot S  + U^\top \nabla \mathcal L(\mathcal Y)V,\delta S \rangle = 0 \\
        & \langle S^\top S\dot V^\top + S^\top U^\top\nabla \mathcal L(\mathcal Y), \delta V^\top \rangle = 0 \\
    \end{aligned}
    \]
    and from this equations we get the known
    \[
    \begin{aligned}
        & \dot U  = - (I-UU^\top)\nabla \mathcal L(\mathcal Y) V S^{-1} \\
        & \dot S   = - U^\top \nabla \mathcal L(\mathcal Y)V \\
        & \dot V = -(I- VV^\top)\nabla \mathcal L(\mathcal Y)^\top U S^{-\top}  \\
    \end{aligned}
    \]
    By writing this equations block-by-block, we get the desidered result.
\end{proof}
This last proposition clarifies how to connect the single matrix setting with the multi-matrix setting, showing that the presentation of \Cref{alg_efficient_TDLRT} is in fact coherent with the single matrix setting. Moreover, investigation of this setting leads naturally to the global truncation strategy.

\subsection{Truncation strategy}
The global truncation strategy proposed in the main manuscript is in fact coherent with the single matrix formulation presented in the previous section. In fact, one can assemble the rank augmented $S$ matrix as:
\[
\widehat S(t = 1) =  \left[
\begin{array}{c|c|c|c|c}
\widehat S_1 & 0 & 0   & \cdots & 0 \\
\hline
0 &\widehat  S_2   & 0 & \cdots & 0 \\
\hline
\vdots & \vdots & \vdots & \ddots & \vdots \\
\hline
0 & 0   & 0   & \widehat S_{L-1} & 0 \\
\hline
0 & 0   & 0   & \cdots & \widehat S_L
\end{array}
\right]
\]
and then truncate the smallest singular values up to required precision. This can be efficiently done by computing an SVD on each diagonal block, giving effectively an SVD of the global matrix. In particular, if $\widehat S_i = P_i \Sigma_i Q_i^\top$ we get that
\[
\widehat S  = blockdiag(P_1,\dots,P_L)blockdiag(\Sigma_1,\dots,\Sigma_L)blockdiag(Q_1,\dots,Q_L)^\top
\]
Since the matrix $blockdiag(\Sigma_1,\dots,\Sigma_L)$ is effectively diagonal, by assuming the diagonal is increasingly ordered, it is natural to globally truncate the ranks according to the minimal $k$ such that
\[
\frac{\sum_{i=k+1}^{2rL} \sigma_{i}^2}{ \sum_{i=1}^{2rL} \sigma_{i}^2}  < \frac{\tau}{1-\tau}
\]
Which corresponds in throwing away the smallest singular values of $\widehat S$ until we reach the desired relative error. By rewriting this criterion on the singular values of each matrix $\widehat S_i$, we get exactly the global criterion proposed in \Cref{sec:method}.

\section{Optimality on the low-rank manifold}\label{app:local_optimality}
We remark below that if $W_\star$ is a local minimum, then $P(W_\star)\nabla \mathcal L(W_\star) = 0$. In particular, since $P(W)\nabla \mathcal L(W)$ is the Riemannian gradient with respect to the ambient metric, then the following holds by definition of the gradient:
\[
\partial_{\delta W} \mathcal L(W) = \langle P(W)\nabla \mathcal L(W),\delta W \rangle 
\]
where $\partial_{\delta W} \mathcal L(W)$ is the directional derivative of $\mathcal L$ along the direction $\delta W$. Thus, $P(W_\star)\nabla \mathcal L(W_\star) = 0$ if and only if $\partial_{\delta W}\mathcal L(W) = 0$ for all $\delta W \in T_{W}\mathcal M$ and this happens if and only if $\nabla \mathcal L(W) \in (T_{W}\mathcal M)^{\perp}$. So geometrically, if $W_\star$ is a local minimum, then $P(W_\star)\nabla \mathcal L(W_\star) = 0$ means that among all available directions, there are none that decrease the loss.\\
For simultaneous descent, the same condition doesn't hold, in fact, the algorithm's stationary points satisfy $\widehat P(W_\star)\nabla \mathcal L(W_\star) = 0$, which given the non-orthogonality does not, in general, imply $P(W_\star)\nabla \mathcal L(W_\star) = 0$, so there could be descent directions unexploited by the method.

\end{document}